\documentclass{article}

\usepackage{graphicx,amsmath,amsfonts,amssymb,bm,url,epsfig,epsf,color,fullpage,MnSymbol,mathbbol,fmtcount,algorithmic,algorithm,semtrans,caption,subcaption,multirow,amsthm,enumitem}

\usepackage[title]{appendix}
\usepackage[utf8]{inputenc} % allow utf-8 input
\usepackage[T1]{fontenc}    % use 8-bit T1 fonts
\usepackage{url}            % simple URL typesetting
\usepackage{booktabs}       % professional-quality tables
\usepackage{amsfonts}       % blackboard math symbols
\usepackage{nicefrac}       % compact symbols for 1/2, etc.
\usepackage{microtype}      % microtypography
\usepackage{hhline}

\newtheorem{thm}{Theorem}
\newtheorem{remark}{Remark}
\newtheorem{lemma}{Lemma}

\newtheorem{prop}{Proposition}
\newtheorem{condition}{Condition}

\newtheorem*{remark*}{Remark}
\newtheorem*{thm*}{Theorem}
\newtheorem*{lemma*}{Lemma}
\newtheorem*{prop*}{Proposition}
\newcommand\numberthis{\addtocounter{equation}{1}\tag{\theequation}}

\usepackage[T1]{fontenc}
\usepackage[utf8]{inputenc}
\usepackage{pgfplots}
\usepackage{graphicx}
\usepackage{subcaption}
\usepackage{pstricks}
\usepackage{color}
\usepackage{amsmath}
\usepackage{bbm}
\usepackage{tcolorbox}
\usepackage{tikz}
\usepackage{pgfplots}
\usepackage{wrapfig}
\usepackage{lipsum}  
\usetikzlibrary{positioning}
\usepackage{tcolorbox}
\usepackage{amsfonts,amssymb}
 \usepackage[numbers]{natbib}
\usepackage{mathtools}
\usepackage{commath}
\usepackage{relsize}
\usepackage{bbm}
\usepackage{bm}
\usepackage[font={small}]{caption}
\usepackage{comment,color,soul}
\usepackage{enumerate} 
\usetikzlibrary{arrows,automata}
\usepackage{amsfonts}
\usepackage{url}
\usepackage{lipsum}
\usepackage[thinlines]{easytable}
\usepackage{nicefrac}
\usepackage{float}
\usepackage{titling}

\usepackage[hidelinks,backref=page]{hyperref}
\definecolor{darkred}{RGB}{150,0,0}
\definecolor{darkgreen}{RGB}{0,150,0}
\definecolor{darkblue}{RGB}{0,0,150}
\hypersetup{colorlinks=true, linkcolor=darkred, citecolor=darkblue, urlcolor=darkblue}

\date{}
    
\newcommand*\samethanks[1][\value{footnote}]{\footnotemark[#1]}

\author{
Farzan Farnia\thanks{Equal contribution}  \thanks{Massachusetts Institute of Technology, Cambridge,
MA, USA, \{farnia,wwang314,jadbabai\}@mit.edu.}
\and 
William Wang\samethanks[1] \samethanks[2] 
\and
Subhro Das\thanks{MIT-IBM Watson AI Lab, IBM Research, Cambridge, MA, USA, subhro.das@ibm.com.}
\and
Ali Jadbabaie\samethanks[2]
}

\title{GAT-GMM: Generative Adversarial Training for\\
Gaussian Mixture Models}

\begin{document}

\maketitle

\begin{abstract}
Generative adversarial networks (GANs) learn the distribution of observed samples through a zero-sum game between two machine players, a generator and a discriminator. While GANs achieve great success in learning the complex distribution of image, sound, and text data, they perform suboptimally in learning multi-modal distribution-learning benchmarks including Gaussian mixture models (GMMs). In this paper, we propose Generative Adversarial Training for Gaussian Mixture Models (GAT-GMM), a minimax GAN framework for learning GMMs. Motivated by optimal transport theory, we design the zero-sum game in GAT-GMM using a random linear generator and a softmax-based quadratic discriminator architecture, which leads to a non-convex concave minimax optimization problem. We show that a Gradient Descent Ascent (GDA) method converges to an approximate stationary minimax point of the GAT-GMM optimization problem. In the benchmark case of a mixture of two symmetric, well-separated Gaussians, we further show this stationary point recovers the true parameters of the underlying GMM. We numerically support our theoretical findings by performing several experiments, which demonstrate that GAT-GMM can perform as well as the expectation-maximization algorithm in learning mixtures of two Gaussians.
\end{abstract}

\section{Introduction}
Learning the distribution of observed data is a basic task in unsupervised learning which has been studied for decades. The recently-introduced concept of Generative Adversarial Networks (GANs) \cite{goodfellow2014generative} has demonstrated great success in various distribution learning tasks. Unlike the traditional maximum-likelihood-based  approaches, GANs learn the distribution of observed data through a zero-sum game between two machine players, a generator $G$ mimicking the true distribution of data and a discriminator $D$ distinguishing the generator's produced samples from real data points. This zero-sum game is typically formulated through a minimax optimization problem where $G$ and $D$ optimize a minimax objective quantifying how dissimilar $G$'s generated samples and real training samples are.  

In GAN minimax optimization problems, the generator and discriminator functions are commonly chosen as two deep neural networks (DNNs). Leveraging the expressive power of DNNs, GANs have achieved state-of-the-art performance in learning complex distributions of image data \cite{karras2017progressive,zhang2018self,brock2018large}. This success, however, is achieved at the cost of their notoriously difficult training procedure which has introduced several challenges to the machine learning community. Addressing these challenges requires a deeper theoretical understanding of GANs, including their approximation, generalization, and optimization properties.   

Specifically, GANs have been frequently observed to fail in learning multi-modal distributions \cite{goodfellow2016nips}. As a widely-recognized training issue, a trained GAN model may collapse into only one or a few modes of the underlying distribution, a phenomenon known as \emph{mode-collapse} in the literature. Despite the recent advances in improving the generalization and stability properties of GAN models, state-of-the-art GAN architectures are often observed to struggle in learning even the simplest class of mixture distributions, i.e., the Gaussian mixture models (GMMs). Such empirical observations question the superiority of GANs over traditional methods of learning mixture models such as the expectation-maximization (EM) algorithm \cite{dempster1977maximum}. A natural question here is whether this gap fundamentally exists in learning multi-modal distributions or is due to a lack of appropriate design for the GAN players and loss function.  In order to better understand how GANs learn multi-modal distributions, we focus on a GAN architecture that is amenable for learning the simplest, nontrivial model: a mixture of two Gaussians. After all, unless we understand how an adversarial learning architecture behaves in the simplest of cases, it would be hard to understand what happens in the case of deep neural networks.

In what follows, we show that with a proper design of the generator and discriminator function classes and choice of the minimax objective in the GAN's optimization problem, it is possible to achieve performance similar to that of the EM algorithm in learning mixtures of two Gaussians. To achieve this goal, we propose \emph{Generative Adversarial Training for Gaussian Mixture Models (GAT-GMM)}, a minimax GAN framework for learning GMMs. GAT-GMM is formulated based on minimizing the Wasserstein distance between the underlying and generative GMMs. Leveraging optimal transport theory, we characterize optimal function spaces for training the generator and discriminator players as well as the minimax objective of the GAN problem. We show that GAT-GMM represents a non-convex concave minimax optimization problem which can be efficiently solved by a gradient descent ascent (GDA) method to reach a stationary minimax point.  

In the well-studied benchmark case of symmetric mixtures of two well-separated Gaussians, we theoretically support GAT-GMM by providing approximation, generalization, and optimization guarantees. We show that the designed generator and discriminator will suffice for reaching zero approximation error in learning such two-component GMMs. Furthermore, we prove that the underlying GMM provides the only stationary minimax point satisfying a separability condition. We also bound the generalization error of estimating the minimax objective and its derivatives from empirical samples. The generalization bounds scale linearly with the dimension of the data as a consequence of GAT-GMM's specific design. The generalization and optimization guarantees together show that a learnt well-separated GMM will generalize to the underlying distribution. 

Finally, we experimentally demonstrate the success of GAT-GMM in learning symmetric mixtures of two Gaussians. We show that in practice, GAT-GMM can be optimized efficiently using a simple GDA optimization algorithm. Numerically, we find that GAT-GMM performs favorably compared to GANs with neural network players and achieves EM-like numerical performance in learning symmetric mixtures of two Gaussians. Our empirical results indicate that the generative adversarial training approach can potentially achieve state-of-the-art performance in learning GMMs. We summarize the main contributions of this work as follows:
\begin{itemize}[leftmargin=*]
    \item Proposing GAT-GMM as a generative adversarial training approach for learning GMMs,
    \item Reducing GAT-GMM to a non-convex concave minimax optimization problem with convergence guarantees to stationary minimax points,
    \item Demonstrating theoretical guarantees for GAT-GMM in learning symmetric mixtures of two well-separated Gaussians,
    \item Providing numerical support for the GAT-GMM approach in learning two-component GMMs.
\end{itemize}

\section{Related Work}
\textbf{Theory of GANs:}
A large body of recent works have studied the theoretical aspects of GANs, including their approximation \cite{liu2017approximation,farnia2018convex,liu2018inductive}, generalization \cite{arora2017generalization,arora2017gans,zhang2017discrimination}, and optimization \cite{nagarajan2017gradient,mescheder2017numerics,roth2017stabilizing,daskalakis2017training,heusel2017gans,mescheder2018training,lin2018pacgan,farnia2020gans,lei2019sgd} properties. We note that our model-based GAN framework for GMMs is similar to the approaches in \cite{feizi2017understanding,bai2018approximability} for learning Gaussians and invertible neural net generators. Specifically, \cite{bai2018approximability} proposes a discriminator function for learning GMMs which matches our proposed model in the special case of identity covariance matrix. However, our work further analyzes the optimization properties of the resulting GAN problem. Also, multiple recent works \cite{arjovsky2017wasserstein,bousquet2017optimal,feizi2017understanding,salimans2018improving,sanjabi2018convergence,genevay2018sample} explore the applications of optimal transport theory in improving GANs' stability and convergence behavior.

\textbf{GANs for Learning GMMs:} Regarding the applications of GANs in learning GMMs, Pac-GAN~\cite{lin2018pacgan} seeks to resolve the mode collapse issue by learning the distribution of two samples and reports improved performance scores in learning mixtures of Gaussians. The related references~\cite{ben2018gaussian,richardson2018gans,xiao2018bourgan} propose considering a GMM input to GAN's generator, and report empirical success for fitting GMMs. Flow-GANs~\cite{grover2018flow} combine the maximum likelihood approach with GAN training and improve the performance of GANs in learning GMMs. Reference~\cite{metz2016unrolled} suggests unrolling the discriminator's optimization and shows it improves learning an underlying GMM. Reference~\cite{li2017limitations} analyzes the convergence behavior of GANs in learning univariate mixtures of two Gaussians. Unlike our work, the above papers consider standard neural network players, with the exception of~\cite{li2017limitations}.

\textbf{Theory of Learning GMMs:} Several related works have studied the theoretical aspects and limits of learning GMMs, including the convergence and generalization behavior of learning GMMs with the EM algorithm \cite{daskalakis2016ten,xu2016global,balakrishnan2017statistical,regev2017learning,yan2017convergence,barazandeh2018behavior,zhao2020statistical,nagarajan2020analysis}, the method of moments \cite{moitra2010settling,anandkumar2012method,hsu2013learning,hardt2015tight,ge2015learning,hopkins2018mixture}, optimal transport tools \cite{chen2018optimal,kolouri2018sliced,gaujac2018gaussian}, and recovery guarantees under separability assumptions \cite{dasgupta1999learning,sanjeev2001learning,dasgupta2007probabilistic,chaudhuri2009learning}. Specifically, \cite{daskalakis2016ten,xu2016global} analyze a similar benchmark setting of mixtures of two symmetric Gaussians for the EM algorithm. We note that these papers also assume a known covariance matrix and aim to learn only the mean parameter, while our theoretical setup further considers and learns an unknown covariance matrix.

\section{Preliminaries}
%Here, we provide a brief overview of the notation and optimal transport tools used in the paper. We defer a more detailed review of background concepts and results to the Appendix. 
\subsection{Gaussian Mixture Models}
We denote a $k$-component Gaussian mixture model (GMM) by $
    p_{{\pi}_i,\boldsymbol{\mu}_i,{\Sigma}_i}(\mathbf{x}) := \sum_{i=1}^k \pi_i\mathcal{N}(\mathbf{x}\,\vert\, \boldsymbol{\mu}_i,\Sigma_i ),$
where $\mathcal{N}(\mathbf{x}\,\vert\, \boldsymbol{\mu},\Sigma) $ is the multivariate Gaussian distribution with mean $\boldsymbol{\mu}$ and covariance matrix $\Sigma$, and, $\pi_i$ stands for the probability of observing a sample from component $i$. A GMM uniformly distributed among its components will also satisfy $\pi_i=\frac{1}{k}$ for every $i$. If the covariance matrix is shared among the components we will have $\Sigma_i=\Sigma_j$ for every $i,j$. We call a two-component GMM symmetric if in addition to the uniform distribution and common covariance among components, we also have opposite means $\boldsymbol{\mu}_1=-\boldsymbol{\mu}_2$.

\subsection{GANs and Optimal Transport Costs}
The GAN framework learns the distribution of data through a minimax problem optimizing generator $G$ and discriminator $D$. The following minimax optimization is the vanilla GAN problem introduced in \cite{goodfellow2014generative}:
\begin{equation}
    \min_{G\in\mathcal{G}}\; \max_{D\in\mathcal{D}}\; \mathbb{E}[\log(D(\mathbf{X}))] + \mathbb{E}[\log(1- D(G(\mathbf{Z})))].
\end{equation}
In the above equation, $\mathcal{G}$ and $\mathcal{D}$ are function spaces for $G$ and $D$, respectively. $\mathbf{Z}\sim\mathcal{N}(\mathbf{0},I)$ is the random input to the generator, which we assume has a standard multivariate Gaussian distribution throughout this work. To improve the stability in training GANs, \cite{arjovsky2017wasserstein} proposes a GAN minimax problem minimizing an optimal transport cost. For a transportation cost $c(\mathbf{x},\mathbf{x}')$, the optimal transport cost $W_c$ is defined as $W_c(P,Q) := \inf_{M \in \Pi(P,Q)}\: \mathbb{E}_M[c(\mathbf{X},\mathbf{X}')]$. Here, $\Pi(P,Q)$ is the set of all joint distributions with marginals $P$ and $Q$. A special case of interest is the 2-Wasserstein cost corresponding to quadratic $c(\mathbf{x},\mathbf{x}')=\frac{1}{2}\Vert \mathbf{x} -\mathbf{x}' \Vert_2^2$. The following minimax problem formulation, called W2GAN~\cite{feizi2017understanding}, minimizes the 2-Wasserstein cost where $D$ is called $c$-concave if for a function $\widetilde{D}$ we have $D(\mathbf{x})=\inf_{\mathbf{x}'}\, \widetilde{D}(\mathbf{x}')+c(\mathbf{x},\mathbf{x}')$, and the $c$-transform is defined as $D^c(\mathbf{x}):=\sup_{\mathbf{x}'} D(\mathbf{x}')-c(\mathbf{x},\mathbf{x}')$:
\begin{equation}\label{Eq: general WGAN}
    \min_{G\in \mathcal{G}}\; \max_{D\, \operatorname{c-concave}}\; \mathbb{E}\bigl[D(\mathbf{X})\bigr]-\mathbb{E}\bigl[D^c\bigl(G(\mathbf{Z})\bigr)\bigr].
\end{equation}

\section{An Optimal Transport-based Design of GAN Players for GMMs}
Consider the W2GAN minimax problem \eqref{Eq: general WGAN} for learning a GMM. Solving the minimax problem over the original class of $c$-concave $D$'s will be statistically and computationally complex \cite{arora2017generalization,feizi2017understanding}. Therefore, we need to characterize appropriate function spaces for learning $G$ and $D$. 
To obtain a tractable minimax optimization problem, these functions need to be optimized over parameterized sets of functions with bounded statistical complexity. 

To find a tractable generator set $\mathcal{G}$ for $k$-component GMMs, we propose a random linear mapping that produces only mixtures of $k$ Gaussians. Here, we consider a randomized mapping of the Gaussian input $\mathbf{Z}\sim \mathcal{N}(\mathbf{0},I_{d\times d})$, specified by matrices $\Lambda_i\in \mathbb{R}^{d\times d}$, vectors $\boldsymbol{\mu}_i\in\mathbb{R}^d$, and random $Y\in\mathcal{Y}=\{1,\ldots,k\}$ distributed as $\Pr(Y=i)=\pi_i$ for each $1\le i\le k$:
\begin{equation}\label{Eq: Generator_general}
G(\mathbf{z})= \sum_{i=1}^k \mathbb{I}(Y=i)\bigl(\Lambda_i\mathbf{z}+\boldsymbol{\mu}_i\bigr).    
\end{equation}
Here $\mathbb{I}(\cdot)$ denotes the indicator function which is equal to $1$ if the input outcome holds and $0$ otherwise. Therefore, the above generator can output any $k$-component GMM. If the underlying GMM is uniformly distributed among its components, we can further take $Y$ to be uniform, that is, $\pi_i=\frac{1}{k}$ for each $i$. If the Gaussian components are also assumed to share the same covariance matrix, we can use the same $\Lambda_i=\Lambda$ in the formulation. In the special case of a symmetric two-component underlying GMM with opposite means, $G$ can be reduced to the following randomized mapping where $Y$ is uniform on $\{-1,+1\}$:
\begin{equation}\label{Eq: Generator_binary}
    G(\mathbf{z})=Y(\Lambda \mathbf{z}+\boldsymbol{\mu}).
\end{equation}

To find an appropriate discriminator set $\mathcal{D}$, we can consider the set of all discriminator functions that are optimal for two GMMs in \eqref{Eq: general WGAN}'s maximization problem. Therefore, we need to analytically characterize an optimal $D$ in \eqref{Eq: general WGAN} for two GMMs. Since we consider the quadratic cost $c(\mathbf{x},\mathbf{x}')=\frac{1}{2}\Vert\mathbf{x}-\mathbf{x}'\Vert^2_2$, we can apply Brenier's theorem from the optimal transport theory literature.
\begin{lemma}[Brenier's theorem, \cite{villani2008optimal}]
Suppose the generator's distribution $P_{G(\mathbf{Z})}$ has finite first-order moment, i.e. $\mathbb{E}[\Vert G(\mathbf{Z})\Vert_2]<\infty$, and is absolutely continuous with respect to the data distribution. Then for $c(\mathbf{x},\mathbf{x}')=\frac{1}{2}\Vert\mathbf{x}-\mathbf{x}'\Vert^2_2$, the optimal $D^*$ in \eqref{Eq: general WGAN} satisfies
\begin{equation}
    \mathbf{X}-\nabla D^*(\mathbf{X}) \stackrel{\tiny\text{\rm dist}}{=} G(\mathbf{Z}),
\end{equation}
where $ \stackrel{\tiny\text{\rm dist}}{=}$ means the two random vectors have the same probability distribution.
\end{lemma}
The above result implies that the optimal discriminator's gradient $\nabla D^*(\mathbf{x})$ provides an optimal transport map between the two GMMs. However, characterizing the precise optimal transport map between two GMMs is known to be challenging~\cite{chen2018optimal}. To address this issue, we use an idea that is adapted from reference~\cite{gozlan2020mixture}'s randomized transportation map between two distributions. As we show here, we can find such a randomized transportation between two GMMs uniformly distributed among their components. This randomized map is employed to obtain a deterministic transport map and bound its approximation error in the W2GAN problem. 

Given $\mathbf{X},\widetilde{\mathbf{X}}$ distributed according to two GMMs, uniformly distributed among their components and parameterized by $(\boldsymbol{\mu}_i,\Sigma_i)$'s and $(\widetilde{\boldsymbol{\mu}}_i,\widetilde{\Sigma}_i)$'s, consider the following randomized transportation map from $\mathbf{X}$ to $\widetilde{\mathbf{X}}$:
\begin{align}
    \Psi(\mathbf{X},Y) := \sum_{i=1}^k \mathbb{I}(Y=i)\bigl(\Gamma_i(\mathbf{X}-\boldsymbol{\mu}_i)+\widetilde{\boldsymbol{\mu}}_i \bigr).
\end{align}
where $\Gamma_i = \widetilde{\Sigma}_i^{1/2}\Sigma_i^{-1/2}$ converts the covariance matrix for the $i$-th component and $Y$ represents the random component label for $\mathbf{X}$. Note that $\Psi(\mathbf{X},Y)$ has the same distribution as the GMM $P_{\widetilde{\mathbf{X}}}$. However, $\Psi$ is a function of both $\mathbf{X}$ and $Y$, so to obtain a deterministic mapping of $\mathbf{X}$, we take the conditional expectation:
\begin{equation}\label{def: psi function} 
    \psi(\mathbf{x}) := \mathbb{E}\bigl[\Psi(\mathbf{X},Y)\,\vert\,\mathbf{X}=\mathbf{x}\bigr] = \sum_{i=1}^k\bigl[ \Pr(Y=i\,\vert\,\mathbf{X}=\mathbf{x})(\Gamma_i(\mathbf{x}-\boldsymbol{\mu}_i)  +\widetilde{\boldsymbol{\mu}}_i)\bigr].
\end{equation}
%given $\mathbf{X}=\mathbf{x}$ the output of $\Psi$ is a function of random $Y$ conditioned to $\mathbf{X}=\mathbf{x}$ and is hence a random vector. In order
The following theorem bounds the approximation error of considering the above transportation map in the W2GAN problem.
\begin{thm}\label{Thm: Approimxating Wasserstein Distance}
Consider the W2GAN problem \eqref{Eq: general WGAN} with quadratic cost $c(\mathbf{x},\mathbf{x}')=\frac{1}{2}\Vert \mathbf{x}-\mathbf{x}'\Vert_2^2$. Assume $\psi$ defined in \eqref{def: psi function} is the gradient of a convex function $\phi$. Then, the following inequalities hold for $\widetilde{D}(\mathbf{x})=\frac{1}{2}\Vert \mathbf{x}\Vert_2^2 - \phi(\mathbf{x})$:
\begin{equation}
    0\, \le \, W_c(P_{\mathbf{X}},P_{\widetilde{\mathbf{X}}}) - \bigl\{\mathbb{E}[\widetilde{D}(\mathbf{X})] - \mathbb{E}[\widetilde{D}^c(\widetilde{\mathbf{X}})]\bigr\} \, \le \, \left(\frac{3}{2}M_1+\sqrt{M_1 M_2}\right)\sqrt{P_e} + \sqrt{M_1 M_2} \sqrt[4]{P_e},
\end{equation}
where $P_e=\Pr(Y^{\text{\rm opt}}(\mathbf{X})\neq Y)$ is the probability of miclassification of the Bayes classifier $Y^{\text{\rm opt}}(\mathbf{X})$ for predicting label $Y$ from $\mathbf{X}$. We define $M_1,M_2$ in the following equations with $\Vert\cdot\Vert_\sigma$ denoting the maximum singular value, i.e., the spectral norm,
\begin{align*}
    M_1 \, & =\, 8\max_{i} \Vert\Gamma_i\Vert^2_{\sigma} \sqrt{\mathbb{E}[\Vert\mathbf{X}\Vert_2^4]} \, + \, 8\sqrt{P_e}\max_i \Vert \Gamma_i\boldsymbol{\mu}_i - \widetilde{\boldsymbol{\mu}}_i \Vert^2_2, \\
    M_2\, &= \, 2\max_i \Vert \Gamma_i-I \Vert^2_{\sigma} \mathbb{E}[  \Vert\mathbf{X}\Vert_2^2] \, + \, 2\max_i\Vert \Gamma_i\boldsymbol{\mu}_i -\widetilde{\boldsymbol{\mu}}_i\Vert^2_2.
\end{align*}
\end{thm}
\begin{proof}
We defer the proof to the Appendix.
\end{proof}
The above theorem suggests that the transportation map $\psi$ in \eqref{def: psi function} provides an approximation of the original optimal transport cost $W_c(P_{\mathbf{X}},P_{\widetilde{\mathbf{X}}})$ with an approximation error vanishing with the Bayes error of classifying the GMM's components. Consequently, Theorem \ref{Thm: Approimxating Wasserstein Distance} implies that if the Gaussian components in the data distribution are well-separated, the integral of the approximate transportation map provides a near optimal discriminator in the W2GAN problem. Our next result characterizes a parametric set of functions whose gradient can approximate the above transportation map $\psi$. %We defer the proof of this result to the Appendix.
\begin{prop}\label{Prop 1}
Consider GMM vectors $\mathbf{X}, \widetilde{\mathbf{X}}$ with parameters $(\boldsymbol{\mu}_i,\Sigma_i)$'s and $(\widetilde{\boldsymbol{\mu}}_i,\widetilde{\Sigma}_i)$'s. Suppose $\Sigma_i$ and $\widetilde{\Sigma}_i$ commute for each $i$, i.e., $\Sigma_i \widetilde{\Sigma}_i = \widetilde{\Sigma}_i \Sigma_i $. For every $\mathbf{x}$, assume $\sum_{i=1}^k \vert \Pr(Y=i\,\vert\,\mathbf{X}=\mathbf{x}) - \Pr(\bar{Y}=i\,\vert\,\bar{\mathbf{X}}=\mathbf{x})\vert \le \epsilon$ is satisfied when $\bar{\mathbf{X}}$ and its label $\bar{Y}$ have parameters $(\boldsymbol{\mu}_i , \widetilde{\Sigma}^{-1/2}_i\Sigma^{1/2}_i)$ and when $\bar{\mathbf{X}},\bar{Y}$ have parameters $(\widetilde{\boldsymbol{\mu}}_i , I)$. 
Then, a parameterized $D$ exists with the form
\begin{equation}\label{Eq: Discriminator_general}
    D_{(A_i)_{i=1}^k,(\mathbf{b}_i,c_i)_{i=1}^{2k}}(\mathbf{x}) = \log\,\biggl(\,\frac{\sum_{i=1}^k \exp\bigl(\frac{1}{2}\mathbf{x}^TA_i\mathbf{x}+\mathbf{b}_i^T\mathbf{x}+c_i\bigr)}{\sum_{i=k+1}^{2k} \exp\bigl(\mathbf{b}_i^T\mathbf{x}+c_i\bigr)}\,\biggr),
\end{equation}
that satisfies the following inequality for $\psi$ defined in \eqref{def: psi function}: \begin{align*}\mathbb{E}[\Vert \psi(\mathbf{X}) - \nabla  D(\mathbf{X}) \Vert_2]
\le \: \epsilon \left( \max_{i,j,l}\, \Vert \boldsymbol{\mu}_i\Vert_2 +  \sqrt{\Vert\widetilde{\Sigma}_j\Sigma^{-1}_j\Vert_{\sigma}}\mathbb{E}[\Vert\mathbf{X}\Vert_2]+ \sqrt{\Vert\widetilde{\Sigma}_l\Sigma^{-1}_l\Vert_{\sigma}}\Vert\widetilde{\boldsymbol{\mu}}_l\Vert_2\bigr) \right).
\end{align*}
\end{prop}
\begin{proof}
We defer the proof to the Appendix.
\end{proof}
Proposition \ref{Prop 1} implies that if the conditional distribution $P_{Y|\mathbf{X}}$ for $\mathbf{X},\widetilde{\mathbf{X}}$ can be well-approximated using the means of $\mathbf{X}$, then the mapping $\psi$ can be captured by a softmax-based function in \eqref{Eq: Discriminator_general}. As a  special case, in the Appendix we show that Proposition \ref{Prop 1}'s assumption when $\mathbf{X},\widetilde{\mathbf{X}}$ are symmetric mixtures of two well-separable Gaussians with means $\boldsymbol{\mu},\widetilde{\boldsymbol{\mu}}$ translates to $\boldsymbol{\mu}^T\widetilde{\boldsymbol{\mu}}>O(\log(1/\epsilon))$.
%Remark \ref{Remark 1} further simplifies this parametric function for GMMs with a common covariance.
\begin{remark}\label{Remark 1}
In the setting of Proposition \ref{Prop 1}, assume that each GMM's components share the same covariance matrix. Then, Proposition \ref{Prop 1} remains valid for the following softmax-based quadratic $D_{A,(\mathbf{b}_i,c_i)_{i=1}^{2k}}$. For a symmetric mixture of two Gaussians, constant $c_i$'s can also be removed while the approximation guarantee still applies.
\begin{equation}\label{Eq: Discriminator_Uniform}
    D_{A,(\mathbf{b}_i,c_i)_{i=1}^{2k}}(\mathbf{x}) = \frac{1}{2}\mathbf{x}^T A\mathbf{x}+\log\,\biggl(\,\frac{\sum_{i=1}^k \exp(\mathbf{b}_i^T\mathbf{x}+c_i)}{\sum_{i=k+1}^{2k} \exp(\mathbf{b}_i^T\mathbf{x}+c_i)}\,\biggr)
\end{equation}
%In the case of a symmetric mixture of two Gaussians, we can further remove the constants and achieve the same approximation guarantee via
%\begin{equation}\label{Eq: Discriminator_Uniform_Binary}
%    D_{A,(\mathbf{b}_i)_{i=1}^{4}}(\mathbf{x}) = \frac{1}{2}\mathbf{x}^T A\mathbf{x}+\log\,\bigl(\,\frac{\exp(\mathbf{b}_1^T\mathbf{x})+ \exp(\mathbf{b}_2^T\mathbf{x})}{\exp(\mathbf{b}_3^T\mathbf{x})+ \exp(\mathbf{b}_4^T\mathbf{x})}\,\bigr).
%\end{equation}
\end{remark}
We note that given a shared identity covariance for the two GMMs, \eqref{Eq: Discriminator_Uniform} reduces to the difference of two softmax functions, which revisits \cite{bai2018approximability}'s proposed architecture for this special case.  
Proposition \ref{Prop 1} combined with Theorem \ref{Thm: Approimxating Wasserstein Distance} shows that the proposed softmax-based quadratic architecture results in an approximate optimal transport map between the two GMMs. Therefore, we use the architectures in  \eqref{Eq: Generator_general} and \eqref{Eq: Discriminator_Uniform} for learning mixtures of Gaussians with a common covariance.

\section{GAT-GMM: A Minimax GAN Framework for Learning GMMs}
As shown earlier, we can constrain the generator and discriminator to the class of functions specified by \eqref{Eq: Generator_general} and \eqref{Eq: Discriminator_Uniform} to approximate the solution to the W2GAN problem when the underlying GMM has common covariance across components. Considering the W2GAN minimax objective in \eqref{Eq: general WGAN}, Proposition \ref{Proposition: c-transform reduction} suggests regularizing the expected c-transform in order to reduce the computational complexity of optimizing the c-transform function. We later show that this regularization will lead to a concave objective in the discriminator maximization problem which can be efficiently solved by a gradient descent ascent (GDA) algorithm.
\begin{prop}\label{Proposition: c-transform reduction}
Consider the discriminator function $D_{A,(\mathbf{b}_i,c_i)_{i=1}^{2k}}$ defined in \eqref{Eq: Discriminator_Uniform}. For constant $\eta > 0$, assume $\lambda_{\max}(A)+2\max_i \Vert \mathbf{b}_i \Vert^2_2\le \eta < 1$ where $\lambda_{\max}(\cdot)$ denotes the maximum eigenvalue. Then, for any set of vectors $(\mathbf{d}_i)_{i=1}^k$ and constants $(e_i)_{i=1}^k$ we have
\begin{align}\label{Eq: Prop 2 D^c upperbound}
     &\mathbb{E}\bigl[ D^c_{A,(\mathbf{b}_i,c_i)_{i=1}^{2k}}(\mathbf{X}) \bigr] \, \le \, \mathbb{E}\bigl[ D_{A,(\mathbf{b}_i,c_i)_{i=1}^{2k}}(\mathbf{X})\bigr] \\
     &\quad + \frac{3k^2(\mathbb{E}[\Vert\mathbf{X}\Vert^2_2]+1)}{1-\eta}\biggl(\Vert A\Vert^2_F + \sum_{i=1}^k \bigl[ \Vert\mathbf{b}_i - \mathbf{d}_i \Vert^2_2 + \Vert\mathbf{b}_{k+i} - \mathbf{d}_i \Vert^2_2 +(c_i-e_i)^2+ (c_{k+i}-e_i)^2\bigr] \biggr). \nonumber
\end{align}
\end{prop}
\begin{proof}
We defer the proof to the Appendix.
\end{proof}
Replacing the c-transform expectation in \eqref{Eq: general WGAN} with its upper-bound in \eqref{Eq: Prop 2 D^c upperbound}, we reach the following regularized minimax problem for learning GMMs with a common covariance. We call the minimax framework \emph{Generative Adversarial Training for Gaussian Mixture Models (GAT-GMM)}:
\begin{align}\label{Eq: GM-GAN: general case}
 \min_{{\Lambda,(\boldsymbol{\mu}_i)_{i=1}^k}}\;\max_{{A,(\mathbf{b}_i,c_i)_{i=1}^{2k}}}\; &\mathbb{E}\bigl[D_{A,(\mathbf{b}_i,c_i)_{i=1}^{2k}}(\mathbf{X})\bigr] - \mathbb{E}\bigl[D_{A,(\mathbf{b}_i)_{i=1}^{2k}}\bigl(G_{\Lambda,(\boldsymbol{\mu}_i)_{i=1}^k} (\mathbf{Z})\bigr)\bigr]  \\
 & \;\; - \frac{\lambda}{2} \biggl(\Vert A\Vert^2_F + \sum_{i=1}^k \bigl[ \Vert\mathbf{b}_i - \mathbf{d}_i \Vert^2_2 + \Vert\mathbf{b}_{i+k} - \mathbf{d}_i \Vert^2_2 +(c_i-e_i)^2+ (c_{k+i}-e_i)^2\bigr] \biggr).   \nonumber
\end{align}
Here $\lambda $ denotes the coefficient of the regularization term suggested by Proposition \ref{Proposition: c-transform reduction}. Also, $(\mathbf{d}_i,e_i)$'s represent a fixed set of real vectors and constants. The $\mathbf{d}_i$'s should be chosen to make the underlying GMM well-separable along their directions; we discuss how to select $\mathbf{d}_i$'s for symmetric mixtures of two Gaussians in the next section. To solve \eqref{Eq: GM-GAN: general case}, we propose a GDA algorithm where we iteratively apply one step of gradient descent for minimization followed by one step of gradient ascent for maximization. The following theorem provides optimization guarantees for the convergence of this algorithm to approximate stationary minimax points. In this theorem, we use $\mathcal{L}(\Lambda,(\boldsymbol{\mu}_i)_{i=1}^k)$ to denote the optimal value of the discriminator objective for generator parameters $\Lambda,(\boldsymbol{\mu}_i)_{i=1}^k$. Also, $\operatorname{vec}(\cdot)$ denotes the vector concatenating the inputs' entries. %We defer the theorem's proof to the Appendix.  
\begin{thm}
Consider the GAT-GMM minimax problem \eqref{Eq: GM-GAN: general case} with the constraint that $\Vert \Lambda \Vert^2_F + \max_i \Vert \boldsymbol{\mu}_i\Vert_2^2 +1 \le \eta$. Suppose $\mathbb{E}[\Vert \mathbf{X} \Vert^2_2] \le \eta < \frac{\lambda}{2}$. Then, the GDA algorithm with maximization and minimization stepsizes $\alpha_{\max}=\frac{1}{\lambda+2\eta}$ and $\alpha_{\min}= \frac{1}{\kappa^2L}$ for $L=2\lambda+4\eta + 10(k+1)(\frac{\eta}{\lambda} +\max_i\Vert\mathbf{d}_i\Vert^2_2)$ and $\kappa=\frac{L}{\lambda-2\eta}$ will find an approximate stationary point such that $\bigl\Vert \nabla_{\operatorname{vec}(\Lambda,(\boldsymbol{\mu}_i)_{i=1}^k)} \mathcal{L}\bigl(\Lambda,(\boldsymbol{\mu}_i)_{i=1}^k\bigr) \bigr\Vert_2 \le \epsilon$ over $\mathcal{O}\bigl(\frac{\kappa L ((2\eta/\lambda)^2 + \kappa)}{\epsilon^2} \bigr)$ iterations.
\end{thm}
\begin{proof}
We defer the proof to the Appendix.
\end{proof}
The above theorem shows that the GAT-GMM's minimax problem, which reduces to a non-convex concave minimax optimization problem \cite{lin2019gradient} for $\lambda$ values characterized in the theorem, can be efficiently solved by GDA to obtain an approximate stationary minimax point characterized by \cite{jin2019minmax}.

\section{Theoretical Guarantees for Symmetric Two-component GMMs}
Here, we focus on the application of GAT-GMM for learning a symmetric mixture of two Gaussians  $\frac{1}{2}\mathcal{N}(\boldsymbol{\mu}_{\mathbf{X}},\Sigma_{\mathbf{X}})+\frac{1}{2}\mathcal{N}(-\boldsymbol{\mu}_{\mathbf{X}},\Sigma_{\mathbf{X}})$. This benchmark class of GMMs has been well studied in the literature for analyzing the convergence properties of the EM algorithm \cite{daskalakis2016ten,xu2016global}. We similarly provide theoretical guarantees for GAT-GMM in this benchmark setting.  To apply GAT-GMM \eqref{Eq: GM-GAN: general case} for learning such a GMM, we use the discriminator architecture in \eqref{Eq: Discriminator_Uniform} with a norm-squared regularization penalty such that $\mathbf{d}_1 = - \mathbf{d}_2 = \mathbf{d}$ to obtain the minimax problem
\begin{align}\label{Eq: GMGAN Binary Case}
   \min_{\Lambda,\boldsymbol{\mu}}\;\max_{A,(\mathbf{b}_i)_{i=1}^4}\; &\mathbb{E}\bigl[D_{A,(\mathbf{b}_i)_{i=1}^4}(\mathbf{X})\bigr] - \mathbb{E}\bigl[D_{A,(\mathbf{b}_i)_{i=1}^4}\bigl(G_{\Lambda,\boldsymbol{\mu}} (\mathbf{Z})\bigr)\bigr] \nonumber \\
   &\quad - \frac{\lambda}{2} \biggl(\Vert A\Vert^2_F + \Vert\mathbf{b}_1 - \mathbf{d} \Vert^2_2 + \Vert\mathbf{b}_2 + \mathbf{d} \Vert^2_2 + \Vert\mathbf{b}_3 - \mathbf{d} \Vert^2_2 + \Vert\mathbf{b}_4 + \mathbf{d} \Vert^2_2 \biggr).     
\end{align}
In the following discussion, we use $\mathcal{L}(G_{\Lambda,\boldsymbol{\mu}})$ to denote the optimal maximum discriminator objective in \eqref{Eq: GMGAN Binary Case}. Theorem \ref{Thm: GM-GAN approximation} proves that the global solution to \eqref{Eq: GMGAN Binary Case} will result in the underlying GMM as long as $\mathbf{X}$'s projection along $\mathbf{d}$ satisfies $\sigma_d+2\sigma^2_d\le \mu_d$ given its mean $\mu_d:={\mathbf{d}}^T\boldsymbol{\mu}_{\mathbf{X}} $ and variance $\sigma^2_d:={\mathbf{d}}^T\Sigma_{\mathbf{X}}{\mathbf{d}}$ parameters. This condition, which is formally stated below, requires sufficient separability among the mixture components in the direction of vector $\mathbf{d}$ as it assumes the signal-to-noise ratio $\frac{\mu_d}{\sigma_d}$ to be greater than $1+2\sigma_d$. In our numerical experiments, we chose $\mathbf{d}$ as the principal eigenvector of empirical $\mathbb{E}[\mathbf{X}\mathbf{X}^T]$ and numerically validated that the condition holds for the experiment's GMM.
%We defer the theorem's proof to the Appendix. 
\begin{condition}\label{Condition: SNR}
For mean and covariance parameters $(\boldsymbol{\mu},\Sigma)$ and vector $\widetilde{\mathbf{d}}$, the following inequality holds for the mean's and covariance's projections along vector $\widetilde{\mathbf{d}}$:
\begin{equation}\label{Eq: SNR condition}
    2\widetilde{\mathbf{d}}^T\Sigma\widetilde{\mathbf{d}}+\sqrt{\widetilde{\mathbf{d}}^T\Sigma\widetilde{\mathbf{d}}} \, \le \, \bigl\vert \boldsymbol{\mu}^T\widetilde{\mathbf{d}} \bigr\vert.
\end{equation}
\end{condition}
\begin{thm}\label{Thm: GM-GAN approximation}
Consider the minimax problem in \eqref{Eq: GMGAN Binary Case} for learning a symmetric two-component GMM $\frac{1}{2}\mathcal{N}(\boldsymbol{\mu}_{\mathbf{X}},\Sigma_{\mathbf{X}})+\frac{1}{2}\mathcal{N}(-\boldsymbol{\mu}_{\mathbf{X}},\Sigma_{\mathbf{X}})$. Suppose that $(\boldsymbol{\mu}_{\mathbf{X}},\Sigma_{\mathbf{X}})$ satisfies Condition \ref{Condition: SNR} along vector $\mathbf{d}$. Then, $(\boldsymbol{\mu},\Sigma)=(\boldsymbol{\mu}_{\mathbf{X}},\Sigma_{\mathbf{X}})$ is the only minimizer of $\mathcal{L}(G_{\Lambda,\boldsymbol{\mu}})$ satisfying Condition \ref{Condition: SNR} along $\mathbf{d}$.
% done \textcolor{red}{Possible rewording: ``... is the only minimizer of $\mathcal{L}(G_{\Lambda,\boldsymbol{\mu}})$ satisfying Condition \ref{Condition: SNR} along $\mathbf{c}$.''}
\end{thm}
\begin{proof}
We defer the proof to the Appendix.
\end{proof}
Theorem \ref{Thm: GM-GAN approximation} explains that while we are constraining discriminator $D$ in \eqref{Eq: GMGAN Binary Case} to the class of softmax-based quadratic functions, the global solution to GAT-GMM's minimax problem still matches an underlying GMM with well-separable components  along $\mathbf{d}$. The next theorem
indicates that the underlying $P_{\mathbf{X}}$ is the only GMM with well-separable components along $\mathbf{d}$ that is also a  minimax stationary point in the GAT-GMM problem. Therefore, constraining $G$ to GMMs with well-separable components along $\mathbf{d}$ results in only one minimax stationary solution which is the underlying GMM.
%We defer the result's proof to the Appendix.
\begin{thm}\label{Thm: GM-GAN Optimization}
Consider the setting of Theorem \ref{Thm: GM-GAN approximation}. Suppose that for every feasible $G_{\Lambda,\boldsymbol{\mu}}$ Condition \ref{Condition: SNR} holds for $(\boldsymbol{\mu}, \Sigma)$ along every $\widetilde{\mathbf{d}}$ in a $\frac{\rho}{\lambda}$-distance from $\mathbf{d}$, i.e., $\Vert \widetilde{\mathbf{d}} - \mathbf{d}\Vert_2\le \frac{\rho }{\lambda}$, with $\rho$ being the maximum $\mathbb{E}[\Vert G_{\Lambda,\boldsymbol{\mu}}(\mathbf{Z})\Vert_2]$ over feasible $G_{\Lambda,\boldsymbol{\mu}}$'s. Suppose  $\mathbb{E}[\Vert\mathbf{X}\Vert^2_2]+\mathbb{E}[\Vert G_{\Lambda,\boldsymbol{\mu}}(\mathbf{Z})\Vert^2_2]\le\lambda$ for feasible $\Lambda,\boldsymbol{\mu}$'s. Then, $(\boldsymbol{\mu}_{\mathbf{X}},\Sigma_\mathbf{X})$ is the only stationary point in the feasible set with $\nabla_{\operatorname{vec}(\boldsymbol{\mu},\Lambda)} \mathcal{L}(G_{\Lambda,\boldsymbol{\mu}}) = \mathbf{0}$.
\end{thm}
\begin{proof}
We defer the proof to the Appendix.
\end{proof}
Finally, Theorem \ref{Thm: GM-GAN Generalization} bounds the generalization error of estimating the GAT-GMM's minimax objective and its derivatives from empirical samples. This result implies that a stationary point of the empirical objective will lead to an approximate stationary point of the underlying objective given $\mathcal{O}(d)$ samples. 
% defer theorem \ref{Thm: GM-GAN Generalization}'s proof to the Appendix.
\begin{thm}\label{Thm: GM-GAN Generalization}
Consider the setting of Theorem 4 with the additional constraints $\mathbf{b}_1=-\mathbf{b}_2,\, \mathbf{b}_3=-\mathbf{b}_4$, which do not change the optimal solution to \eqref{Eq: GMGAN Binary Case}. Let $\widehat{\mathcal{L}}(G_{\Lambda,\boldsymbol{\mu}})$ denote the empirical objective for $n$ i.i.d. $\mathbf{x}_i$'s sampled from $P_{\mathbf{X}}$. Let $ \frac{\partial{\mathcal{L}}(G_{\Lambda,\boldsymbol{\mu}})}{\partial \Lambda},\, \frac{\partial{\mathcal{L}}(G_{\Lambda,\boldsymbol{\mu}})}{\partial \boldsymbol{\mu}}$ denote the objective's derivative with respect to $\Lambda ,\, \boldsymbol{\mu}$. Then, for every $\delta>0$, with probability at least $1-\delta$ the following bounds hold uniformly for every $\boldsymbol{\mu},\Lambda$ such that $\Vert\boldsymbol{\mu}\boldsymbol{\mu}^T +\Lambda\Lambda^T \Vert_F\le \mathbb{E}[\Vert\mathbf{X}\Vert^2_2]$:
\begin{align}
    &\big\vert \widehat{\mathcal{L}}(G_{\Lambda,\boldsymbol{\mu}})- {\mathcal{L}}(G_{\Lambda,\boldsymbol{\mu}})\big\vert \le \mathcal{O}\bigl(\sqrt{\frac{d^2\log(1/\delta) }{\lambda^2 n}}\bigr), \\
    &\big\Vert \frac{\partial\widehat{\mathcal{L}}(G_{\Lambda,\boldsymbol{\mu}})}{\partial \boldsymbol{\mu}}(\boldsymbol{\mu})- \frac{\partial\mathcal{L}(G_{\Lambda,\boldsymbol{\mu}})}{\partial \boldsymbol{\mu}}(\boldsymbol{\mu})\big\Vert_{2} + \big\Vert \frac{\partial\widehat{\mathcal{L}}(G_{\Lambda,\boldsymbol{\mu}})}{\partial \Lambda}(\Lambda)- \frac{\partial{\mathcal{L}}(G_{\Lambda,\boldsymbol{\mu}})}{\partial \Lambda}(\Lambda)\big\Vert_{\sigma} \, \le \, \mathcal{O}\bigl(\sqrt{\frac{d\log(1/\delta) }{\lambda^2 n}}\bigr).\nonumber
\end{align}
\end{thm}
\begin{proof}
We defer the proof to the Appendix.
\end{proof}

\section{Numerical Experiments}
\begin{figure}[t]
    \centering
    \includegraphics[width=1.0\textwidth]{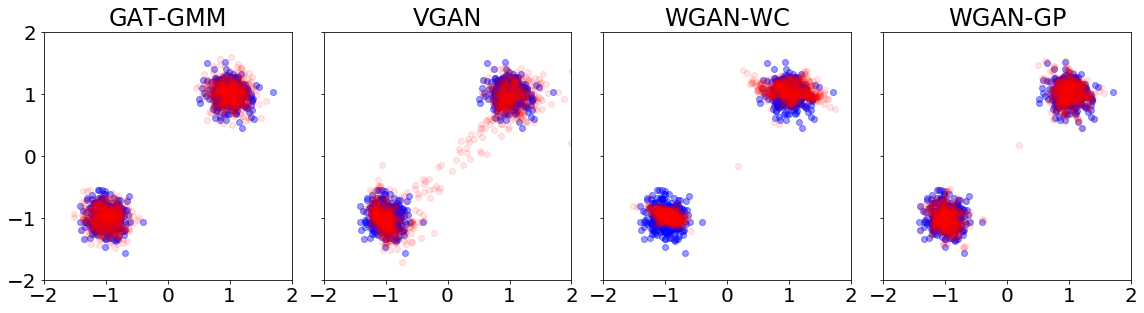}
    \caption{Samples produced from GAT-GMM, VGAN, WGAN-WC, WGAN-GP on isotropic task.}\label{fig:isotropic-samples}
\end{figure}

%\subsection{Setup}
We considered two datasets of symmetric, 2-component GMMs. The first was a well-separated, 20-dimensional GMM with isotropic covariance. The second was a 100-dimensional GMM with a randomly rotated and scaled covariance, which is more difficult to learn due to high dimensionality and lack of axis-alignment \cite{ashtiani2018nearly}. We numerically verified for both datasets that Condition \ref{Condition: SNR} holds along the top eigenvector of the empirical $\mathbb{E}[\mathbf{X}\mathbf{X}^T]$.
%For each dataset, we applied our proposed GAT-GMM formulation and compared with standard GANs and the EM algorithm. 
For each dataset, we trained GAT-GMM using alternating gradient descent ascent. We compared GAT-GMM with the EM algorithm and the following standard neural net-based GANs: vanilla GAN (VGAN) \cite{goodfellow2014generative}, spectrally-normalized GAN (SN-GAN) \cite{miyato2018spectral}, Pac-VGAN and Pac-SNGAN \cite{lin2018pacgan}, GM-GAN \cite{ben2018gaussian}, and Wasserstein GAN with weight clipping (WGAN-WC) \cite{arjovsky2017wasserstein} and gradient penalty (WGAN-GP) \cite{gulrajani2017improved}. Full details on the datasets and grid of hyper-parameter choices, as well as additional experiments for mixtures with more than 2 components, are provided in the Appendix.

To numerically evaluate our model, we computed the negative log-likelihood of the samples produced. In addition, for GAT-GMM and EM, we computed the following \emph{GMM Objective} which is the minimum of the 2-Wasserstein costs $W_c(\mathcal{N}(\boldsymbol{\mu},\Sigma),\mathcal{N}(\hat{\boldsymbol{\mu}},\hat{\Sigma}))$ and $W_c(\mathcal{N}(\boldsymbol{\mu},\Sigma),\mathcal{N}(-\hat{\boldsymbol{\mu}},\hat{\Sigma}))$:
\begin{equation}\label{Eq: Accuracy}
    %\|\Sigma^{1/2}-\hat{\Sigma}^{1/2}\|_F^2+\min\{\|\boldsymbol{\mu}-\hat{\boldsymbol{\mu}}\|^2.\|\boldsymbol{\mu}+\hat{\boldsymbol{\mu}}\|^2\}
   \operatorname{Tr}\bigl(\Sigma+\hat{\Sigma} - 2(\Sigma^{1/2}\hat{\Sigma}\Sigma^{1/2})^{1/2}\bigr)+\min\{\|\boldsymbol{\mu}-\hat{\boldsymbol{\mu}}\|^2,\|\boldsymbol{\mu}+\hat{\boldsymbol{\mu}}\|^2\}.
\end{equation}
For the baseline GANs where we did not have access to exact $\hat{\boldsymbol{\mu}}$ and $\hat{\Sigma}$, we used the trained generator to generate $10^5$ samples, divided the samples into those falling into the positive and negative orthants, computed the empirical means $\hat{\boldsymbol{\mu}}_1,\hat{\boldsymbol{\mu}}_2$ and covariances $\hat{\Sigma}_1,\hat{\Sigma}_2$ within the orthants, and estimated the GMM Objective using:
\begin{equation}\label{Eq: NN Accuracy}
\frac{1}{2}\sum_{i=1}^2\left[\operatorname{Tr}\bigl(\Sigma+\hat{\Sigma}_i - 2(\Sigma^{1/2}\hat{\Sigma}_i\Sigma^{1/2})^{1/2}\bigr)+\min\{\|\boldsymbol{\mu}-\hat{\boldsymbol{\mu}}_i\|^2,\|\boldsymbol{\mu}+\hat{\boldsymbol{\mu}}_i\|^2\}\right].
\end{equation}
%\frac{1}{2}\bigl[\|\Sigma^{1/2}-\hat{\Sigma}^{1/2}\|_F^2+\|\Sigma^{1/2}-\hat{\Sigma}_2^{1/2}\|_F^2\bigr]+\bigl[\min_{\pi\in S_2} \sum_{i=1}^2\frac{1}{2} \|\boldsymbol{\mu}_{\pi(i)}-\hat{\boldsymbol{\mu}}_i\|^2\bigr]
%where $S_2$ is the set of permutations of $(1,2)$. When the learned GMM is well-separated, this provides a reasonable comparison metric with (\ref{Eq: Accuracy}).

\begin{figure}[t]
    \centering
    \includegraphics[width=1.0\textwidth]{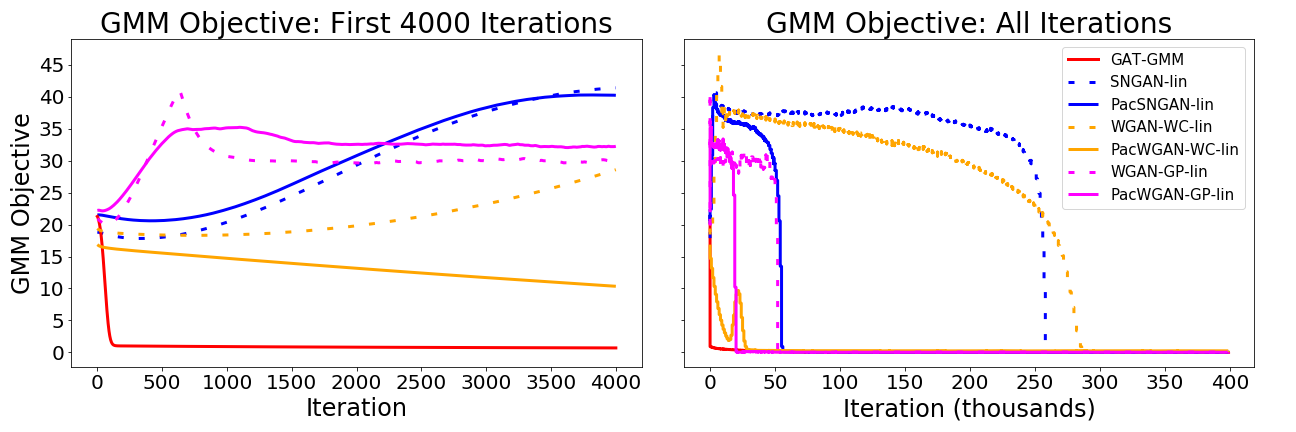}
    \caption{Left: GMM Objective over first 4000 iterations for GAT-GMM and various linear generator (\ref{Eq: Generator_binary}), neural net discriminator GANs. Right: GMM Objective over all training iterations.}\label{fig:loss}
\end{figure}

%\subsection{Results}

Figure \ref{fig:isotropic-samples} plots the samples generated by the trained GAT-GMM, vanilla GAN, WGAN-WC, and WGAN-GP for the isotropic case. The other baseline GANs either performed about as well as WGAN-GP or worse than WGAN-WC and vanilla GAN. Their generated samples, as well as the samples in the rotated covariance case, are displayed in the supplementary material. We observed that the GMM learned by GAT-GMM is visually very close to the underlying GMM. While WGAN-GP was able to qualitatively learn the GMM as well, both WGAN-WC and vanilla GAN produced visibly distinguishable samples. We note that WGAN-WC and vanilla GAN displayed mode collapse for many choices of hyper-parameters.

% Quantitatively, GAT-GMM achieved a GMM objective of \textbf{0.0061} and \textbf{0.862} for isotropic and rotated covariances, respectively, while EM achieved GMM objectives of \textbf{0.0062} and \textbf{0.860}. Thus, GAT-GMM resulted in numerical scores comparable to the EM's scores in those experiments. Meanwhile, the top GMM objective values obtained by baseline GANs were \textbf{0.023} and \textbf{6.081}, both achieved by WGAN-GP. Table \ref{tbl:isotropic_main} provides a full table of the numerical results including the evaluated GMM objectives as well as negative log-likelihoods (NLL). In all our experiments, the standard baseline GANs were unable to approach the EM's performance, indicating that obtaining EM-like performance via a minimax GAN framework is non-trivial in general. From a computational efficiency perspective, GAT-GMM ran quickly on CPUs and achieved good results very quickly with 1-1 gradient descent-ascent (Figure \ref{fig:loss}), while the neural net-based GANs required more discriminator steps per generator step to train properly. Also, we observed that for models with unstable training, PacGAN was able to stabilize optimization and alleviate mode collapse in our experiments. However, this improvement did not result in an EM-like numerical performance.

Quantitatively, GAT-GMM achieved a GMM objective of \textbf{0.0061} and \textbf{0.862} for isotropic and rotated covariances, respectively, while EM achieved GMM objectives of \textbf{0.0062} and \textbf{0.860}. Thus, GAT-GMM resulted in numerical scores comparable to the EM's scores in those experiments. Meanwhile, the top GMM objective values obtained by baseline GANs were \textbf{0.023} and \textbf{6.081}, both achieved by WGAN-GP. Table \ref{tbl:isotropic_main} provides a table of the numerical results for some GANs, including the evaluated GMM objectives as well as negative log-likelihoods (NLL). The full table is provided in the Appendix. 

In all our experiments, the standard baseline GANs were unable to approach EM's performance, indicating that obtaining EM-like performance via a minimax GAN framework is non-trivial in general. From a computational perspective, GAT-GMM ran quickly on CPUs and achieved good results very quickly with 1-1 gradient descent-ascent (Figure \ref{fig:loss}), while neural net-based GANs required more discriminator steps to train properly. To further isolate the computational benefits of our discriminator choice, we also trained the linear generator (\ref{Eq: Generator_binary}) with neural net discriminators. As shown in Figure \ref{fig:loss}, GAT-GMM training is far faster and more stable than neural net GAN training for the same generator class. %While PacGAN stabilized training for the neural net models, it did not significantly improve their numerical performance.

% did not always result in higher numerical accuracy, even when applied to the more successful GANs, such as WGAN-GP.
% The full numerical results are shown in Table (\ref{tbl:isotropic}) in the Appendix. We find that GAT-GMM outperforms the neural network GANs in our GMM objective score for both tasks, and even performs comparably with the EM algorithm. Although we tried several different choices of hyper-parameters for the neural network GAN training, none of these GANs could learn the parameters of these Gaussian Mixture to a comparable accuracy as EM and GAT-GMM. Indeed, the numerical difference in accuracy between GAT-GMM and neural net GANs becomes more pronounced in the rotated covariance task, which indicates that obtaining EM-like performance via minimax optimization is non-trivial in general. From a computational efficiency perspective, GAT-GMM ran quickly on CPU machines and achieved good results even with 1-1 gradient descent-ascent, while most of the neural network GANs required more discriminator steps per generator step to train properly. In addition, while PacGAN was able to stabilize the training and alleviate mode collapse for most of the neural network models, it did not generally result in better numerical accuracy, even when we applied it to the more successful neural net GANs, such as SNGAN.

 \begin{table}[b]
 \centering
 {\renewcommand{\arraystretch}{1.25}%
 \begin{tabular}{|c|cc|cc|}
 \hline
 & \multicolumn{2}{c|}{Isotropic} & \multicolumn{2}{c|}{Rotated}\\ \hline
   Method        & GMM Objective & NLL & GMM Objective & NLL    \\ \hhline{=|==|==}
 GAT-GMM               & \textbf{0.0061}         & -5.873 & 0.862    & \textbf{54.351} \\ 
 EM                  & 0.0062         & -5.966    & \textbf{0.860}    &  54.968    \\ 
 VGAN & 0.338         & 1.048  & 35.445& 180.46 \\ 
 SN-GAN               & 0.027         & -6.257  & 6.111    & 55.183  \\
 WGAN-WC             & 0.225          & \textbf{-7.525} & 16.906    & 64.526  \\ 
 WGAN-GP             & 0.023         & -7.094 & 6.081     & 55.656   \\
 \hline
 \end{tabular}}
 \caption{Numerical results on the isotropic and rotated datasets.}  \label{tbl:isotropic_main}
 \end{table}

\newpage
\bibliographystyle{unsrtnat}

\newpage
\begin{appendices}
\section{Experiment Details and Additional Figures}\label{section:additional-figs}
% Experimental detail
\subsection{Additional Details on Numerical Experiments}
As discussed in the main text, we considered two symmetric, 2-component GMM learning tasks with $n=640$ samples each. The first was an isotropic symmetric GMM with $d=20$, mean $\boldsymbol{\mu}_1=(1,...,1)$, $\boldsymbol{\mu}_2=-\boldsymbol{\mu}_1$, and shared covariance $\Sigma=0.03\times\text{I}_d$. The second was a high-dimensional symmetric GMM with a randomly rotated covariance matrix: we took $d=100$, $\boldsymbol{\mu}_1=(1,...,1)$, $\boldsymbol{\mu}_2=-\boldsymbol{\mu}_1$, and $\Sigma=Q\Lambda Q^T$ where $\Lambda$ is diagonal with entries distributed uniform on $(\frac{1}{2d}, \frac{1}{2})$ and $Q$ is a random orthogonal matrix.

All experiments were implemented using PyTorch \cite{paszke2019pytorch}. For GAT-GMM, we trained with $1$ or $5$ discriminator updates per generator update, all combinations of generator and discriminator learning rates in $\{1e-1,1e-2,1e-3,1e-4\}$, and regularization $\lambda \in \{0.1,1.0,2.0\}$. We initialized the generator mean $\mu$ to be uniform in $(-0.5,0.5)^d$ and covariance parameter $\Lambda =\sigma (I +0.01\times\mathcal{N}(0,1))$, where $\sigma=O(2^{-1/d})$ is chosen small enough so that the components are well-separated. For the discriminator, we initialized $A=I_d+0.01\mathcal{N}(0,1)$ and all of the $\mathbf{b_i}$ to be $0.01\mathcal{N}(0,1)$. We found generator/discriminator learning rate pairs of $(1e-3,1e-2)$ and regularization $\lambda=0.1$ to work best for $d=5$, and $(1e-2,1e-1)$ with regularization $\lambda=2.0$ to work best for $d=1$, although the training was robust to small changes in these hyperparameters.

For the neural network GANs, we used 4-layer nets of width 256 for both generator and discriminator with leaky ReLU activations of slope $0.2$, along with $1$ or $5$ discriminator updates per generator update. Parameters were initialized with the He initialization \cite{he2015delving}. The latent dimension was chosen to be the same as the data dimension. For VGAN, SN-GAN, and GM-GAN, we used the Adam \cite{kingma2014adam} optimizer with learning rates in $\{1e-5,1e-4,1e-3\}$, and other hyperparameters as default PyTorch settings ($\beta_1=0.9,\beta_2=0.999$). For VGAN and SN-GAN, we used batch normalization \cite{ioffe2015batch} with default PyTorch settings ($\epsilon=1e-05$, momentum $0.1$), and for SN-GAN we used spectral normalization with 1 power iteration, the default setting. For WGAN-WC and WGAN-GP, we used learning rates in $\{5e-5,4e-5\}$. WGAN-WC was optimized using RMSProp \cite{tieleman2012lecture} and weight clipping parameter $0.2$, and WGAN-GP was optimized using Adam with default hyperparameters and regularization strength $\{0.1,1.0\}$. The PacGAN variants were all trained with the same original hyperparameters and a packing level of 3. The results reported in this paper are for the best-performing GANs in the GMM objective.
\subsection{Additional Figures and Tables}
The full version of Table \ref{tbl:isotropic_main} is provided here in Table \ref{tbl:isotropic_appendix}. We provide images of the samples for all the models on our isotropic and rotated covariance models in Figures \ref{fig:isotropic-samples}, \ref{fig:rot-samples}, and \ref{fig:collapsed-samples}.

\begin{table}[h]
 \centering
 {\renewcommand{\arraystretch}{1.0}%
 \begin{tabular}{|c|cc|cc|}
 \hline
 & \multicolumn{2}{c|}{Isotropic} & \multicolumn{2}{c|}{Rotated}\\ \hline
   Method        & GMM Objective & NLL & GMM Objective & NLL    \\ \hhline{=|==|==}
 GAT-GMM               & \textbf{0.0061}         & -5.873 & 0.862    & \textbf{54.351} \\ 
 EM                  & 0.0062         & -5.966    & \textbf{0.860}    &  54.968    \\ 
 VGAN & 0.338         & 1.048  & 35.445& 180.46 \\ 
 SN-GAN               & 0.027         & -6.257  & 6.111    & 55.183  \\
 Pac-VGAN              & 0.324         & -5.731  &36.101 &65.136 \\ 
   Pac-SNGAN              & 0.030         & -6.176  & 6.229    & 58.924  \\
 GM-GAN             & 0.180         & \textbf{-9.730}  &33.651 &   75.003   \\ 
 WGAN-WC             & 0.225          & -7.525 & 16.906    & 64.526  \\ 
 WGAN-GP             & 0.023         & -7.094 & 6.081     & 55.656   \\
 \hline
 \end{tabular}}
 \caption{Numerical results on the isotropic and rotated datasets.}  \label{tbl:isotropic_appendix}
 \end{table}
 
\begin{figure}[h!]
    \centering
    \includegraphics[width=.8\textwidth]{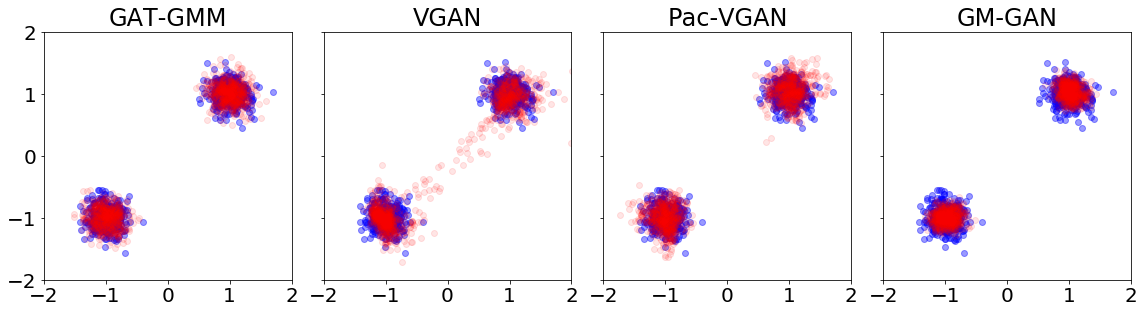}
    \includegraphics[width=.8\textwidth]{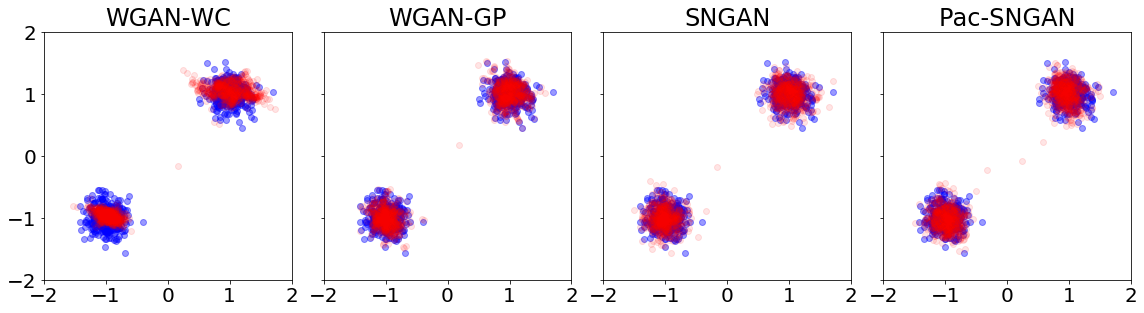}
    \caption{Samples from isotropic task.}\label{fig:isotropic-samples}
\end{figure}
\begin{figure}[h!]
    \centering
    \includegraphics[width=.8\textwidth]{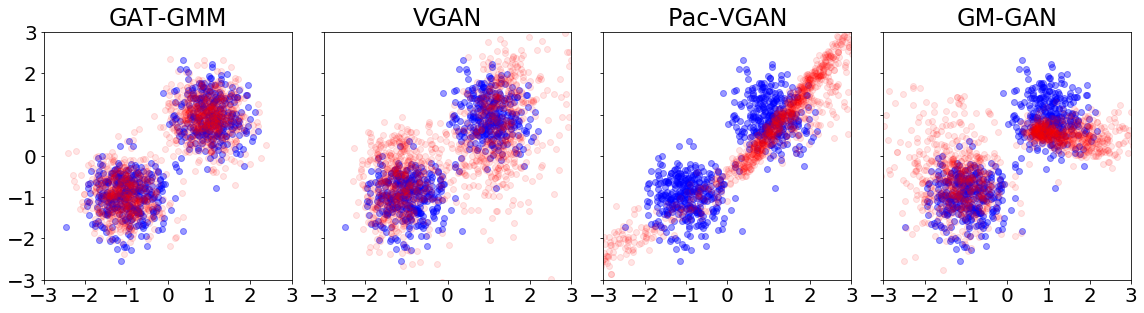}
    \includegraphics[width=.8\textwidth]{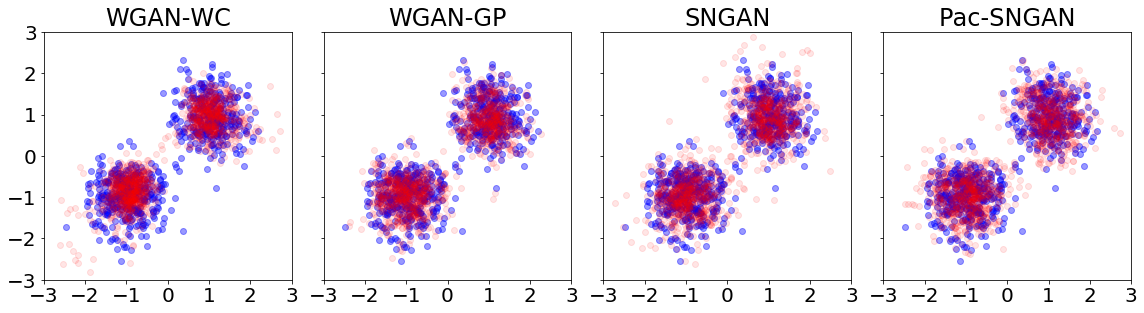}
    \caption{Samples of first two components from rotated task.}\label{fig:rot-samples}
\end{figure}
\begin{figure}[h!]
    \centering
    \includegraphics[width=.8\textwidth]{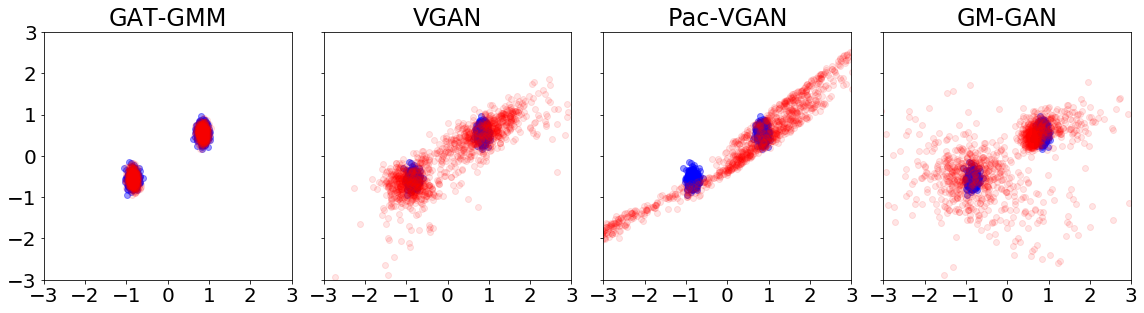}
    \includegraphics[width=.8\textwidth]{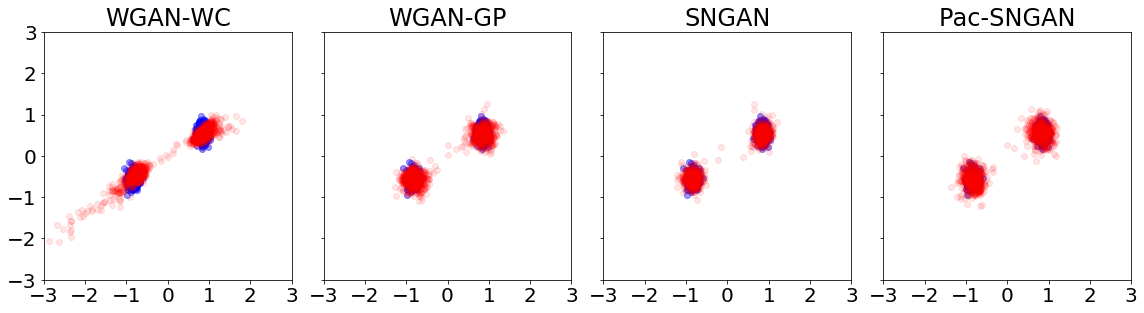}
    \caption{Samples along axes with small variance for rotated task.}\label{fig:collapsed-samples}
\end{figure}

% \begin{figure}
%     \centering
%     \includegraphics[width=.4\textwidth]{b2_necessary.png}
%     \caption{W2GAN trained with no $b_2$ term, as in formulation \eqref{Eq: Discriminator_Uniform}}\label{fig:no-b2-failure}
% \end{figure}

% \textcolor{red}{Add figures on: PacGAN+GMMGAN, Linear gan neural network discriminator, }
% The samples are plotted in Figure \ref{fig:rotated-samples}, and 
% \begin{figure}
%     \centering
%     \includegraphics[width=.3\textwidth]{w2gan-rot100.png}
%     \includegraphics[width=.3\textwidth]{wgan-wc-rot100d.png}
%     \includegraphics[width=.3\textwidth]{wgan-gp-rot100d.png}
%     \caption{Samples produced from W2GAN and WGAN with Weight Clipping and Gradient Penalty on rotated covariance task.}\label{fig:rotated-samples}
% \end{figure}

% We note that in our experiments for GAT-GMM, the $b_2$ term was necessary, and training failed for every set of hyper-parameters we tried. An example of this failure is Figure \ref{fig:no-b2-failure} in \ref{section:additional-figs}

\subsection{Multi-component Mixtures}
While the theory in this paper is primarily for symmetric mixtures with two components, we provide preliminary experiments for a mixture with 4 components. We train GAT-GMM with random initialization on a four-component mixture in 20 dimensions, with the $\mathbf{d}_i$ parameters selected as the top eigenvectors of the empirical $\mathbb{E}[XX^T]$. While the samples in Figure~\ref{fig:four-components} show some bias, the GMM objective decreases nicely over training, and we suspect that with more careful hyper-parameter tuning, we would be able to achieve better performance.
\begin{figure}[h!]
    \centering
    \includegraphics[scale=.45]{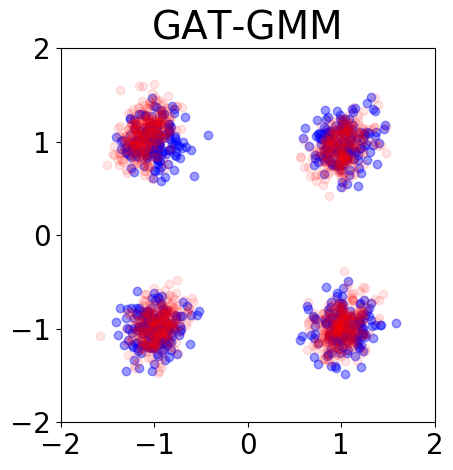}
    \includegraphics[scale=.45]{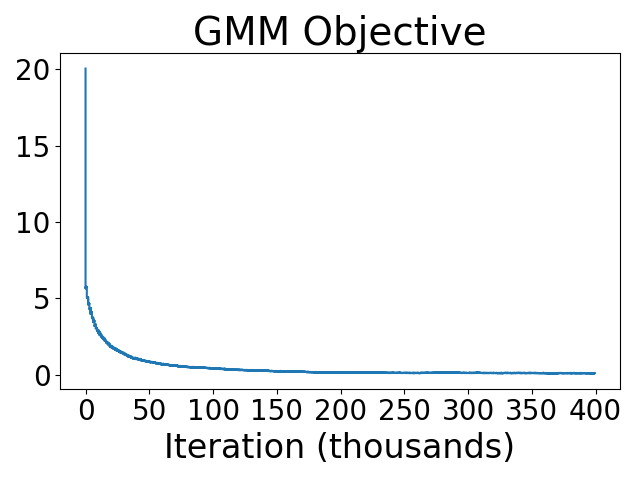}
    \caption{Samples and loss over time for GAT-GMM trained on a 4-component mixture.}\label{fig:four-components}
\end{figure}

\section{Proofs}
\subsection{Proof of Theorem 1}
%\begin{thm*}
%Consider the 2-Wasserstein GAN problem \eqref{app: Eq: general WGAN} with quadratic cost $c(\mathbf{x},\mathbf{x}')=\frac{1}{2}\Vert \mathbf{x}-\mathbf{x}'\Vert_2^2$. Assume $\psi$ defined in the main text is the gradient of a convex function $\phi$. Then, for $\widetilde{D}(\mathbf{x})=\frac{1}{2}\Vert \mathbf{x}\Vert_2^2 - \phi(\mathbf{x})$ we have
%\begin{equation}
%    0\le W_c(P_{\mathbf{X}},P_{G(\mathbf{Z})}) - \bigl\{\mathbb{E}[\widetilde{D}(\mathbf{X})] - \mathbb{E}[\widetilde{D}^c(G(\mathbf{Z}))]\bigr\} \le (\frac{3}{2}M_1+\sqrt{M_1 M_2})\sqrt{P_e} + \sqrt{M_1 M_2} \sqrt[4]{P_e},
%\end{equation}
%where $P_e=\Pr(Y^{\text{\rm opt}}(\mathbf{X})\neq Y)$ is the misclassification rate of the optimal classifier $Y^{\text{\rm opt}}(\mathbf{X})$ for predicting label $Y$ from underlying GMM $\mathbf{X}$. Also, we define $M_1,M_2$ in the following equations with $\Vert\cdot\Vert_\sigma$ denoting the maximum singular value, i.e., the spectral norm,
%\begin{align*}
%    M_1 \, & =\, 8\max_{i} \Vert\Gamma_i\Vert^2_{\sigma} \sqrt{\mathbb{E}[\Vert\mathbf{X}\Vert_2^4]} \, + \, 8\sqrt{P_e}\max_i \Vert \Gamma_i\boldsymbol{\mu}_i - \widetilde{\boldsymbol{\mu}}_i \Vert^2_2 \\
%    M_2\, &= \, 2\max_i \Vert \Gamma_i-I \Vert^2_{\sigma} \mathbb{E}[  \Vert\mathbf{X}\Vert_2^2] \, + \, 2\max_i\Vert \Gamma_i\boldsymbol{\mu}_i -\widetilde{\boldsymbol{\mu}}_i\Vert^2_2.
%\end{align*}
%\end{thm*}
%\begin{proof}

We start by proving the following two lemmas.
\begin{lemma}\label{app: Lemma1, Thm 1's proof}
Define $M_1:=8\max_{i} \Vert\Gamma_i\Vert^2_{\sigma} \sqrt{\mathbb{E}[\Vert\mathbf{X}\Vert_2^4]} \, + \, 8\sqrt{P_e}\max_i \Vert \Gamma_i\boldsymbol{\mu}_i - \widetilde{\boldsymbol{\mu}}_i \Vert^2_2$. Then, the following inequalities hold:
\begin{align*}
    W_c(P_{\psi(\mathbf{X})},P_{G(\mathbf{Z})})\, \le\, \mathbb{E}\bigl[ \big\Vert\Psi(\mathbf{X},Y) - \psi(\mathbf{X})\big\Vert^2_2\bigr]
    \, \le\, M_1\sqrt{P_e}. \numberthis
\end{align*}
\end{lemma}
\begin{proof}
Note that $G(\mathbf{Z})$ shares the same distribution with $\Psi(\mathbf{X},Y)$. Therefore,
\begin{align*}
   W_c(P_{\psi(\mathbf{X})},P_{G(\mathbf{Z})}) &\le \mathbb{E}\bigl[ \big\Vert\Psi(\mathbf{X},Y) - \psi(\mathbf{X})\big\Vert^2_2\bigr] \\
   &=  \mathbb{E}\bigl[ \big\Vert\Psi(\mathbf{X},Y) - \mathbb{E}[\Psi(\mathbf{X},Y)\,\big\vert\, \mathbf{X}]\big\Vert^2_2\bigr] \\
   &=  \mathbb{E}\bigl[ \operatorname{Var}\bigl(\Psi(\mathbf{X},Y)\,\big\vert\, \mathbf{X}\bigr)\bigr] \\
   &\le \mathbb{E}\bigl[ \big\Vert\Psi(\mathbf{X},Y) -\Psi(\mathbf{X},Y^{\text{\rm opt}}(\mathbf{X})) \big\Vert^2_2\bigr] \\
   &\le 8\sqrt{P_e}\max_{i} \Vert\Gamma_i\Vert^2_{\sigma} \sqrt{\mathbb{E}[\Vert\mathbf{X}\Vert_2^4]} \, + \, 8P_e\max_i \Vert \Gamma_i\boldsymbol{\mu}_i - \widetilde{\boldsymbol{\mu}}_i \Vert^2_2  \\
   & = M_1\sqrt{P_e} \numberthis
\end{align*}
Here the last inequality holds due to the tower property of expectations, implying that
\begin{align*}
&\mathbb{E}\bigl[ \big\Vert\Psi(\mathbf{X},Y) -\Psi(\mathbf{X},Y^{\text{\rm opt}}(\mathbf{X})) \big\Vert^2_2\bigr] \\
=\: &\mathbb{E}\bigl[ \mathbb{E}\bigl[\big\Vert\Psi(\mathbf{X},Y) -\Psi(\mathbf{X},Y^{\text{\rm opt}}(\mathbf{X})) \big\Vert^2_2 \, \big\vert \, \mathbf{X}\bigr]\bigr]    \\
\le\: & \mathbb{E}\bigl[ P_e(\mathbf{X}) \max_{i,j}\,\bigl\{ \big\Vert (\Gamma_i-\Gamma_j)\mathbf{X} + (\Gamma_i\boldsymbol{\mu}_i - \widetilde{\boldsymbol{\mu}}_i - \Gamma_j\boldsymbol{\mu}_j + \widetilde{\boldsymbol{\mu}}_j) \big\Vert^2_2 \bigr\}\bigr] \\
\le\: & 4\mathbb{E}\bigl[ P_e(\mathbf{X}) \max_{i}\,\bigl\{ \big\Vert \Gamma_i\mathbf{X} + (\Gamma_i\boldsymbol{\mu}_i - \widetilde{\boldsymbol{\mu}}_i) \big\Vert^2_2 \bigr\}\bigr] \\
\stackrel{(a)}{\le}\: & 8 \mathbb{E}\bigl[ P_e(\mathbf{X}) \max_{i}\,\bigl\{ \Vert \Gamma_i\mathbf{X}\Vert_2^2 + \Vert \Gamma_i\boldsymbol{\mu}_i - \widetilde{\boldsymbol{\mu}}_i \Vert^2_2 \bigr\} \bigr] \\
\stackrel{(b)}{\le}\: & 8 \mathbb{E}\bigl[ P_e(\mathbf{X})\max_{i}\,\bigl\{ \Vert \Gamma_i\Vert^2_{\sigma}\Vert\mathbf{X}\Vert_2^2 + \Vert \Gamma_i\boldsymbol{\mu}_i - \widetilde{\boldsymbol{\mu}}_i \Vert^2_2 \bigr\} \bigr] \\
\le\: & 8\max_{i}\bigl\{\Vert \Gamma_i\Vert^2_{\sigma} \mathbb{E}[ P_e(\mathbf{X}) \Vert\mathbf{X}\Vert_2^2]\bigr\} \, + \, 8\max_i \bigl\{ \Vert \Gamma_i\boldsymbol{\mu}_i - \widetilde{\boldsymbol{\mu}}_i \Vert^2_2  \mathbb{E}[ P_e(\mathbf{X})] \bigr\}  \\
=\: & 8\max_{i}\{\Vert \Gamma_i\Vert^2_{\sigma}\} \mathbb{E}[ P_e(\mathbf{X}) \Vert\mathbf{X}\Vert_2^2] \, + \, 8\max_i \bigl\{ \Vert \Gamma_i\boldsymbol{\mu}_i - \widetilde{\boldsymbol{\mu}}_i \Vert^2_2  \bigr\} \mathbb{E}[ P_e(\mathbf{X})]  \\
\stackrel{(c)}{\le}\: & 8\max_{i}\{\Vert \Gamma_i\Vert^2_{\sigma}\} \sqrt{\mathbb{E}[ P^2_e(\mathbf{X})] \mathbb{E}[\Vert\mathbf{X}\Vert_2^4]} \, + \, 8\max_i \bigl\{ \Vert \Gamma_i\boldsymbol{\mu}_i - \widetilde{\boldsymbol{\mu}}_i \Vert^2_2  \bigr\} \mathbb{E}[ P_e(\mathbf{X})]  \\
\stackrel{(d)}{\le}\: &8\max_{i}\{\Vert \Gamma_i\Vert^2_{\sigma}\} \sqrt{\mathbb{E}[ P_e(\mathbf{X})]} \sqrt{\mathbb{E}[\Vert\mathbf{X}\Vert_2^4]} \, + \, 8\max_i \bigl\{ \Vert \Gamma_i\boldsymbol{\mu}_i - \widetilde{\boldsymbol{\mu}}_i \Vert^2_2  \bigr\} \mathbb{E}[ P_e(\mathbf{X})]  \\
\stackrel{(e)}{=}\: & 8\sqrt{P_e}\max_{i}\{\Vert \Gamma_i\Vert^2_{\sigma}\} \sqrt{\mathbb{E}[\Vert\mathbf{X}\Vert_2^4]} \, + \, 8P_e\max_i \bigl\{ \Vert \Gamma_i\boldsymbol{\mu}_i - \widetilde{\boldsymbol{\mu}}_i \Vert^2_2  \bigr\} . \numberthis
\end{align*}
In the above equations, $P_e(\mathbf{x}):= \Pr(Y^{\operatorname{opt}}(\mathbf{x})\neq Y|\mathbf{X}=\mathbf{x})$ denotes the error probability of optimal Bayes classifier at $\mathbf{X}=\mathbf{x}$ and therefore $\mathbb{E}[P_e(\mathbf{X})]=\mathbb{E}[P(Y^{\operatorname{opt}}(\mathbf{X})\neq Y|\mathbf{X}=\mathbf{X})]=\Pr(Y^{\operatorname{opt}}(\mathbf{X})\neq Y)=P_e$ is the expected classification error in predicting $Y$ from $\mathbf{X}$. $(a)$ holds since for any two vectors $\mathbf{b}_1,\mathbf{b}_2$ we have 
\begin{equation}
\Vert \mathbf{b}_1+\mathbf{b}_2\Vert^2_2= \Vert \mathbf{b}_1\Vert^2_2+\Vert \mathbf{b}_2\Vert^2_2 + 2\mathbf{b}_1^T\mathbf{b}_2\le 2(\Vert \mathbf{b}_1\Vert^2_2+\Vert \mathbf{b}_2\Vert^2_2).
\end{equation}
$(b)$ follows from the definition of the spectral norm implying that $\Vert A\mathbf{x} \Vert_2\le \Vert A \Vert_{\sigma}\Vert \mathbf{x} \Vert_2$. $(c)$ is the direct result of the Cauchy-Schwarz inequality. $(d)$ holds because we have $P^2_e(\mathbf{x})\le P_e(\mathbf{x})$ as $0\le P_e(\mathbf{x})\le 1$. Finally $(e)$ is based on the definition of misclassification rate $P_e:=\mathbb{E}[P_e(\mathbf{X})]$.
\end{proof}
\begin{lemma}\label{app: Lemma2, Thm 1's proof}
Defining $M_2 := 2\max_i \Vert \Gamma_i-I \Vert^2_{\sigma} \mathbb{E}[  \Vert\mathbf{X}\Vert_2^2] \, + \, 2\max_i\Vert \Gamma_i\boldsymbol{\mu}_i -\widetilde{\boldsymbol{\mu}}_i\Vert^2_2$, then
\begin{equation}
 \mathbb{E}\bigl[\Vert\nabla \widetilde{D}(\mathbf{X}) \Vert^2_2\bigr]=\mathbb{E}\bigl[\Vert\nabla \widetilde{D}^c(\psi(\mathbf{X})) \Vert^2_2\bigr]\le M_2.   
\end{equation}
\end{lemma}
\begin{proof}
Note that according to the Brenier's theorem \cite{villani2008optimal}, $\mathbf{x}+\nabla \widetilde{D}(\mathbf{x})$ provides the optimal transport map from $P_\mathbf{X}$ to $P_{\psi(\mathbf{X})}$. Similarly $\mathbf{x}-\nabla \widetilde{D}^c(\mathbf{x})$ provides the inverse optimal transport map from $P_{\psi(\mathbf{X})}$ to $P_\mathbf{X}$. Plugging these optimal transport maps into the definition of optimal transport cost $W_c$, we obtain
\begin{equation}
    W_c(P_\mathbf{X},P_{\psi(\mathbf{X})}) = \frac{1}{2}\mathbb{E}\bigl[\Vert\nabla \widetilde{D}(\mathbf{X}) \Vert^2_2 \bigr] = \frac{1}{2}\mathbb{E}\bigl[\Vert\nabla \widetilde{D}^c(\psi(\mathbf{X}) )\Vert^2_2 \bigr].
\end{equation}
Therefore, $\mathbb{E}\bigl[\Vert\nabla \widetilde{D}^c(\psi(\mathbf{X}) )\Vert^2_2 \bigr]  = \mathbb{E}\bigl[\Vert\nabla \widetilde{D}(\mathbf{X}) \Vert^2_2 \bigr]$ holds. Moreover, note that
\begin{align*}
    \mathbb{E}\bigl[\Vert\nabla \widetilde{D}(\mathbf{X}) \Vert^2_2 \bigr] \, &= \, \mathbb{E}\bigl[\Vert \mathbf{X} - \psi(\mathbf{X}) \Vert^2_2 \bigr] \\
    &=\, \mathbb{E}\bigl[ \Vert \sum_{i=1}^k P(Y=i|\mathbf{X}=\mathbf{X})((I-\Gamma_i)\mathbf{X}+(\Gamma_i\boldsymbol{\mu}_i -\widetilde{\boldsymbol{\mu}}_i)) \Vert^2_2 \bigr] \\
    &\stackrel{(f)}{\le}\, \mathbb{E}\bigl[ \max_i \Vert  (I-\Gamma_i)\mathbf{X}+(\Gamma_i\boldsymbol{\mu}_i -\widetilde{\boldsymbol{\mu}}_i) \Vert^2_2 \bigr] \\
    &\stackrel{(g)}{\le}\, 2\mathbb{E}\bigl[ \max_i\bigl\{ \Vert  (I-\Gamma_i)\mathbf{X}\Vert_2^2+ \Vert \Gamma_i\boldsymbol{\mu}_i -\widetilde{\boldsymbol{\mu}}_i\Vert^2_2 \bigr\} \bigr] \\
    &\stackrel{(h)}{\le}\, 2\max_i \Vert \Gamma_i-I \Vert^2_{\sigma} \mathbb{E}[  \Vert\mathbf{X}\Vert_2^2] \, + \, 2\max_i\Vert \Gamma_i\boldsymbol{\mu}_i -\widetilde{\boldsymbol{\mu}}_i\Vert^2_2  . \numberthis
\end{align*}
Here $(f)$ is implied by the convexity of the Euclidean norm-squared function. $(g)$ holds since for any two vectors $\mathbf{b}_1,\mathbf{b}_2$ we have $\Vert\mathbf{b}_1+\mathbf{b}_2\Vert^2_2\le 2\Vert\mathbf{b}_1\Vert^2_2+2\Vert\mathbf{b}_2\Vert^2_2$. $(h)$ follows from the spectral norm's definition implying that $\Vert A\mathbf{b}\Vert_2\le \Vert A\Vert_{\sigma}\Vert\mathbf{b}\Vert_2$ for any matrix $A$ and vector $\mathbf{b}$. Hence, the lemma's proof is complete.
\end{proof}

To prove Theorem 1, note that according to the Kantorovich duality \cite{villani2008optimal} we have
\begin{align}
    W_c(P_{\mathbf{X}},P_{G(\mathbf{Z})}) &= \max_{D} \: \mathbb{E}[D(\mathbf{X})] - \mathbb{E}[D^c(G(\mathbf{Z}))]  \nonumber\\
    &\ge \mathbb{E}[\widetilde{D}(\mathbf{X})] - \mathbb{E}[\widetilde{D}^c(G(\mathbf{Z}))].
\end{align}
Therefore, in order to prove the theorem it suffices to show 
\begin{equation}
    W_c(P_{\mathbf{X}},P_{G(\mathbf{Z})}) - \bigl\{\mathbb{E}[\widetilde{D}(\mathbf{X})] - \mathbb{E}[\widetilde{D}^c(G(\mathbf{Z}))]\bigr\} \le \sqrt{P_e}M_1M_2+P_e M^2_1,
\end{equation}
Having the quadratic cost $c(\mathbf{x},\mathbf{x}')$, $\sqrt{W_c(P,Q)}$ results in the 2-Wasserstein distance which gives a metric distance among probability distributions. Therefore,
\begin{align*}
    & W_c(P_{\mathbf{X}},P_{G(\mathbf{Z})}) \\
    \, \stackrel{(i)}{\le} \, &\biggl( \sqrt{ W_c(P_{\mathbf{X}},P_{\psi(\mathbf{X})})} + \sqrt{W_c(P_{\psi(\mathbf{X})},P_{G(\mathbf{Z})})} \biggr)^2 \\
    =\, & W_c(P_{\mathbf{X}},P_{\psi(\mathbf{X})}) + W_c(P_{\psi(\mathbf{X})},P_{G(\mathbf{Z})}) + 2\sqrt{W_c(P_{\mathbf{X}},P_{\psi(\mathbf{X})})\,W_c(P_{\psi(\mathbf{X})},P_{G(\mathbf{Z})})} \\
    \stackrel{(j)}{\le} \, & W_c(P_{\mathbf{X}},P_{\psi(\mathbf{X})}) + \sqrt{P_e}M_1 + \sqrt{P_e M_1 M_2} \\
    \stackrel{(k)}{=} \, & \mathbb{E}[\widetilde{D}(\mathbf{X})] - \mathbb{E}[\widetilde{D}^c(\psi(\mathbf{X}))] + \sqrt{P_e}M_1 + \sqrt{P_e M_1 M_2}   \\
    \stackrel{(l)}{\le} \, & \mathbb{E}[\widetilde{D}(\mathbf{X})] - \mathbb{E}[\widetilde{D}^c(\Psi(\mathbf{X},Y))] + \sqrt{P_e}M_1 + \sqrt{P_e M_1 M_2}  \\
    &\quad + \mathbb{E}[\nabla \widetilde{D}^c(\psi(\mathbf{X}))^T(\psi(\mathbf{X})-\Psi(\mathbf{X},Y))] +\frac{1}{2}\mathbb{E}\bigl[\Vert \psi(\mathbf{X})-\Psi(\mathbf{X},Y) \Vert^2_2 \bigr]\\
    \stackrel{(m)}{\le} \, & \mathbb{E}[\widetilde{D}(\mathbf{X})] - \mathbb{E}[\widetilde{D}^c(G(\mathbf{Z}))] + \sqrt{P_e}M_1 + \sqrt{P_e M_1 M_2} \\
    &\quad + \sqrt{ \mathbb{E}\bigl[\Vert \nabla \widetilde{D}^c(\psi(\mathbf{X}))\Vert^2_2\bigr]\, \mathbb{E}\bigl[\Vert \psi(\mathbf{X})-\Psi(\mathbf{X},Y))]\Vert^2_2]} +\frac{1}{2}\mathbb{E}\bigl[\Vert \psi(\mathbf{X})-\Psi(\mathbf{X},Y) \Vert^2_2 \bigr]\\
    \stackrel{(n)}{\le} \, & \mathbb{E}[\widetilde{D}(\mathbf{X})] - \mathbb{E}[\widetilde{D}^c(G(\mathbf{Z}))] + \frac{3}{2}\sqrt{P_e}M_1 + \sqrt{P_e M_1 M_2} + \sqrt{ M_2M_1}\sqrt[4]{P_e} \numberthis
\end{align*}
In the above equations, $(i)$ follows from the application of the triangle inequality holding for the $2$-Wasserstein distance which is the square root of $W_c(P,Q)$. $(j)$ is the consequence of Lemma \ref{app: Lemma1, Thm 1's proof} and Lemma \ref{app: Lemma2, Thm 1's proof}. $(k)$ holds because of the Brenier theorem implying that $\widetilde{D}$ is the solution to the dual maximization problem to $W_c(P_{\mathbf{X}},P_{\psi(\mathbf{X})})$. $(l)$ uses the $c$-concavity of $\widetilde{D}^c$ implying $\widetilde{D}^c$'s Hessian is always upper-bounded by the identity matrix and therefore
\begin{equation}
    \widetilde{D}^c(G(\mathbf{z})) \le \widetilde{D}^c(\psi(\mathbf{x})) + \nabla \widetilde{D}^c(\psi(\mathbf{x}))^T (G(\mathbf{z}) - \psi(\mathbf{x})) + \frac{1}{2}\Vert G(\mathbf{z}) - \psi(\mathbf{x}) \Vert^2_2. 
\end{equation}
$(m)$ follows form applying the Cauchy-Shwarz inequality, and finally $(n)$ results from Lemma $\ref{app: Lemma1, Thm 1's proof}$ and $\ref{app: Lemma2, Thm 1's proof}$. Therefore, we have proved
\begin{align*}
  W_c(P_{\mathbf{X}},P_{G(\mathbf{Z})}) - \bigl\{ \mathbb{E}[\widetilde{D}(\mathbf{X})] - \mathbb{E}[\widetilde{D}^c(G(\mathbf{Z}))] \bigr\} \le (\frac{3}{2}M_1+\sqrt{M_1 M_2})\sqrt{P_e} + \sqrt{M_1 M_2} \sqrt[4]{P_e}
\end{align*}
which completes the proof.
%\end{proof}

\subsection{Proof of Proposition 1}
\begin{lemma}\label{app: Lemma1, Prop1's proof}
Consider GMM $\mathbf{X}\sim \sum_{i=1}^k \pi_i \mathcal{N}(\boldsymbol{\mu}_i,\Sigma_i)$ with component variable $Y$. Then, for $\phi(\mathbf{x}):= -\log\bigl(\sum_{i=1}^k \exp\bigl( \frac{1}{2}(\mathbf{x}-\boldsymbol{\mu}_i)^T\Sigma^{-1}_i(\mathbf{x}-\boldsymbol{\mu}_i) + \frac{1}{2}\log(\pi_i^2/\operatorname{det}(\Sigma_i))  \bigr) \bigr)$ we have
\begin{equation}
    \nabla \phi(\mathbf{x}) := \sum_{i=1}^k \Pr(Y=i|\mathbf{X}=\mathbf{x}) \Sigma^{-1}_i(\mathbf{x}-\boldsymbol{\mu}_i)
\end{equation}
\end{lemma}
\begin{proof}
The lemma is an immediate consequence of the following application of the Bayes rule:
\begin{equation}
    \Pr(Y=i|\mathbf{X}=\mathbf{x})= \frac{\exp\bigl( -\frac{1}{2}(\mathbf{x}-\boldsymbol{\mu}_i)^T\Sigma^{-1}_i(\mathbf{x}-\boldsymbol{\mu}_i) + \frac{1}{2}\log(\pi_i^2/\operatorname{det}(\Sigma_i))  \bigr)}{\sum_{i=1}^k \exp\bigl( \frac{1}{2}(\mathbf{x}-\boldsymbol{\mu}_i)^T\Sigma^{-1}_i(\mathbf{x}-\boldsymbol{\mu}_i) + \frac{1}{2}\log(\pi_i^2/\operatorname{det}(\Sigma_i))  \bigr)}. 
\end{equation}
\end{proof}
Let's rewrite function $\psi$ as
\begin{align}\label{app: Prop1 proof Eq 1}
    \psi(\mathbf{x}) \, &= \, \sum_{i=1}^k \bigl[\Pr(Y=i|\mathbf{X}=\mathbf{x}) \widetilde{\Sigma}^{1/2}_i{\Sigma}^{-1/2}_i(\mathbf{x}-\boldsymbol{\mu}_i)+\widetilde{\boldsymbol{\mu}}_i \bigr]\nonumber \\
    &= \, \sum_{i=1}^k \bigl[\Pr(Y=i|\mathbf{X}=\mathbf{x}) \widetilde{\Sigma}^{1/2}_i{\Sigma}^{-1/2}_i(\mathbf{x}-\boldsymbol{\mu}_i)\bigr] + \sum_{i=1}^k\bigl[ \Pr(Y=i|\mathbf{X}=\mathbf{x})\widetilde{\boldsymbol{\mu}}_i\bigr]. 
\end{align}
Considering $\bar{\mathbf{X}}_1\sim \sum_{i=1}^k \pi_i \mathcal{N}(\boldsymbol{\mu}_i,\widetilde{\Sigma}^{-1/2}_i\Sigma^{1/2}_i)$ with label variable $\bar{Y}_1$, Lemma \ref{app: Lemma1, Prop1's proof} together with the assumption that $\Sigma_i$ and $\widetilde{\Sigma}_i$ commute imply that for function $\phi_1$ defined as
\begin{equation}\label{app: Prop1: Phi1}
    \phi_1(\mathbf{x}):= -\log\biggl(\sum_{i=1}^k \exp\bigl( -\frac{1}{2}(\mathbf{x}-\boldsymbol{\mu}_i)^T\widetilde{\Sigma}^{1/2}_i\Sigma^{-1/2}_i(\mathbf{x}-\boldsymbol{\mu}_i) + \log(\pi_i\operatorname{det}(\widetilde{\Sigma}_i)/\operatorname{det}(\Sigma_i))  \bigr) \biggr)
\end{equation}
we have
\begin{equation}\label{app: Prop1: Phi1 Gradient}
    \nabla \phi_1(\mathbf{x}) = \sum_{i=1}^k \bigl[\Pr(\bar{Y}_1=i|\bar{\mathbf{X}}_1=\mathbf{x}) \widetilde{\Sigma}^{1/2}_i{\Sigma}^{-1/2}_i(\mathbf{x}-\boldsymbol{\mu}_i)\bigr].
\end{equation}
Similarly, for GMM $\bar{\mathbf{X}}_2\sim \sum_{i=1}^k \widetilde{\pi}_i \mathcal{N}({\widetilde{\boldsymbol{\mu}}}_i,I)$ and its component variable $\bar{Y}_2$ we define
\begin{align}\label{app: Prop1: Phi2}
    \phi_2(\mathbf{x}):&= -\log\biggl(\sum_{i=1}^k \exp\bigl(  -\frac{1}{2}\Vert\mathbf{x}-\widetilde{\boldsymbol{\mu}}_i\Vert^2_2 + \log(\pi_i)  \bigr) \biggr)\nonumber \\
    &= \frac{1}{2}\Vert \mathbf{x}\Vert^2_2 - \log\biggl(\sum_{i=1}^k \exp\bigl( \mathbf{x}^T\widetilde{\boldsymbol{\mu}}_i -\frac{1}{2}\Vert \widetilde{\boldsymbol{\mu}}_i\Vert^2_2 + \log(\pi_i)  \bigr) \biggr)
\end{align}
for which Lemma \ref{app: Lemma1, Prop1's proof} implies
\begin{align}\label{app: Prop1: Phi2 Gradient}
   \nabla \phi_2(\mathbf{x}) &= -\sum_{i=1}^k \bigl[\Pr(\bar{Y}_2=i|\bar{\mathbf{X}}_2=\mathbf{x}) (\mathbf{x}-\widetilde{\boldsymbol{\mu}}_i)\bigr] \nonumber \\
   & = \mathbf{x} - \sum_{i=1}^k \Pr(\bar{Y}_2=i|\bar{\mathbf{X}}_2=\mathbf{x}) \widetilde{\boldsymbol{\mu}}_i
\end{align}
Therefore, the proposition's assumptions imply that for every $\mathbf{x}$ we have
\begin{align*}
    &\big\Vert \nabla \phi_1(\mathbf{x}) - \sum_{i=1}^k \Pr(Y=i|\mathbf{X}=\mathbf{x}) \widetilde{\Sigma}^{1/2}_i{\Sigma}^{-1/2}_i(\mathbf{x}-\boldsymbol{\mu}_i)\big\Vert_2 \\
    = \, &\big\Vert \sum_{i=1}^k \bigl[\bigl(\Pr(Y=i|\mathbf{X}=\mathbf{x}) - \Pr(\bar{Y}_1=i|\bar{\mathbf{X}}_1=\mathbf{x})\bigr) \widetilde{\Sigma}^{1/2}_i{\Sigma}^{-1/2}_i(\mathbf{x}-\boldsymbol{\mu}_i)\bigr]\big\Vert_2 \\
    \le \, &  \sum_{i=1}^k\bigl[\bigl\vert\Pr(Y=i|\mathbf{X}=\mathbf{x}) - \Pr(\bar{Y}_1=i|\bar{\mathbf{X}}_1=\mathbf{x})\bigr\vert\bigr]\, \max_{i}\,\Vert\widetilde{\Sigma}^{1/2}_i{\Sigma}^{-1/2}_i(\mathbf{x}-\boldsymbol{\mu}_i)\Vert_2 \\
    \le \, & \epsilon \max_{i}\,\sqrt{\Vert\widetilde{\Sigma}_i\Sigma^{-1}_i\Vert_{\sigma}}\bigl(\Vert\mathbf{x}\Vert_2+\Vert\boldsymbol{\mu}_i\Vert_2\bigr). \numberthis
\end{align*}
Also, we have
\begin{align*}
    &\big\Vert \mathbf{x} - \nabla  \phi_2(\mathbf{x}) - \sum_{i=1}^k \Pr(Y=i|\mathbf{X}=\mathbf{x})\widetilde{\boldsymbol{\mu}}_i\big\Vert_2 \\
    =\, & \big\Vert\sum_{i=1}^k\bigl[ \bigl(\Pr(Y=i|\mathbf{X}=\mathbf{x})-\Pr(\bar{Y}_2=i|\bar{\mathbf{X}}_2=\mathbf{x})\bigr)\widetilde{\boldsymbol{\mu}}_i\bigr]\big\Vert_2 \\
    \le\, & \sum_{i=1}^k\bigl[\vert  \bigl(\Pr(Y=i|\mathbf{X}=\mathbf{x})-\Pr(\bar{Y}_2=i|\bar{\mathbf{X}}_2=\mathbf{x})\vert\bigr]\, \max_i\, \Vert \widetilde{\boldsymbol{\mu}}_i\Vert_2 \\
    \le\, &\epsilon\, \max_i\, \Vert \widetilde{\boldsymbol{\mu}}_i\Vert_2.
\end{align*}
Combining the above inequalities with \eqref{app: Prop1 proof Eq 1} we obtain
\begin{equation}
    \big\Vert \nabla\bigl\{\frac{1}{2}\Vert\mathbf{x}\Vert^2_2+\phi_1(\mathbf{x}) - \phi_2(\mathbf{x})\bigr\} - \psi(\mathbf{x}) \big\Vert_2 \, \le \, \epsilon \left( \max_i\, \Vert \boldsymbol{\mu}_i\Vert_2 +  \max_{i}\,\sqrt{\Vert\widetilde{\Sigma}_i\Sigma^{-1}_i\Vert_{\sigma}}\bigl(\Vert\mathbf{x}\Vert_2+\Vert\widetilde{\boldsymbol{\mu}}_i\Vert_2\bigr) \right) 
\end{equation}
Note that $\phi_1(\mathbf{x})= -\log(\sum_{i=1}^k\exp(\frac{1}{2}\mathbf{x}^TA_i\mathbf{x}+\mathbf{b}_i^T\mathbf{x}+c_i))$ for some choice of $(A_i,\mathbf{b}_i,c_i)^k_{i=1}$. Also, $\frac{1}{2}\Vert\mathbf{x}\Vert^2_2-\phi_2(\mathbf{x})=\log(\sum_{i=k+1}^{2k}\exp(\mathbf{b}_i^T\mathbf{x}+c_i))$ for some choice of $(\mathbf{b}_i,c_i)_{i=k+1}^{2k}$. Therefore, $\frac{1}{2}\Vert\mathbf{x}\Vert^2_2-\phi_2(\mathbf{x})+\phi_1(\mathbf{x})$ provides a function of the form $D_{(A_i)_{i=1}^k,(\mathbf{b}_i,c_i)_{i=1}^{2k}}$. Moreover, the above inequality implies that for this choice of $D_{(A_i)_{i=1}^k,(\mathbf{b}_i,c_i)_{i=1}^{2k}}$
\begin{align}
    &\mathbb{E}\biggl[\bigl\Vert \psi(\mathbf{X})- \nabla D_{(A_i)_{i=1}^k,(\mathbf{b}_i,c_i)_{i=1}^{2k}}(\mathbf{X})\bigr\Vert_2 \biggr] \nonumber \\
    \le \, &\epsilon \left( \max_i\, \Vert \boldsymbol{\mu}_i\Vert_2 +  \max_{i}\,\sqrt{\Vert\widetilde{\Sigma}_i\Sigma^{-1}_i\Vert_{\sigma}}\mathbb{E}[\Vert\mathbf{X}\Vert_2]+ \max_i\, \sqrt{\Vert\widetilde{\Sigma}_i\Sigma^{-1}_i\Vert_{\sigma}}\Vert\widetilde{\boldsymbol{\mu}}_i\Vert_2\bigr) \right). \label{Prop 2: final Equation}
\end{align}
which completes the proof. 

For the special case where we have symmetric mixtures of two well-separated Gaussians $\mathbf{X}\sim\frac{1}{2}\mathcal{N}(\boldsymbol{\mu},\Sigma)+\frac{1}{2}\mathcal{N}(-\boldsymbol{\mu},\Sigma)$ and $\widetilde{\mathbf{X}}\sim\frac{1}{2}\mathcal{N}(\widetilde{\boldsymbol{\mu}},\widetilde{\Sigma})+\frac{1}{2}\mathcal{N}(-\widetilde{\boldsymbol{\mu}},\widetilde{\Sigma})$, it can be seen that
\begin{equation}
    \Pr(Y=0|\mathbf{X}=\mathbf{x})=\frac{1}{1+\exp(2\boldsymbol{\mu}^T\Sigma^{-1}\mathbf{x})},\quad  \Pr(Y=1|\mathbf{X}=\mathbf{x})=\frac{1}{1+\exp(-2\boldsymbol{\mu}^T\Sigma^{-1}\mathbf{x})}.
\end{equation}
As a result, for every $\mathbf{x}$ such that $\big\vert\boldsymbol{\mu}^T\Sigma^{-1}{\mathbf{x}}\big\vert>\frac{1}{2}\log(\frac{2}{\epsilon})$, $\big\vert\boldsymbol{\mu}^T\widetilde{\Sigma}^{1/2}\Sigma^{-1/2}{\mathbf{x}}\big\vert>\frac{1}{2}\log(\frac{2}{\epsilon})$ and $\big\vert\widetilde{\boldsymbol{\mu}}^T{\mathbf{x}}\big\vert>\frac{1}{2}\log(\frac{2}{\epsilon})$ hold we will have $\sum_{i=1}^2 \vert \Pr(Y=i\,\vert\,\mathbf{X}=\mathbf{x}) - \Pr(\bar{Y}=i\,\vert\,\bar{\mathbf{X}}=\mathbf{x})\vert \le \epsilon$ for both choices of the GMMs in the proposition. Provided that the two GMMs have well-separated components we assume that $\boldsymbol{\mu}^T\Sigma^{-1}\boldsymbol{\mu}>4\max\{1,\frac{\lambda_{\max}(\widetilde{\Sigma})}{\lambda_{\min}(\Sigma)}\}\log(6/\epsilon)$. Note that according to standard Gaussian tail bounds we have
\begin{align*}
    \Pr\bigl(\vert\boldsymbol{\mu}^T\Sigma^{-1}\mathbf{X}\vert \le \boldsymbol{\mu}^T\Sigma^{-1}\boldsymbol{\mu} -t\sqrt{\boldsymbol{\mu}^T\Sigma^{-1}\boldsymbol{\mu}}\bigr) \, &\le\, \exp(\frac{-t^2}{2}), \\
    \Pr\bigl(\vert\boldsymbol{\mu}^T\widetilde{\Sigma}^{1/2}\Sigma^{-1/2}\mathbf{X}\vert \le \boldsymbol{\mu}^T\widetilde{\Sigma}^{1/2}\Sigma^{-1/2}\boldsymbol{\mu} -t\sqrt{\boldsymbol{\mu}^T\widetilde{\Sigma}\boldsymbol{\mu}}\bigr) \, &\le \, \exp(\frac{-t^2}{2}), \\
    \Pr\bigl(\vert\widetilde{\boldsymbol{\mu}}^T\mathbf{X}\vert \le \vert \widetilde{\boldsymbol{\mu}}^T\boldsymbol{\mu}\vert -t\sqrt{\widetilde{\boldsymbol{\mu}}^T\Sigma\widetilde{\boldsymbol{\mu}}}\bigr) \, &\le\, \exp(\frac{-t^2}{2}). \numberthis
\end{align*}
Therefore, choosing $t=\sqrt{2\log(6/\epsilon)}$ and taking a union bound shows that the outcomes $\vert\boldsymbol{\mu}^T\Sigma^{-1}\mathbf{X}\vert> \frac{1}{2}\log(\frac{2}{\epsilon})$, $\big\vert\boldsymbol{\mu}^T\widetilde{\Sigma}^{1/2}\Sigma^{-1/2}{\mathbf{X}}\big\vert>\frac{1}{2}\log(\frac{2}{\epsilon})$ and $\big\vert\widetilde{\boldsymbol{\mu}}^T{\mathbf{X}}\big\vert>\frac{1}{2}\log(\frac{2}{\epsilon})$ simultaneously hold with probability at least $1-\epsilon$ given that 
\begin{equation}
\vert \widetilde{\boldsymbol{\mu}}^T\boldsymbol{\mu}\vert>\sqrt{\log(6/\epsilon)\widetilde{\boldsymbol{\mu}}^T\Sigma\widetilde{\boldsymbol{\mu}}}+\frac{1}{2}\log(\frac{2}{\epsilon}) =\mathcal{O}\bigl(\log(1/{\epsilon})\bigr).    
\end{equation}
Defining the intersection of the above outcomes as event $E$, we can apply the law of total probability to rewrite \eqref{Prop 2: final Equation} as 
\begin{align*}
    &\mathbb{E}\biggl[\bigl\Vert \psi(\mathbf{X})- \nabla D_{(A_i)_{i=1}^k,(\mathbf{b}_i,c_i)_{i=1}^{2k}}(\mathbf{X})\bigr\Vert_2 \biggr] \nonumber \\
    \le \, &P(E)\epsilon \left( \max_i\, \Vert \boldsymbol{\mu}_i\Vert_2 +  \max_{i}\,\sqrt{\Vert\widetilde{\Sigma}_i\Sigma^{-1}_i\Vert_{\sigma}}\mathbb{E}[\Vert\mathbf{X}\Vert_2| E ]+ \max_i\, \sqrt{\Vert\widetilde{\Sigma}_i\Sigma^{-1}_i\Vert_{\sigma}}\Vert\widetilde{\boldsymbol{\mu}}_i\Vert_2\bigr) \right)  \\
    &\: +  (1-P(E)) \left( \max_i\, \Vert \boldsymbol{\mu}_i\Vert_2 +  \max_{i}\,\sqrt{\Vert\widetilde{\Sigma}_i\Sigma^{-1}_i\Vert_{\sigma}}\mathbb{E}[\Vert\mathbf{X}\Vert_2| E^{\sim} ]+ \max_i\, \sqrt{\Vert\widetilde{\Sigma}_i\Sigma^{-1}_i\Vert_{\sigma}}\Vert\widetilde{\boldsymbol{\mu}}_i\Vert_2\bigr) \right) \\
    \le \, & \epsilon \biggl( \max_i\, \Vert \boldsymbol{\mu}_i\Vert_2 +  \max_{i}\,\sqrt{\Vert\widetilde{\Sigma}_i\Sigma^{-1}_i\Vert_{\sigma}}\bigl(P(E)\mathbb{E}[\Vert\mathbf{X}\Vert_2| E ] + \mathbb{E}[\Vert\mathbf{X}\Vert_2| E^{\sim} ]\bigr)\\
    &\quad + \max_i\, \sqrt{\Vert\widetilde{\Sigma}_i\Sigma^{-1}_i\Vert_{\sigma}}\Vert\widetilde{\boldsymbol{\mu}}_i\Vert_2\bigr) \biggr) \\
    \le \, & 2\epsilon \left( \max_i\, \Vert \boldsymbol{\mu}_i\Vert_2 +  \max_{i}\,\sqrt{\Vert\widetilde{\Sigma}_i\Sigma^{-1}_i\Vert_{\sigma}}\mathbb{E}[\Vert\mathbf{X}\Vert_2] + \max_i\, \sqrt{\Vert\widetilde{\Sigma}_i\Sigma^{-1}_i\Vert_{\sigma}}\Vert\widetilde{\boldsymbol{\mu}}_i\Vert_2\bigr) \right) \numberthis
\end{align*}
where $E^{\sim}$ denotes the complement of event $E$. Note that the last inequality holds because of the tower property of expectations implying that $P(E)\mathbb{E}[\Vert\mathbf{X}\Vert_2| E ] \le \mathbb{E}[\Vert\mathbf{X}\Vert_2]$ and noting that $\mathbb{E}[\Vert\mathbf{X}\Vert_2| E^{\sim} ]\le \mathbb{E}[\Vert\mathbf{X}\Vert_2]$ since the event $E^{\sim}$ only requires bounded projected norms along the characterized vectors.

%\end{proof}

\subsection{Proof of Remark 1}
%\begin{remark*}
%In Proposition \ref{app: Prop 1}, if we also assume that the components of each GMM share the same covariance matrix, then the function $\psi$ can be approximated by the quadratic softmax-based $D_{A,(\mathbf{b}_i,c_i)_{i=1}^{2k}}$ with the same approximation guarantee:
%\begin{equation}\label{app: Eq: Discriminator_Uniform}
%    D_{A,(\mathbf{b}_i,c_i)_{i=1}^{2k}}(\mathbf{x}) = \frac{1}{2}\mathbf{x}^T A\mathbf{x}+\log\,\biggl(\,\frac{\sum_{i=1}^k \exp(\mathbf{b}_i^T\mathbf{x}+c_i)}{\sum_{i=k+1}^{2k} \exp(\mathbf{b}_i^T\mathbf{x}+c_i)}\,\biggr).
%\end{equation}
%In the case of a mixture of two symmetric Gaussians, we can further remove constant $c_i$'s and achieve the same approximation guarantee via
%\begin{equation}\label{app: Eq: Discriminator_Uniform_Binary}
%    D_{A,(\mathbf{b}_i)_{i=1}^{2k}}(\mathbf{x}) = \frac{1}{2}\mathbf{x}^T A\mathbf{x}+\log\,\bigl(\,\frac{\exp(\mathbf{b}_1^T\mathbf{x})+ \exp(\mathbf{b}_2^T\mathbf{x})}{\exp(\mathbf{b}_3^T\mathbf{x})+ \exp(\mathbf{b}_4^T\mathbf{x})}\,\bigr).
%\end{equation}
%\end{remark*}
%\begin{proof}
Note that if $\mathbf{X},\widetilde{\mathbf{X}}$ have the same covariance matrix across their components, we will have $\Sigma_i=\Sigma_j$ and $\widetilde{\Sigma}_i=\widetilde{\Sigma}_j$ for every $i,j$. Therefore, the exponents summed in $\phi_1$'s definition in Lemma \ref{app: Lemma2, Thm 1's proof} share the same quadratic term $-\frac{1}{2}\mathbf{x}^T\widetilde{\Sigma}^{1/2}_i\Sigma^{-1/2}_i\mathbf{x}$ which can be factored out from the log-sum-exp function. As a result, function $\frac{1}{2}\Vert\mathbf{x}\Vert^2_2-\phi_2(\mathbf{x})+\phi_1(\mathbf{x})$ used for approximating $\psi$ in Proposition \ref{Prop 1}'s proof can be simplified and parameterized as $D_{A,(\mathbf{b}_i,c_i)_{i=1}^{2k}}$. 

In the case of a symmetric mixture of two Gaussians, we have $\pi_1=\pi_2=\frac{1}{2}$ and also $\boldsymbol{\mu}_1=-\boldsymbol{\mu}_2$. As a result, the constant terms in approximation $D_{A,(\mathbf{b}_i,c_i)_{i=1}^{4}}$ will be equal and we have $c_1=c_2$, and $c_3=c_4$. Therefore, $D_{A,(\mathbf{b}_i,c_i)_{i=1}^{4}}(\mathbf{x})$ can be simplified to $ D_{A,(\mathbf{b}_i)_{i=1}^{4}}(\mathbf{x})+c$ for a constant $c$. However, the additive constant $c$ will be canceled out in the objective of the dual problem $\mathbb{E}[D(\mathbf{X})]-\mathbb{E}[D^c(\widetilde{\mathbf{X}})]$ and hence can be removed to reach $ D_{A,(\mathbf{b}_i)_{i=1}^{4}}$. 
%\end{proof}

\subsection{Proof of Proposition 2}
%\begin{prop*}
%Consider the discriminator function $D_{A,(\mathbf{b}_i,c_i)_{i=1}^{2k}}$ defined in \eqref{app: Eq: Discriminator_Uniform}. For constant $\eta > 0$, assume $\lambda_{\max}(A)+2\max_i \Vert \mathbf{b}_i \Vert^2_2\le \eta < 1$ where $\lambda_{\max}(\cdot)$ denotes the maximum eigenvalue. Then, given any set of vectors $(\mathbf{d}_i)_{i=1}^k$ and constants $(e_i)_{i=1}^k$ we have
%\begin{align*}
%     &\mathbb{E}\bigl[ D^c_{A,(\mathbf{b}_i,c_i)_{i=1}^{2k}}(\mathbf{X}) \bigr] \, \le \, \mathbb{E}\bigl[ D_{A,(\mathbf{b}_i,c_i)_{i=1}^{2k}}(\mathbf{X})\bigr] \\
%     &\quad + \frac{3k^2(\mathbb{E}[\Vert\mathbf{X}\Vert^2_2]+1)}{1-\eta}\biggl(\Vert A\Vert^2_2 + \sum_{i=1}^k \bigl[ \Vert\mathbf{b}_i - \mathbf{d}_i \Vert^2_2 + \Vert\mathbf{b}_{k+i} - \mathbf{d}_i \Vert^2_2 +(c_i-e_i)^2+ (c_{k+i}-e_i)^2\bigr] \biggr).
%\end{align*}
%\end{prop*}
%\begin{proof}

We first prove the following lemmas.
\begin{lemma}\label{app: Lemma1: Prop2}
For every $\mathbf{x}$, we have $\nabla_{\mathbf{x}}^2 D_{A,(\mathbf{b}_i,c_i)_{i=1}^{2k}}(\mathbf{x})\le (\lambda_{\max}(A)+2\max_i \Vert \mathbf{b}_i \Vert^2_2) I $.
\end{lemma}
\begin{proof}
Since the Hessian of $\frac{1}{2}\mathbf{x}^TA\mathbf{x}$ is $A$, we only need to show for $g_{(\mathbf{b}_i,c_i)_{i=1}^{k}}(\mathbf{x}):=\log(\sum_{i=1}^k \exp(b_i^T\mathbf{x}+c_i))$ we have $\nabla^2 g_{(\mathbf{b}_i,c_i)_{i=1}^{k}}(\mathbf{x}) \le \max_i \Vert \mathbf{b}_i\Vert^2_2$. To show this, note that defining multinomial distribution $q_i(\mathbf{x}):=\frac{\exp(\mathbf{b}^T_i\mathbf{x}+c_i)}{\sum_{j=1}^k \exp(\mathbf{b}^T_j\mathbf{x}+c_j)}$ we have 
\begin{align*}
 \nabla^2 g_{(\mathbf{b}_i,c_i)_{i=1}^{k}}(\mathbf{x}) &= \sum_{i} q_i(\mathbf{x})  \mathbf{b}_i\mathbf{b}^T_i - \sum_{i,j} q_i(\mathbf{x})q_j(\mathbf{x})  \mathbf{b}_i\mathbf{b}^T_j \\
 &= \sum_{i} q_i(\mathbf{x})  \mathbf{b}_i\mathbf{b}^T_i - \big(\sum_i q_i(\mathbf{x}) \mathbf{b}_i \bigr)\bigl(\sum_i q_i(\mathbf{x}) \mathbf{b}_i\bigr)^T\numberthis
\end{align*}
Therefore, since $\sum_i q_i(\mathbf{x})=1$ we have
\begin{align*}
    \operatorname{Tr}(\nabla^2 g_{(\mathbf{b}_i,c_i)_{i=1}^{k}}(\mathbf{x})) &\le \operatorname{Tr}\bigl(\sum_{i} q_i(\mathbf{x})  \mathbf{b}_i\mathbf{b}^T_i \bigr) \\
    &= \sum_i q_i(\mathbf{x}) \Vert \mathbf{b}_i \Vert^2_2 \\
    &\le \max_i \,\Vert\mathbf{b}_i \Vert^2_2.\numberthis
\end{align*}
Noting that for every positive semi-definite matrix $B$, $\lambda_{\max}(B)\le \operatorname{Tr}(B)$, the proof is complete.
\end{proof}
\begin{lemma}\label{app: Lemma 2: Prop2}
For vectors $(\mathbf{d}_i)_{i=1}^k$ and constants $(e_i)_{i=1}^k$ we have
\begin{align*}
 &\bigl\Vert \nabla_{\mathbf{x}} D_{A,(\mathbf{b}_i,c_i)_{i=1}^{2k}}(\mathbf{x}) \bigr\Vert^2_2 \\
 \le\, &  3k^2(1+\Vert\mathbf{x}\Vert_2)^2\biggl(\Vert A\Vert^2_{\sigma} + \sum_{i=1}^k \bigl[ \Vert\mathbf{b}_i - \mathbf{d}_i \Vert^2_2 + \Vert\mathbf{b}_{k+i} - \mathbf{d}_i \Vert^2_2 +( c_i-e_i)^2 + \vert (c_{k+i}-e_i)^2\bigr] \biggr). \numberthis 
\end{align*}
\end{lemma}
\begin{proof}
Defining $q_i(\mathbf{x}):=\frac{\exp(\mathbf{b}_i^T\mathbf{x}+c_i)}{\sum_{j=1}^k\exp(\mathbf{b}_j^T\mathbf{x}+c_j)}$ for $1\le i\le k$ and $q_i(\mathbf{x}):=\frac{\exp(\mathbf{b}_i^T\mathbf{x}+c_i)}{\sum_{j=k+1}^{2k}\exp(\mathbf{b}_j^T\mathbf{x}+c_j)}$ for $k+1\le i\le 2k$, we have
\begin{equation}
  \nabla_{\mathbf{x}} D_{A,(\mathbf{b}_i,c_i)_{i=1}^{2k}}(\mathbf{x}) = A\mathbf{x} + \sum_{i=1}^k q_i(\mathbf{x})\mathbf{b}_i -  \sum_{i=k+1}^{2k} q_i(\mathbf{x})\mathbf{b}_i. 
\end{equation}
Therefore, if we define $p_i(\mathbf{x}):=\frac{\exp(\mathbf{d}_i^T\mathbf{x}+e_i)}{\sum_{j=1}^k\exp(\mathbf{d}_j^T\mathbf{x}+e_j)}$ for $1\le i\le k$, we can write
\begin{equation}\label{app: Prop2: Lemma 2 Eq 1}
  \nabla_{\mathbf{x}} D_{A,(\mathbf{b}_i,c_i)_{i=1}^{2k}}(\mathbf{x}) = A\mathbf{x} + \sum_{i=1}^k\bigl[ q_i(\mathbf{x})\mathbf{b}_i - p_i(\mathbf{x})\mathbf{d}_i\bigr] -  \sum_{i=1}^{k} \bigl[ q_{k+i}(\mathbf{x})\mathbf{b}_{k+i} - p_i(\mathbf{x})\mathbf{d}_i\bigr]. 
\end{equation}
To bound the norm of $\sum_{i=1}^k\bigl[ q_i(\mathbf{x})\mathbf{b}_i - p_i(\mathbf{x})\mathbf{d}_i\bigr]$, consider the Jacobian matrix of $\sum_{i=1}^k q_i(\mathbf{x})\mathbf{b}_i$ with respect to vector $\mathbf{b}_j$ and constant $c_j$ which can be written as
\begin{align*}
    \mathbf{J}_{\nabla_{\mathbf{x}} D_{A,(\mathbf{b}_i,c_i)_{i=1}^{2k}}}(\mathbf{b}_j) &= q_j(\mathbf{x})I+\sum_{i=1}^k q_i(\mathbf{x})(1- q_j(\mathbf{x}))\mathbf{b}^T_i\mathbf{x} \numberthis\\
    \mathbf{J}_{\nabla_{\mathbf{x}} D_{A,(\mathbf{b}_i,c_i)_{i=1}^{2k}}}(c_j) &= (1- q_j(\mathbf{x}))\sum_{i=1}^k q_i(\mathbf{x})\mathbf{b}^T_i \numberthis
\end{align*}
Consequently, the spectral norm of the Jacobian of $\sum_{i=1}^k q_i(\mathbf{x})\mathbf{b}_i$ with respect to the concatenation $\operatorname{vec}([\mathbf{b}_1,\ldots , \mathbf{b}_k,c_1,\ldots,c_k])$ is bounded by
\begin{align*}
    &\sum_{j=1}^k\biggl[ \big\Vert q_j(\mathbf{x})I+\sum_{i=1}^k q_i(\mathbf{x})(1- q_j(\mathbf{x}))\mathbf{b}^T_i\mathbf{x} \big\Vert_{\sigma}+ \big\Vert\sum_{i=1}^k q_i(\mathbf{x})(1- q_j(\mathbf{x}))\mathbf{b}_i \big\Vert_{\sigma}\biggr]\\
    \le\: & \sum_{j=1}^k \bigl[ q_j(\mathbf{x}) +(1-q_j(\mathbf{x}))(\Vert\mathbf{x}\Vert_2+1)\sum_{i=1}^k q_i(\mathbf{x})\Vert\mathbf{b}_i\Vert_2 \bigr]\\
    \le \: &1 + (k-1)(\Vert\mathbf{x}\Vert_2+1)\max_i\,\Vert\mathbf{b}_i\Vert_2. \numberthis
\end{align*}
Therefore, assuming $\max_i \Vert \mathbf{b}_i\Vert_2\le 1$ and $\max_i \Vert \mathbf{d}_i\Vert_2\le 1$ we have:
\begin{align*}
 \big\Vert \sum_{i=1}^k\bigl[ q_i(\mathbf{x})\mathbf{b}_i - p_i(\mathbf{x})\mathbf{d}_i\bigr] \big\Vert_2 \, &\le \, k(1 + \Vert\mathbf{x}\Vert_2)\sqrt{\sum_{i=1}^k\bigl[\Vert\mathbf{b}_i-\mathbf{d}_i\Vert^2_2+(c_i-e_i)^2\bigr]}, \numberthis \\
 \big\Vert \sum_{i=1}^k\bigl[ q_{k+i}(\mathbf{x})\mathbf{b}_{k+i} - p_i(\mathbf{x})\mathbf{d}_i\bigr] \big\Vert_2 \, &\le \, k(1 + \Vert\mathbf{x}\Vert_2)\sqrt{\sum_{i=1}^k\bigl[\Vert\mathbf{b}_{k+i}-\mathbf{d}_i\Vert^2_2+(c_{k+i}-e_i)^2\bigr]}. \numberthis
\end{align*}
Combining the above bounds with \eqref{app: Prop2: Lemma 2 Eq 1} and noting that for any three vectors $\mathbf{a},\mathbf{b},\mathbf{c}$ we have $\Vert\mathbf{a}+\mathbf{b}+\mathbf{c}\Vert^2_2\le 3(\Vert\mathbf{a}\Vert_2^2+\Vert\mathbf{b}\Vert_2^2+\Vert\mathbf{c}\Vert_2^2)$, we obtain 
\begin{align*}
 &\bigl\Vert \nabla_{\mathbf{x}} D_{A,(\mathbf{b}_i,c_i)_{i=1}^{2k}}(\mathbf{x}) \bigr\Vert^2_2 \\
 \le\, &  3\biggl(\Vert A\mathbf{x}\Vert^2_{2} + k^2(1+\Vert\mathbf{x}\Vert_2)^2\sum_{i=1}^k \bigl[ \Vert\mathbf{b}_i - \mathbf{d}_i \Vert^2_2 + \Vert\mathbf{b}_{k+i} - \mathbf{d}_i \Vert^2_2 +( c_i-e_i)^2 +  (c_{k+i}-e_i)^2\bigr] \biggr).  \numberthis
\end{align*}
The lemma is an immediate result of the above inequality, noting that $\Vert A \mathbf{x}\Vert_2 \le \Vert A\Vert_{\sigma}\Vert\mathbf{x}\Vert_2$.
\end{proof}
According to Lemma \ref{app: Lemma1: Prop2} and the proposition's assumptions,
\begin{align*}
    D^c_{A,(\mathbf{b}_i,c_i)_{i=1}^{2k}}(\mathbf{x}) &:= \, \max_{\mathbf{u}}\: D_{A,(\mathbf{b}_i,c_i)_{i=1}^{2k}}(\mathbf{x}+\mathbf{u}) - \frac{1}{2}\Vert \mathbf{u}\Vert^2_2 \\
    &\le \, \max_{\mathbf{u}}\: D_{A,(\mathbf{b}_i,c_i)_{i=1}^{2k}}(\mathbf{x})+\mathbf{u}^T\nabla D_{A,(\mathbf{b}_i,c_i)_{i=1}^{2k}}(\mathbf{x}) - \frac{1-\eta}{2}\Vert \mathbf{u}\Vert^2_2 \\
    &= \, D_{A,(\mathbf{b}_i,c_i)_{i=1}^{2k}}(\mathbf{x})+\frac{\bigl\Vert \nabla D_{A,(\mathbf{b}_i,c_i)_{i=1}^{2k}}(\mathbf{x}) \bigr\Vert^2_2}{2(1-\eta)}.\numberthis
\end{align*}
Applying Lemma \ref{app: Lemma 2: Prop2} to the above inequality shows that
\begin{align*}
    &D^c_{A,(\mathbf{b}_i,c_i)_{i=1}^{2k}}(\mathbf{x}) \le D_{A,(\mathbf{b}_i,c_i)_{i=1}^{2k}}(\mathbf{x}) \\
    &\quad +\frac{3k^2(1+\Vert\mathbf{x}\Vert_2)^2\bigl(\Vert A\Vert^2_{\sigma} + \sum_{i=1}^k \bigl[ \Vert\mathbf{b}_i - \mathbf{d}_i \Vert^2_2 + \Vert\mathbf{b}_{k+i} - \mathbf{d}_i \Vert^2_2 +( c_i-e_i)^2 +  (c_{k+i}-e_i)^2\bigr] \bigr)}{2(1-\eta)}.\numberthis
\end{align*}
Considering that $\frac{(1+\Vert\mathbf{x}\Vert_2)^2}{2}\le 1+\Vert\mathbf{x}\Vert^2_2$ and $\Vert A\Vert_{\sigma}\le \Vert A\Vert_{F}$, Proposition 2 directly follows from the above inequality.
%\end{proof}

\subsection{Proof of Theorem 2}
%\begin{thm*}
%Consider the GAT-GMM minimax problem in the main text with the additional constraint that $\Vert \Lambda \Vert^2_F + \max_i \Vert \boldsymbol{\mu}_i\Vert_2^2 +1 \le \eta$. Assume that $\mathbb{E}[\Vert \mathbf{X} \Vert^2_2] \le \eta < \frac{\lambda}{2}$. Then, the $(1,1)$-gradient descent ascent optimization algorithm with maximization stepsize $\alpha_{\max}=\frac{1}{\lambda+2\eta}$ and minimization stepsize $\alpha_{\min}= \frac{1}{\kappa^2L}$ for $L=2\lambda+4\eta + 10(k+1)\bigl(\frac{\eta}{\lambda} +\max_i\Vert\mathbf{d}_i\Vert^2_2 +  \bigr)$ and $\kappa=\frac{L}{\lambda-2\eta}$ will find an approximate stationary point such that $\bigl\Vert \nabla_{\operatorname{vec}(\Lambda,(\boldsymbol{\mu}_i)_{i=1}^k)} \mathcal{L}\bigl(\Lambda,(\boldsymbol{\mu}_i)_{i=1}^k\bigr) \bigr\Vert_2 \le \epsilon$ over $\mathcal{O}\bigl(\frac{\kappa L ((2\eta/\lambda)^2 + \kappa)}{\epsilon^2} \bigr)$ iterations.
%\end{thm*}
%\begin{proof}

For simplicity, we merge $\mathbf{b}_i$ and $c_i$ to get $\mathbf{b}'_i=[\mathbf{b}_i,c_i]$ for each $i$ by adding a constant feature $1$ to get feature vector $\mathbf{X}'=[\mathbf{X},1]$. Under the theorem's assumption, the maximization's objective will be $(\lambda-2\eta)$-strongly concave in its variables, since the trace and hence the maximum eigenvalue of $h(\mathbf{z})=\log(\sum_{i=1}\exp(z_i))$'s Hessian is bounded by $1$. Thus, the term $\mathbb{E}[h(\mathbf{B}\mathbf{X})]$ for matrix $B=\bigl[\mathbf{b}'_1;\cdots ;\mathbf{b}'_k\bigr]$ has a $(\mathbb{E}[\Vert\mathbf{X} \Vert^2_2]+1)$-Lipschitz derivative with respect to $\mathbf{B}$. Also, note that $\mathbb{E}[\Vert G_{\Lambda,\boldsymbol{\mu}}(\mathbf{Z})\Vert^2_2]\le \Vert \Lambda \Vert^2_F + \max_i \Vert \boldsymbol{\mu}_i\Vert_2^2 +1 \le \eta$. As a result, the maximization objective is a $(\lambda-2\eta)$-strongly concave function which has a $(\lambda+2\eta)$-Lipschitz derivative. Hence, $\kappa$ results in an upper-bound on the condition number of this maximization problem. Furthermore, since $\mathbf{1}^T\nabla h(\mathbf{z})=1$ the optimal $\mathbf{B}'$ 's Frobenius distance to $[[\mathbf{d}_1,e_1];\cdots;[\mathbf{d}_k,e_k]]$ is bounded by $\frac{\mathbb{E}[\Vert\mathbf{X}'\Vert_2]+\mathbb{E}[\Vert G(\mathbf{Z})'\Vert_2]}{\lambda}\le \frac{2{\eta}}{\lambda}$. 

Regarding the minimization variable, we note that according to Lemma \ref{app: Lemma1: Prop2} the minimax objective's derivative with respect to the generator's parameters $\Lambda,(\boldsymbol{\mu}_i)_{i=1}^k$ provides a Lipschitz function, considering the Frobenius distance, with its Lipschitz constant upper-bounded by 
\begin{align*}
    (1+\mathbb{E}[\Vert\mathbf{Z}\Vert^2_2])\bigl(\lambda_{\max}(A)+2\max_i\Vert\mathbf{b}_i\Vert^2_2\bigr) &= (k+1)\bigl(\lambda_{\max}(A)+2\max_i\Vert\mathbf{b}_i\Vert^2_2\bigr) \\
    &\stackrel{(a)}{\le} (k+1)\bigl(\frac{\eta}{\lambda} +4\max_i\Vert\mathbf{d}_i\Vert^2_2+4\frac{\eta}{\lambda} \bigr) \\
    &\le 5(k+1)\bigl(\frac{\eta}{\lambda} +\max_i\Vert\mathbf{d}_i\Vert^2_2 \bigr)
\end{align*}
where $(a)$ holds because the optimal $A$ and $\mathbf{b}_i$'s satisfy the following inequalities
\begin{align*}
    &\lambda_{\max}(A)= \frac{1}{\lambda}\lambda_{\max} \bigl(\mathbb{E}[\mathbf{X}\mathbf{X}^T] - \mathbb{E}[G(\mathbf{Z})G(\mathbf{Z})^T] \bigr) \le \frac{\eta}{\lambda}, \\
    &\Vert\mathbf{b}_i\Vert^2_2\le 2(\Vert\mathbf{b}_i-\mathbf{d}_i\Vert^2_2 +\Vert\mathbf{d}_i\Vert^2_2)\le 2\frac{\eta}{\lambda}+2\max_i \Vert\mathbf{d}_i\Vert^2_2.
\end{align*}
As a result, the minimax objective is $5(k+1)\bigl(\frac{\eta}{\lambda} +\max_i\Vert\mathbf{d}_i\Vert^2_2 \bigr)$-smooth in the generator's variables. Therefore, the objective will be $L$-smooth with respect to the vector containing both the minimization and maximization variables given 
\begin{equation*}
    L = 2\lambda+4\eta + 10(k+1)\bigl(\frac{\eta}{\lambda} +\max_i\Vert\mathbf{d}_i\Vert^2_2  \bigr).
\end{equation*}
Having the strong convexity-degree bounded by $\lambda-2\eta$, condition number $\kappa=\frac{\lambda+2\eta}{\lambda-2\eta}$, and diameter of the feasible set for the maximization problem $\frac{2{\eta}}{\lambda}$, and also the smoothness coefficient $L$ in the non-convex strongly-concave optimization problem, the theorem follows from Theorem 4.4 in \cite{lin2019gradient}.
%\end{proof}

\subsection{Proof of Theorem 3}
%\begin{thm*}
%Consider the minimax problem in \eqref{app: Eq: GMGAN Binary Case} for learning a symmetric two-component GMM $\frac{1}{2}\mathcal{N}(\boldsymbol{\mu}_{\mathbf{X}},\Sigma_{\mathbf{X}})+\frac{1}{2}\mathcal{N}(-\boldsymbol{\mu}_{\mathbf{X}},\Sigma_{\mathbf{X}})$. Suppose that $(\boldsymbol{\mu}_{\mathbf{X}},\Sigma_{\mathbf{X}})$ satisfies Condition \ref{Condition: SNR} along vector $\mathbf{c}$. Then, $(\boldsymbol{\mu},\Sigma)=(\boldsymbol{\mu}_{\mathbf{X}},\Sigma_{\mathbf{X}})$ is the only $\mathcal{L}(G_{\Lambda,\boldsymbol{\mu}})$'s minimum solution satisfying Condition \ref{Condition: SNR} along $\mathbf{c}$.  
%\end{thm*}
%\begin{proof}

\begin{lemma}\label{app: Lemma1 Thm3}
Considering $\tanh(\mathbf{x})=\frac{\exp(\mathbf{x})-\exp(-\mathbf{x})}{\exp(\mathbf{x})+\exp(-\mathbf{x})}$ and a scalar Gaussian variable $X$, we have
\begin{equation}
    \mathbb{E}[X]\mathbb{E}[\tanh(X)] - \mathbb{E}[X]^2\mathbb{E}[\tanh'(X)] \, \ge\, 0
\end{equation}
where the inequality holds with equality if and only if $\mathbb{E}[X]=0$. 
\end{lemma}
\begin{proof}
Without loss of generality, we suppose $X=\mu+\sigma Z$ where we denote $X$'s mean and standard deviation using $\mu$ and $\sigma$ and $Z\sim N(0,1)$ has a standard Gaussian distribution. Note that $\tanh$ is an odd function, and therefore its derivative $\tanh'$ is even. Since for $\mu=0$ $\mathbf{X}$'s density function is symmetric around zero, if $\mu=0$ we have
\begin{equation}
    \mathbb{E}[\tanh(X)] - \mathbb{E}[X]\mathbb{E}[\tanh'(X)]=0.
\end{equation}
Note that in general
\begin{equation}
   \mathbb{E}[\tanh(X)] - \mathbb{E}[X]\mathbb{E}[\tanh'(X)] = E[\tanh(\mu+\sigma Z)] - \mu\mathbb{E}[\tanh'(\mu+\sigma Z)].
\end{equation}
\textbf{Claim:} $g_{\sigma}(\mu):=E[\tanh(\mu+\sigma Z)] - \mu\mathbb{E}[\tanh'(\mu+\sigma Z)]$ is strictly increasing with $\mu$.

To prove this claim, we consider $g_{\sigma}$'s derivative:
\begin{equation}
  g_{\sigma}'(\mu) =  E[\tanh'(\mu+\sigma Z)] - \mathbb{E}[\tanh'(\mu+\sigma Z)] - \mu \mathbb{E}[\tanh''(\mu+\sigma Z)] = - \mu \mathbb{E}[\tanh''(\mu+\sigma Z)].
\end{equation}
However, $\tanh''$ is an odd function which takes negative values for positive inputs. As a result, $\mathbb{E}[\tanh''(\mu+\sigma Z)]>0$ for any $\mu >0 $, because
\begin{align*}
   \mathbb{E}[\tanh''(\mu+\sigma Z)] &= \int_{-\infty}^{\infty} p_Z(z)\tanh''(\mu+\sigma z)\mathbf{d}z\\
   &= \int_{0}^{\infty} p_Z(z)\bigl(\tanh''(\mu+\sigma z) + \tanh''(\mu-\sigma z) \bigr) \mathbf{d}z\\
    &> 0. \numberthis
\end{align*}
Similarly, one can show $\mathbb{E}[\tanh''(\mu+\sigma Z)]$ is positive given a negative $\mu$. Hence, $g'_{\sigma}(\mu)=-\mu \mathbb{E}[\tanh''(\mu+\sigma Z)]$ is positive everywhere except at $\mu=0$ where it becomes zero. This implies that $g_{\sigma}(\mu)$ is strictly increasing in $\mu$ and the claim is valid. 

Since the claim holds and as shown earlier $g_{\sigma}(0)=0$, we always have
\begin{equation}
    \mu g_{\sigma}(\mu) \ge 0
\end{equation}
where the inequality holds with equality only if $\mu=0$. The lemma's proof is therefore complete.
\end{proof}
\begin{lemma}\label{app: Lemma2: Thm1}
Considering $\tanh(\mathbf{x})=\frac{\exp(\mathbf{x})-\exp(-\mathbf{x})}{\exp(\mathbf{x})+\exp(-\mathbf{x})}$ and a univariate Gaussianly-distributed $X$ with $\mathbb{E}[X]\ge 0$, then
\begin{equation}
    2\mathbb{E}[\tanh''(X)] + \mathbb{E}[\tanh'''(X)] \, \le\, 0.
\end{equation}
\end{lemma}
\begin{proof}
To show this lemma, note that $\tanh''$ is an odd function whereas $\tanh'''$ is an even function. Without loss of generality, suppose $X=\mu +\sigma Z$ for standard Gaussian $Z\sim\mathcal{N}(0,1)$ and $\mu\ge 0$. Note that if $\mu=0$, we have
\begin{align*}
   2\mathbb{E}[\tanh''(X)] + \mathbb{E}[\tanh'''(X)] = \mathbb{E}[\tanh'''(X)]  < 0. \numberthis
\end{align*}
Here, the last equality holds because $\tanh''$ is odd while the density function $p_X(x)$ is even given $\mu=0$. Also, $\mathbb{E}[\tanh'''(X)]<0$ for a zero-mean Gaussian $X$, because $\tanh'''(x)$ is an even function which has only one zero in the positive side ($x>0$) at $x_0=\frac{1}{2}\cosh^{-1}(2)$. For $0\le x<x_0$ $\tanh'''(x)<0$ is negative, and for $x_0<x$ we have $\tanh'''(x)>0$. This is while a zero-mean univariate normal density is strictly decreasing over $[0,+\infty)$ and hence
\begin{align*}
    \mathbb{E}[\tanh'''(X)]&=\int_{-\infty}^{+\infty}P_Z(z)\tanh'''(\sigma z)\mathbf{d}z \\
    & = 2\int_{0}^{+\infty}P_Z(z)\tanh'''(\sigma z)\mathbf{d}z \\
    & =2\bigl( \int_{0}^{x_0/\sigma}P_Z(z)\tanh'''(\sigma z)\mathbf{d}z+ \int_{x_0/\sigma}^{+\infty}P_Z(z)\tanh'''(\sigma z)\mathbf{d}z \bigr)  \\
    & < 2P_Z(x_0/\sigma)\bigl( \int_{0}^{x_0/\sigma}\tanh'''(\sigma z)\mathbf{d}z + \int_{x_0/\sigma}^{+\infty}\tanh'''(\sigma z)\mathbf{d}z \bigr) \\
    & = 0. \numberthis
\end{align*}
Here, the last equality holds because $\int_{0}^{\infty}\tanh'''(\sigma z)\mathbf{d}z=\tanh''(+\infty)-\tanh''(0) = 0$. Also, the last inequality holds because $\tanh'''(x)$ is positive over $(x_0,+\infty)$ and negative over $[0,x_0)$, while the positive $P_Z(\mathbf{z})$ is strictly decreasing over $[0,\infty).$ Therefore, the lemma holds for $\mu =0$.

We provide a similar proof for $\mu >0$. Define
\begin{equation}
    g_{\mu}(x) := 2\bigl[\tanh''(\mu+x)+\tanh''(\mu-x)\bigr] + \bigl[\tanh'''(\mu+x)+\tanh'''(\mu-x)\bigr]. 
\end{equation}
Clearly, $g_\mu$ is an even function for every $\mu$. 

\textbf{Claim:} If $\mu\ge 0$, there exists some $x_\mu>0$ for which we have $g_{\mu}(x)<0$ for every $0\le x<x_\mu$ and $g_{\mu}(x)>0$ for every $x_\mu<x$.

To show this claim, note that we have already proven this claim for $\mu=0$. If $\mu >0$, then note that there exists a unique $x_\mu$ such that $\mu+\tanh^{-1}(\frac{1}{3})<x_\mu<\mu-\log(\frac{7-\sqrt{33}}{8})/2$ and $g_{\mu}(x_{\mu})=0$. This is because $2\tanh''(x)+\tanh'''(x)=\frac{2(-1 + \tanh(x)) (1 + 3 \tanh(x))}{\cosh^2(x)}$ has only one zero at $-\tanh^{-1}(\frac{1}{3})$ above which it takes negative values and below which it takes positive values where it takes its maximum value at $x_{\max}=\log(\frac{7-\sqrt{33}}{8})/2$. Therefore, since $2\tanh''(x)+\tanh'''(x)$ monotonically changes from its maximum value to $0$ over $[\log(\frac{7-\sqrt{33}}{8})/2,-\tanh^{-1}(\frac{1}{3})]$, such a unique $x_\mu$ will exist. We claim that no other $x$ can achieve the same value. This is because $\tanh'''$ is even and $\tanh''$ is odd, resulting in the following inequality for every $\mu>0$
\begin{equation*}
    2\bigl[\tanh''(\mu+x_\mu)+\tanh''(-\mu-x_\mu)\bigr] + \bigl[\tanh'''(\mu+x_\mu)+\tanh'''(-\mu-x_\mu)\bigr] = 2 \tanh''(\mu+x_\mu)> 0
\end{equation*}
where the inequality holds because $x_{\mu}\ge x_0$ for every positive $\mu$. As a result, for any $x>\mu - x_{\max}$ we have
\begin{align*}
   g_{\mu}(x) &:= 2\bigl[\tanh''(\mu+x)+\tanh''(\mu-x)\bigr] + \bigl[\tanh'''(\mu+x)+\tanh'''(\mu-x)\bigr] \\
   &\ge 2\bigl[\tanh''(\mu+x)+\tanh''(-\mu-x)\bigr] + \bigl[\tanh'''(\mu+x)+\tanh'''(-\mu-x)\bigr] \\
   & = 2 \tanh'''(\mu+x) \\
   & > 0. \numberthis
\end{align*}
As a result, the even $g_\mu(x)$ is negative at every $|x|<x_\mu$ and takes positive values if $|x|>x_\mu$. Therefore, we have
\begin{align*}
    &\mathbb{E}[2\tanh''(X)+\tanh'''(X)] \\
    =\, & \int_{-\infty}^{+\infty} p_Z(z)[2\tanh''(\mu+\sigma Z)+\tanh'''(\mu+\sigma z)]\mathbf{d}z \\
    =\, & \int_{0}^{+\infty} p_Z(z)g_{\mu}(\sigma z)\mathbf{d}z \\
    = \, &\bigl( \int_{0}^{x_{\mu}/\sigma} p_Z(z)g_{\mu}(\sigma z)\mathbf{d}z+ \int_{x_{\mu}/\sigma}^{+\infty} p_Z(z)g_{\mu}(\sigma z)\mathbf{d}z \bigr) \\
    < \, & p_Z(x_{\mu}/\sigma) \bigl( \int_{0}^{x_{\mu}/\sigma} g_{\mu}(\sigma z)\mathbf{d}z + \int_{x_{\mu}/\sigma}^{\infty} g_{\mu}(\sigma z)\mathbf{d}z \bigr) \\
    = \, & p_Z(x_{\mu}/\sigma)  \int_{0}^{\infty} g_{\mu}(\sigma z)\mathbf{d}z \\
    = \, & 0. \numberthis
\end{align*}
In the above equations, the inequality holds because $P_Z(z)$ is strictly decreasing over $[0,+\infty)$ and as earlier shown $g_\mu(x)$ takes negative values over $[0,x_\mu)$ and positive values over $(x_\mu,+\infty)$. Moreover, the final equality holds, since $g_{\mu}(x)$ is defined as the even part of $2\tanh''(x+\mu)+\tanh'''(x+\mu)$, and hence its integral is the odd part of $2\tanh'(x+\mu)+\tanh''(x+\mu)$. However, both $\tanh'$ and $\tanh''$ go to zero as $x$ approaches either $-\infty$ or $+\infty$. Therefore, $\int_0^{+\infty}g_{\mu}(x)\mathbf{d}(x)=\frac{1}{2}\int_{-\infty}^{+\infty}g_{\mu}(x)\mathbf{d}(x) = 0$. Therefore, the lemma's proof is complete.   
\end{proof}
\begin{lemma}\label{app: Lemma3: Thm1}
Define $g(\mu)=\mathbb{E}\bigl[(\mu+\sqrt{s-\mu^2}Z)\tanh(\mu+\sqrt{s-\mu^2}Z) \bigr]$ where $Z\sim\mathcal{N}(0,1)$ is distributed according to a standard Gaussian distribution. Then, for any $0<\mu\le \sqrt{s}$ satisfying $2\mu^2 - 2\mu\sqrt{s-\mu^2} \ge s$, we have $g'(\mu)>0$.
\end{lemma}
\begin{proof}
We use Stein's lemma to further simplify $g'(\mu)$:
\begin{align*}
    g'(\mu)&=\bigl(1-\frac{\mu}{\sqrt{s-\mu^2}}Z\bigr)\bigl( \mathbb{E}\bigl[\tanh(\mu+\sqrt{s-\mu^2}Z)+(\mu+\sqrt{s-\mu^2}Z)\tanh'(\mu+\sqrt{s-\mu^2}Z)\bigr]\bigr) \\
    &\stackrel{(a)}{=} \mathbb{E}[\tanh(\mu+\sqrt{s-\mu^2}Z)] + \mu\mathbb{E}[\tanh'(\mu+\sqrt{s-\mu^2}Z)] \\
    &\quad +(s-\mu^2)\mathbb{E}[\tanh''(\mu+\sqrt{s-\mu^2}Z)] -\mu\mathbb{E}[\tanh'(\mu+\sqrt{s-\mu^2}Z)]\\
    &\quad -\mu^2\mathbb{E}[\tanh''(\mu+\sqrt{s-\mu^2}Z))]-\mu\mathbb{E}[Z^2\tanh'(\mu+\sqrt{s-\mu^2}Z)] \\
    &\stackrel{(b)}{=} \mathbb{E}[\tanh(\mu+\sqrt{s-\mu^2}Z)] +(s-2\mu^2) \mathbb{E}[\tanh''(\mu+\sqrt{s-\mu^2}Z))] \\
    &\quad -\mu \mathbb{E}[\tanh'(\mu+\sqrt{s-\mu^2}Z)+\sqrt{s-\mu^2}Z\tanh''(\mu+\sqrt{s-\mu^2}Z)] \\
    &\stackrel{(c)}{=} \mathbb{E}[\tanh(\mu+\sqrt{s-\mu^2}Z)] -\mu \mathbb{E}[\tanh'(\mu+\sqrt{s-\mu^2}Z)] \\
    &\quad +(s-2\mu^2) \mathbb{E}[\tanh''(\mu+\sqrt{s-\mu^2}Z))] -\mu(s-\mu^2) \mathbb{E}[\tanh'''(\mu+\sqrt{s-\mu^2}Z)]. \numberthis
\end{align*}
Here, $(a)$, $(b)$, and $(c)$ follow from Stein's lemma showing that for a normally distributed $X\sim\mathcal{N}(\mu,\sigma^2)$ and differentiable $h$, the following holds:
\begin{equation}
    \mathbb{E}[(X-\mu)g(X)]=\sigma^2\mathbb{E}[g'(\mathbf{X})]. 
\end{equation}
As a result, we have
\begin{align}\label{app: Thm 1, Lemma 3 's proof}
    g'(\mu) &= \mathbb{E}[\tanh(\mu+\sqrt{s-\mu^2}Z)] -\mu \mathbb{E}[\tanh'(\mu+\sqrt{s-\mu^2}Z)] \nonumber\\
    &\quad +(s-2\mu^2) \mathbb{E}[\tanh''(\mu+\sqrt{s-\mu^2}Z))] -\mu(s-\mu^2) \mathbb{E}[\tanh'''(\mu+\sqrt{s-\mu^2}Z)] 
\end{align}
According to Lemma \ref{app: Lemma1 Thm3}, $\mathbb{E}[\tanh(\mu+\sqrt{s-\mu^2}Z)] -\mu \mathbb{E}[\tanh'(\mu+\sqrt{s-\mu^2}Z)] >0$ for $\mu>0$. Also, as shown in Lemma \ref{app: Lemma2: Thm1}, under the assumption that $\mu > 0$
\begin{align}
    &(s-2\mu^2) \mathbb{E}[\tanh''(\mu+\sqrt{s-\mu^2}Z))] -\mu(s-\mu^2) \mathbb{E}[\tanh'''(\mu+\sqrt{s-\mu^2}Z)] \nonumber \\
    \ge \, &(s-2\mu^2- \mu(s-\mu^2))\mathbb{E}[\tanh'''(\mu+\sqrt{s-\mu^2}Z)].
\end{align}
As a result, $g'(\mu)>0$ if $s-2\mu^2- \mu(s-\mu^2)<0$ which completes the Lemma's proof.
\end{proof}
To show the theorem, we decompose $\mathcal{L}(G_{\Lambda,\boldsymbol{\mu}})=\mathcal{L}_1(G_{\Lambda,\boldsymbol{\mu}})+\mathcal{L}_2(G_{\Lambda,\boldsymbol{\mu}})$ into the following two components
\begin{align*}
    \mathcal{L}_1(G_{\Lambda,\boldsymbol{\mu}}):&= \max_A \mathbb{E}[\frac{1}{2}\mathbf{X}^TA\mathbf{X}] - \mathbb{E}[\frac{1}{2}G_{\Lambda,\boldsymbol{\mu}}(\mathbf{Z})^TAG_{\Lambda,\boldsymbol{\mu}}(\mathbf{Z})] + \frac{\lambda}{2}\Vert A\Vert^2_F \\
    &= \frac{1}{2\lambda}\big\Vert \mathbb{E}[\mathbf{X}\mathbf{X}^T]- \mathbb{E}[G_{\Lambda,\boldsymbol{\mu}}(\mathbf{Z})G_{\Lambda,\boldsymbol{\mu}}(\mathbf{Z})^T] \big\vert^2_F, \numberthis \\
    \mathcal{L}_2(G_{\Lambda,\boldsymbol{\mu}}) :&= \max_{\mathbf{b}_1,\mathbf{b}_2,\mathbf{b}_3,\mathbf{b}_4}\: \mathbb{E}\bigl[\log\bigl(\frac{\exp(\mathbf{b}_1^T\mathbf{X})+\exp(\mathbf{b}_2^T\mathbf{X})}{\exp(\mathbf{b}_3^T\mathbf{X})+\exp(\mathbf{b}_4^T\mathbf{X})}\bigr) \bigr] \\
    &\quad - \mathbb{E}\bigl[\log\bigl(\frac{\exp(\mathbf{b}_1^TG_{\Lambda,\boldsymbol{\mu}}(\mathbf{Z}))+\exp(\mathbf{b}_2^TG_{\Lambda,\boldsymbol{\mu}}(\mathbf{Z}))}{\exp(\mathbf{b}_3^TG_{\Lambda,\boldsymbol{\mu}}(\mathbf{Z}))+\exp(\mathbf{b}_4^TG_{\Lambda,\boldsymbol{\mu}}(\mathbf{Z}))}\bigr) \bigr] \\
    &\quad -\frac{\lambda}{2}\biggl(\Vert\mathbf{b}_1-\mathbf{d}\Vert^2_2 + \Vert\mathbf{b}_2+\mathbf{d}\Vert^2_2 + \Vert\mathbf{b}_3-\mathbf{d}\Vert^2_2+ \Vert\mathbf{b}_4+\mathbf{d}\Vert^2_2 \biggr). \numberthis
\end{align*}
Note that both $\mathcal{L}_1(G_{\Lambda,\boldsymbol{\mu}})$ and $\mathcal{L}_2(G_{\Lambda,\boldsymbol{\mu}})$ are non-negative, and hence their summation will be minimized if and only if both of them are zero. Notice that the zero value is achievable at $(\boldsymbol{\mu},\Sigma)=(\boldsymbol{\mu}_{\mathbf{X}},\Sigma_{\mathbf{X}})$.

Therefore, the necessary and sufficient conditions to achieve the minimum zero value is $\mathcal{L}_1(G_{\Lambda,\boldsymbol{\mu}})=0$, i.e. 
\begin{equation}\label{app: Thm3: Condition EX2}
\mathbb{E}[G_{\Lambda,\boldsymbol{\mu}}(\mathbf{Z})G_{\Lambda,\boldsymbol{\mu}}(\mathbf{Z})^T] = \mathbb{E}[\mathbf{X}\mathbf{X}^T], 
\end{equation} 
 and $\mathcal{L}_2(G_{\Lambda,\boldsymbol{\mu}})=0$ which holds only if the optimal maximization variables satisfy $\mathbf{b}_1=\mathbf{b}_3=-\mathbf{b}_2=-\mathbf{b}_4=\mathbf{d}$. Setting the gradient of the maximization objective to be zero at these values implies the following necessary and sufficient condition:
\begin{equation}\label{app: Thm3: Condition EXtanhx}
    \mathbb{E}\bigl[ \mathbf{X} \tanh(\mathbf{d}^T\mathbf{X})\bigr] =\mathbb{E}\bigl[ G_{\Lambda,\boldsymbol{\mu}}(\mathbf{Z}) \tanh(\mathbf{d}^TG_{\Lambda,\boldsymbol{\mu}}(\mathbf{Z}))\bigr]. 
\end{equation}
Notice that both $\mathbb{E}[\mathbf{X}\mathbf{X}^T]$ and $\mathbb{E}[\mathbf{X}\tanh(\mathbf{d}^T\mathbf{X})]$ represent expectations of even functions which are shared by the two symmetric Gaussian components in the underlying GMM and the generator's output GMM. Therefore, we only need to guarantee that the mean and covariance parameters of $\widetilde{\mathbf{X}}\sim\mathcal{N}(\boldsymbol{\mu}_{\mathbf{X}},\Sigma_{\mathbf{X}})$ are uniquely characterized by $\mathbb{E}[\widetilde{\mathbf{X}}\widetilde{\mathbf{X}}^T]$ and $\mathbb{E}[\widetilde{\mathbf{X}}\tanh(\mathbf{d}^T\widetilde{\mathbf{X}})]$.

First of all note that for $Y=\mathbf{d}^T\widetilde{\mathbf{X}}$, we can calculate $\mathbb{E}[Y^2]$ and $\mathbb{E}[Y\tanh(Y)]$ as
\begin{align*}
    \mathbb{E}[Y^2] &=\mathbf{d}^T \mathbb{E}[\widetilde{\mathbf{X}}\widetilde{\mathbf{X}}^T]\mathbf{d}, \\
    \mathbb{E}[Y\tanh(Y)] &= \mathbf{d}^T \mathbb{E}[\widetilde{\mathbf{X}}\tanh(\mathbf{d}^T\widetilde{\mathbf{X}})].\numberthis
\end{align*}
Therefore, due to the well-separable condition assumed in the theorem, Lemma \ref{app: Lemma3: Thm1} implies that $\mathbb{E}[Y\tanh(Y)]$ is a strictly increasing function of $\mathbb{E}[Y]$ conditioned to a fixed $\mathbb{E}[Y^2]$. As a result, there is only one Gaussian distribution $Y\sim\mathcal{N}(\mathbf{d}^T\boldsymbol{\mu}_{\mathbf{X}},\mathbf{d}^T\Sigma_\mathbf{X}\mathbf{d})$ that satisfies the well-separable condition which at the same time matches the known $\mathbb{E}[Y^2]$ and $\mathbb{E}[Y\tanh(Y)]$. 

Hence, $G_{\Lambda,\boldsymbol{\mu}}$ will match the underlying GMM along direction $\mathbf{d}$. We further claim that this result also implies $G_{\Lambda,\boldsymbol{\mu}}$ will match the underlying GMM along any direction $\widetilde{\mathbf{d}}$ and so has the same distribution as $\mathbf{X}$. To show this, note that if we denote $\widetilde{Y}=\widetilde{\mathbf{d}}^TG_{\Lambda,\boldsymbol{\mu}(\mathbf{Z})}\sim\mathcal{N}(\widetilde{\mathbf{d}}^T\boldsymbol{\mu},\widetilde{\mathbf{d}}^T\Lambda \Lambda^T\widetilde{\mathbf{d}})$. Note that we have
\begin{align}
    \mathbb{E}[\widetilde{Y}^2]&=\widetilde{\mathbf{d}}^T\mathbb{E}[\mathbf{X}\mathbf{X}^T]\widetilde{\mathbf{d}} \\
    \mathbb{E}[\widetilde{Y}Y]&=\widetilde{\mathbf{d}}^T\mathbb{E}[\mathbf{X}\mathbf{X}^T]{\mathbf{d}} \\
    \mathbb{E}[\widetilde{Y}\tanh(Y)]&=\widetilde{\mathbf{d}}^T\mathbb{E}[\mathbf{X}\tanh(\mathbf{d}^T\mathbf{X})].
\end{align}
We can evaluate all the above moments using the matched $\mathbb{E}[\mathbf{X}\mathbf{X}^T]$ and $\mathbb{E}[\mathbf{X}\tanh(\mathbf{d}^T\mathbf{X})]$. We claim that the above moments uniquely characterize the mean and variance of $\widetilde{Y}$. This is because $(Y,\widetilde{Y})$ have a jointly-Gaussian distribution and hence we can write $\widetilde{Y}=aY+V+u$ where $a$ and $u$ are constants and $V\sim\mathcal{N}(0,\sigma^2_V)$ is a zero-mean Gaussian variable independent from $Y$. Then, since $\mathbb{E}[VY]=\mathbb{E}[V\tanh(Y)]=0$ we have
\begin{align}\label{app: Thm3: Systems of Equations}
    \mathbb{E}[\widetilde{Y}Y]&=\mathbb{E}[Y^2]a+\mathbb{E}[Y]u,\\
    \mathbb{E}[\widetilde{Y}\tanh(Y)]&=\mathbb{E}[Y\tanh(Y)]a+\mathbb{E}[\tanh(Y)]u.
\end{align}
Note that if we consider the determinant of the above linear system of two equations and two unknowns ($a,u$), according to Stein's lemma we have
\begin{align*}
    &\mathbb{E}[Y^2]\mathbb{E}[\tanh(Y)]-\mathbb{E}[Y]\mathbb{E}[Y\tanh(Y)] \\
    =\, & \mathbb{E}[Y^2]\mathbb{E}[\tanh(Y)]-\mathbb{E}[Y]^2\mathbb{E}[\tanh(Y)] - \mathbb{E}[Y]\sigma^2_Y\mathbb{E}[\tanh'(Y)] \\
    =\, & \sigma^2_Y\bigl(\mathbb{E}[\tanh(Y)] - \mathbb{E}[Y] \mathbb{E}[\tanh'(Y)]\bigr) \numberthis
\end{align*}
which, due to Lemma \ref{app: Lemma1 Thm3}, is positive if $\mathbb{E}[Y]>0$. This condition will hold as a result of well-separated components implying that $\mathbb{E}[\mathbf{X}]$ is not orthogonal to $\mathbf{d}$. As a result, \eqref{app: Thm3: Systems of Equations} uniquely characterizes $a,u$ in terms of the known information. Knowing the value of $a,u$ we can furthr evaluate the variance of $V$ using $\mathbb{E}[\widetilde{Y}^2]$ which shows that the distribution of $\widetilde{Y}:=\widetilde{\mathbf{d}}^T\mathbf{X}$ can be uniquely characterized from $\mathbb{E}[\mathbf{X}\mathbf{X}^T]$ and $\mathbb{E}[\mathbf{X}\tanh(\mathbf{d}^T\mathbf{X})]$. Therefore, there exists only one probability distribution matching $\mathbb{E}[\mathbf{X}\mathbf{X}^T]$ and $\mathbb{E}[\mathbf{X}\tanh(\mathbf{d}^T\mathbf{X})]$ with the underlying GMM and satisfying the well-separable components assumption which will be the underlying GMM. The theorem's proof is hence complete. 

%\end{proof}

\subsection{Proof of Theorem 4}
%\begin{thm*}
%Consider the GMM learning setting in Theorem \ref{Thm: GM-GAN approximation}. Suppose that for any feasible $G_{\Lambda,\boldsymbol{\mu}}$ Condition \ref{Condition: SNR} holds for $(\boldsymbol{\mu}, \Sigma)$ along any $\widetilde{\mathbf{c}}$ in a $\frac{\rho}{\lambda}$-ball around $\mathbf{c}$, i.e., $\Vert \widetilde{\mathbf{c}} - \mathbf{c}\Vert_2\le \frac{\rho }{\lambda}$, with $\rho$ denoting the maximum $\mathbb{E}[\Vert G_{\Lambda,\boldsymbol{\mu}}(\mathbf{Z})\Vert_2]$ over feasible $G_{\Lambda,\boldsymbol{\mu}}$'s. Also, assume that the coefficient $\lambda$ satisfies $\mathbb{E}[\Vert\mathbf{X}\Vert^2_2]+\mathbb{E}[\Vert G_{\Lambda,\boldsymbol{\mu}}(\mathbf{Z})\Vert^2_2]\le\lambda$ for all feasible $\Lambda,\boldsymbol{\mu}$'s. Then, $(\boldsymbol{\mu}_{\mathbf{X}},\Sigma_\mathbf{X})$ is the only stationary point in the feasible set where $\nabla_{\operatorname{vec}(\boldsymbol{\mu},\Lambda)} \mathcal{L}(G_{\Lambda,\boldsymbol{\mu}}) = \mathbf{0}$.
%\end{thm*}
%\begin{proof}
Similar to Theorem 3's proof, we write $\mathcal{L}(G_{\Lambda,\boldsymbol{\mu}})$ as the summation of the following two non-negative components
\begin{align*}
    \mathcal{L}_1(G_{\Lambda,\boldsymbol{\mu}}):&= \max_A\: \mathbb{E}[\frac{1}{2}\mathbf{X}^TA\mathbf{X}] - \mathbb{E}[\frac{1}{2}G_{\Lambda,\boldsymbol{\mu}}(\mathbf{Z})^TAG_{\Lambda,\boldsymbol{\mu}}(\mathbf{Z})] + \frac{\lambda}{2}\Vert A\Vert^2_F \\
    &= \frac{1}{2\lambda}\big\Vert \mathbb{E}[\mathbf{X}\mathbf{X}^T]- \mathbb{E}[G_{\Lambda,\boldsymbol{\mu}}(\mathbf{Z})G_{\Lambda,\boldsymbol{\mu}}(\mathbf{Z})^T] \big\Vert^2_F, \numberthis \\
    \mathcal{L}_2(G_{\Lambda,\boldsymbol{\mu}}) :&= \max_{\mathbf{b}_1,\mathbf{b}_2,\mathbf{b}_3,\mathbf{b}_4}\: \mathbb{E}\bigl[\log\bigl(\frac{\exp(\mathbf{b}_1^T\mathbf{X})+\exp(\mathbf{b}_2^T\mathbf{X})}{\exp(\mathbf{b}_3^T\mathbf{X})+\exp(\mathbf{b}_4^T\mathbf{X})}\bigr) \bigr] \\
    &\quad - \mathbb{E}\bigl[\log\bigl(\frac{\exp(\mathbf{b}_1^TG_{\Lambda,\boldsymbol{\mu}}(\mathbf{Z}))+\exp(\mathbf{b}_2^TG_{\Lambda,\boldsymbol{\mu}}(\mathbf{Z}))}{\exp(\mathbf{b}_3^TG_{\Lambda,\boldsymbol{\mu}}(\mathbf{Z}))+\exp(\mathbf{b}_4^TG_{\Lambda,\boldsymbol{\mu}}(\mathbf{Z}))}\bigr) \bigr] \\
    &\quad -\frac{\lambda}{2}\biggl(\Vert\mathbf{b}_1-\mathbf{d}\Vert^2_2 + \Vert\mathbf{b}_2+\mathbf{d}\Vert^2_2 + \Vert\mathbf{b}_3-\mathbf{d}\Vert^2_2+ \Vert\mathbf{b}_4+\mathbf{d}\Vert^2_2 \biggr). \numberthis
\end{align*}
Note that the maximization objective of $\mathcal{L}_2(G_{\Lambda,\boldsymbol{\mu}})$ is $\lambda-\mathbb{E}[\Vert\mathbf{X} \Vert^2_2]-\mathbb{E}[\Vert G_{\Lambda,\boldsymbol{\mu}}(\mathbf{Z}) \Vert^2_2]$-strongly concave in its variables, since $\mathbb{E}[\log(\exp(a_1)+\exp(a_2))]$'s Hessian's maximum eigenvalue is upper-bounded by $1$. As a result, $\mathcal{L}_2(G_{\Lambda,\boldsymbol{\mu}})$ is the optimal value of maximizing a strongly-concave objective which has a unique solution.

Since we assume both $G_{\Lambda,\boldsymbol{\mu}}(\mathbf{Z}),\mathbf{X}$ are zero-mean (because of their symmetric components), the optimal solution $\mathbf{b}_i$'s will satisfy $\mathbf{b}_2=-\mathbf{b}_1$ and $\mathbf{b}_3=-\mathbf{b}_4$ for any feasible $G_{\Lambda,\boldsymbol{\mu}}$. This is because $\mathbb{E}[(\frac{\mathbf{b}_1+\mathbf{b}_2}{2})^T\mathbf{X}]=\mathbb{E}[(\frac{\mathbf{b}_1+\mathbf{b}_2}{2})^TG_{\Lambda,\boldsymbol{\mu}}(\mathbf{Z})]=0$ and $\mathbb{E}[(\frac{\mathbf{b}_3+\mathbf{b}_4}{2})^T\mathbf{X}]=\mathbb{E}[(\frac{\mathbf{b}_3+\mathbf{b}_4}{2})^TG_{\Lambda,\boldsymbol{\mu}}(\mathbf{Z})]=0$ always hold, and in addition
\begin{align}
    \Vert\mathbf{b}_1-\mathbf{d}\Vert^2_2 + \Vert\mathbf{b}_2+\mathbf{d}\Vert^2_2 &= 2\big\Vert \frac{\mathbf{b}_1-\mathbf{b}_2}{2}-\mathbf{d}\big\Vert^2_2 + 2\big\Vert \frac{\mathbf{b}_1+\mathbf{b}_2}{2}\big\Vert^2_2, \numberthis\\
    \Vert\mathbf{b}_3-\mathbf{d}\Vert^2_2 + \Vert\mathbf{b}_4+\mathbf{d}\Vert^2_2 &= 2\big\Vert \frac{\mathbf{b}_3-\mathbf{b}_4}{2}-\mathbf{d}\big\Vert^2_2 + 2\big\Vert \frac{\mathbf{b}_3+\mathbf{b}_4}{2}\big\Vert^2_2.
\end{align}
Hence, in analyzing $\mathcal{L}(G_{\Lambda,\boldsymbol{\mu}})$'s stationary points 
for which according to the Danskin's theorem \cite{bernhard1995theorem} we need to apply optimal $\mathbf{b}_i$'s, without loss of generality we simplify $\mathcal{L}_2(G_{\Lambda,\boldsymbol{\mu}})$ as 
\begin{align*}
       \mathcal{L}_2(G_{\Lambda,\boldsymbol{\mu}}) :&= \max_{\mathbf{b}_1,\mathbf{b}_3}\: \mathbb{E}\bigl[\log\bigl(\frac{\exp(\mathbf{b}_1^T\mathbf{X})+\exp(-\mathbf{b}_1^T\mathbf{X})}{\exp(\mathbf{b}_3^T\mathbf{X})+\exp(-\mathbf{b}_3^T\mathbf{X})}\bigr) \bigr] \\
    &\quad - \mathbb{E}\bigl[\log\bigl(\frac{\exp(\mathbf{b}_1^TG_{\Lambda,\boldsymbol{\mu}}(\mathbf{Z}))+\exp(-\mathbf{b}_1^TG_{\Lambda,\boldsymbol{\mu}}(\mathbf{Z}))}{\exp(\mathbf{b}_3^TG_{\Lambda,\boldsymbol{\mu}}(\mathbf{Z}))+\exp(-\mathbf{b}_3^TG_{\Lambda,\boldsymbol{\mu}}(\mathbf{Z}))}\bigr) \bigr]  -\lambda\bigl(\Vert\mathbf{b}_1-\mathbf{d}\Vert^2_2 + \Vert\mathbf{b}_3-\mathbf{d}\Vert^2_2\bigr).\numberthis 
\end{align*}
Notice that the above maximization objective and also $\mathbb{E}[G_{\Lambda,\boldsymbol{\mu}}(\mathbf{Z})G_{\Lambda,\boldsymbol{\mu}}(\mathbf{Z})^T]$ include even functions of $G_{\Lambda,\boldsymbol{\mu}}(\mathbf{Z})$. Thus, in the following analysis without loss of generality we suppose $G_{\Lambda,\boldsymbol{\mu}}(\mathbf{Z})\sim\mathcal{N}(\boldsymbol{\mu},\Lambda\Lambda^T)$ where $\boldsymbol{\mu}^T\mathbf{b}_1>0$ for optimal $\mathbf{b}_1$. This is because the other component $\mathcal{N}(-\boldsymbol{\mu},\Lambda\Lambda^T)$ of the GMM results in the same expected values given an even function.

To characterize the stationary points of
$\mathcal{L}(G_{\Lambda,\boldsymbol{\mu}})$, we expand both  $\nabla_{\boldsymbol{\mu}} \mathcal{L}(G_{\Lambda,\boldsymbol{\mu}}) = \mathbf{0}$ and $\nabla_{\operatorname{vec}(\Lambda)} \mathcal{L}(G_{\Lambda,\boldsymbol{\mu}}) = \mathbf{0}$ as 
\begin{align}
    -\mathbb{E}\bigl[\tanh\bigl(\mathbf{b}^T_1G_{\Lambda,\boldsymbol{\mu}}(\mathbf{Z})\bigr)\bigr]\mathbf{b}_1 +\mathbb{E}\bigl[\tanh\bigl(\mathbf{b}^T_3G_{\Lambda,\boldsymbol{\mu}}(\mathbf{Z})\bigr)\bigr]\mathbf{b}_3 + \frac{1}{2\lambda}M^T\boldsymbol{\mu} \,&=\, \mathbf{0}, \\
    -\mathbf{b}_1\mathbb{E}\bigl[\tanh\bigl(\mathbf{b}^T_1G_{\Lambda,\boldsymbol{\mu}}(\mathbf{Z})\bigr)\mathbf{Z}^T\bigr] +\mathbf{b}_3\mathbb{E}\bigl[\tanh\bigl(\mathbf{b}^T_3G_{\Lambda,\boldsymbol{\mu}}(\mathbf{Z})\bigr)\mathbf{Z}^T\bigr] + \frac{1}{2\lambda}M^T\boldsymbol{\Lambda} \,&=\, \mathbf{0}.
\end{align}
Here $M:=\boldsymbol{\mu}\boldsymbol{\mu}^T+\Lambda\Lambda^T - \boldsymbol{\mu}_{\mathbf{X}}\boldsymbol{\mu}_{\mathbf{X}}^T - \Sigma_\mathbf{X}$ denotes the residual appearing in $\mathcal{L}_1(G_{\Lambda,\boldsymbol{\mu}})$'s gradient. Also, we define $\mathbf{b}_1$ and $\mathbf{b}_3$ as the optimal solutions to the maximization problem for $G_{\Lambda,\boldsymbol{\mu}}$.

Since we reduced the analysis to a multi-variate Gaussian $G_{\Lambda,\boldsymbol{\mu}}(\mathbf{Z})$, we can apply the multivariate generalization of Stein's lemma.  This application implies $\mathbb{E}\bigl[\tanh\bigl(\mathbf{b}^TG_{\Lambda,\boldsymbol{\mu}}(\mathbf{Z})\bigr)\mathbf{Z}^T\bigr] = \mathbb{E}\bigl[\tanh'\bigl(\mathbf{b}^TG_{\Lambda,\boldsymbol{\mu}}(\mathbf{Z})\bigr)\bigr]\mathbf{b}^T\Lambda$ and reduces the above identities to
\begin{align}
    -\mathbb{E}\bigl[\tanh\bigl(\mathbf{b}^T_1G_{\Lambda,\boldsymbol{\mu}}(\mathbf{Z})\bigr)\bigr]\mathbf{b}_1 +\mathbb{E}\bigl[\tanh\bigl(\mathbf{b}^T_3G_{\Lambda,\boldsymbol{\mu}}(\mathbf{Z})\bigr)\bigr]\mathbf{b}_3 + M^T\boldsymbol{\mu} \,&=\, \mathbf{0},\\
    \bigl\{-\mathbb{E}\bigl[\tanh'(\mathbf{b}^T_1G_{\Lambda,\boldsymbol{\mu}}(\mathbf{Z}))\bigr]\mathbf{b}_1\mathbf{b}^T_1 +\mathbb{E}\bigl[\tanh'(\mathbf{b}^T_3G_{\Lambda,\boldsymbol{\mu}}(\mathbf{Z}))\bigr]\mathbf{b}_3\mathbf{b}_3^T + M^T\bigr\}\boldsymbol{\Lambda} \,&=\, \mathbf{0}. 
\end{align}
Assuming the optimal $\Lambda$ is a full-rank matrix, we obtain:
\begin{equation}
   M^T = \mathbb{E}\bigl[\tanh'(\mathbf{b}^T_1G_{\Lambda,\boldsymbol{\mu}}(\mathbf{Z}))\bigr]\mathbf{b}_1\mathbf{b}^T_1 -\mathbb{E}\bigl[\tanh'(\mathbf{b}^T_3G_{\Lambda,\boldsymbol{\mu}}(\mathbf{Z}))\bigr]\mathbf{b}_3\mathbf{b}_3^T.
\end{equation}
Combining the above equations results in
\begin{align}
    &\bigl(\mathbb{E}\bigl[\tanh\bigl(\mathbf{b}^T_1G_{\Lambda,\boldsymbol{\mu}}(\mathbf{Z})\bigr)\bigr] - (\mathbf{b}^T_1\boldsymbol{\mu})\mathbb{E}\bigl[\tanh'(\mathbf{b}^T_1G_{\Lambda,\boldsymbol{\mu}}(\mathbf{Z}))\bigr]\bigr)\mathbf{b}_1 \nonumber \\
    =\, & \bigl(\mathbb{E}\bigl[\tanh\bigl(\mathbf{b}^T_3G_{\Lambda,\boldsymbol{\mu}}(\mathbf{Z})\bigr)\bigr] - (\mathbf{b}^T_3\boldsymbol{\mu})\mathbb{E}\bigl[\tanh'(\mathbf{b}^T_3G_{\Lambda,\boldsymbol{\mu}}(\mathbf{Z}))\bigr]\bigr)\mathbf{b}_3.
\end{align}
The theorem's assumption together with Lemma \ref{app: Lemma1 Thm3} shows that the scalars $\mathbb{E}\bigl[\tanh\bigl(\mathbf{b}^T_1G_{\Lambda,\boldsymbol{\mu}}(\mathbf{Z})\bigr)\bigr] - (\mathbf{b}^T_1\boldsymbol{\mu})\mathbb{E}\bigl[\tanh'(\mathbf{b}^T_1G_{\Lambda,\boldsymbol{\mu}}(\mathbf{Z}))\bigr]$ and $\mathbb{E}\bigl[\tanh\bigl(\mathbf{b}^T_3G_{\Lambda,\boldsymbol{\mu}}(\mathbf{Z})\bigr)\bigr] - (\mathbf{b}^T_3\boldsymbol{\mu})\mathbb{E}\bigl[\tanh'(\mathbf{b}^T_3G_{\Lambda,\boldsymbol{\mu}}(\mathbf{Z}))\bigr]$ are non-zero, since $G_{\Lambda,\boldsymbol{\mu}}(\mathbf{Z})$ satisfies the well-separability assumption and hence has a non-zero mean along any optimal $\mathbf{b}_1$ or $\mathbf{b}_3$. Therefore, $\mathbf{b}_1$ and $\mathbf{b}_3$ have the same direction and for a real $\alpha\in\mathbb{R}$ we have $\mathbf{b}_3=\alpha\mathbf{b}_1$. Note that $\alpha>0$ holds because as supposed in the theorem both optimal $\mathbf{b}_1, \mathbf{b}_3$ are inside a ball with radius $\frac{\mathbb{E}[\Vert\mathbf{X} \Vert_2]}{\lambda}$ around $\mathbf{d}$ which does not include $\mathbf{0}$.  As a result, we can reduce the above identity to
\begin{align}
    &\mathbb{E}\bigl[\tanh\bigl(\mathbf{b}^T_1G_{\Lambda,\boldsymbol{\mu}}(\mathbf{Z})\bigr)\bigr] - (\mathbf{b}^T_1\boldsymbol{\mu})\mathbb{E}\bigl[\tanh'(\mathbf{b}^T_1G_{\Lambda,\boldsymbol{\mu}}(\mathbf{Z}))\bigr] \nonumber \\
    =\, &\alpha\mathbb{E}\bigl[\tanh\bigl(\alpha\mathbf{b}^T_1G_{\Lambda,\boldsymbol{\mu}}(\mathbf{Z})\bigr)\bigr] - \alpha^2(\mathbf{b}^T_1\boldsymbol{\mu})\mathbb{E}\bigl[\tanh'(\alpha\mathbf{b}^T_1G_{\Lambda,\boldsymbol{\mu}}(\mathbf{Z}))\bigr]. \label{app: Thm4: Proof, Eq1}
\end{align}
\textbf{Claim:} $\alpha=1.$

To show this claim we define $Y=\mathbf{b}^T_1G_{\Lambda,\boldsymbol{\mu}}(\mathbf{Z})\sim \mathcal{N}(\mathbf{b}^T_1\boldsymbol{\mu},\mathbf{b}^T_1\Lambda\Lambda^T\mathbf{b}^T_1)$. Based on this definition, we define function $h(\alpha)$ and simplify \eqref{app: Thm4: Proof, Eq1} to $h(\alpha)=0$:
\begin{equation}
 h(\alpha):=\alpha\biggl( \mathbb{E}\bigl[\tanh(\alpha Y)\bigr] - \mathbb{E}[\alpha Y]\mathbb{E}\bigl[\tanh'(\alpha Y)\bigr]\biggr) - \biggl(\mathbb{E}\bigl[\tanh(Y)\bigr] - \mathbb{E}[Y]\mathbb{E}\bigl[\tanh'(Y)\bigr]\biggr).
\end{equation}
Note that $h(1)=0$ holds by definition. For the derivative $h'(\alpha)$ we have
\begin{align*}
    h'(\alpha)\, &= \, \mathbb{E}\bigl[\tanh(\alpha Y)\bigr] - \mathbb{E}[\alpha Y]\mathbb{E}\bigl[\tanh'(\alpha Y)\bigr] + \\
    &\quad + \mathbb{E}\biggl[\bigl[\alpha Y-\mathbb{E}[\alpha Y]\bigr]\tanh'(\alpha Y) \biggr] - \mathbb{E}[\alpha Y]\mathbb{E}[\alpha Y\tanh''(\alpha Y)] \\
    &\stackrel{(a)}{=} \mathbb{E}\bigl[\tanh(\alpha Y)\bigr] - \mathbb{E}[\alpha Y]\mathbb{E}\bigl[\tanh'(\alpha Y)\bigr] \\
    &\quad + \alpha^2\bigl(\sigma^2_Y - \mathbb{E}[Y]^2\bigr) \mathbb{E}\bigl[\tanh''(\alpha Y) \bigr] - \alpha^3 \mathbb{E}[Y]\sigma^2_Y\mathbb{E}\bigl[\tanh'''(\alpha Y)\bigr]. \numberthis
\end{align*}
Here, $(a)$ follows from Stein's lemma. According to the theorem's assumption, the underlying GMM satisfies the well-separability condition along optimal $\mathbf{b}_3=\alpha\mathbf{b}_1$. As a result, for any feasible $\alpha$, $\alpha^2\bigl(\mathbb{E}[Y]^2-\sigma^2_Y \bigr)- 2\alpha^3 \mathbb{E}[Y]\sigma^2_Y\ge 0$. Then, Lemma \ref{app: Lemma2: Thm1} and Lemma \ref{app: Lemma1 Thm3} imply that 
\begin{align*}
    h'(\alpha)\, &= \mathbb{E}\bigl[\tanh(\alpha Y)\bigr] - \mathbb{E}[\alpha Y]\mathbb{E}\bigl[\tanh'(\alpha Y)\bigr] \\
    &\quad + \alpha^2\bigl(\sigma^2_Y - \mathbb{E}[Y]^2\bigr) \mathbb{E}\bigl[\tanh''(\alpha Y) \bigr] - \alpha^3 \mathbb{E}[Y]\sigma^2_Y\mathbb{E}\bigl[\tanh'''(\alpha Y)\bigr] \\
    & > 0. \numberthis
\end{align*}
Hence, $h$ is an increasing function for any feasible $\alpha$ implying that $\alpha=1$ is the only solution for which $h(\alpha)=1$. The claim's validity follows from this result.

Since the claim holds, we have $\mathbf{b}_1=\mathbf{b}_3$. However, since $\mathbf{b}_1=\mathbf{b}_3$ means $\log(\frac{\exp(\mathbf{b}_1^T\mathbf{x})+\exp(-\mathbf{b}_1^T\mathbf{x})}{\exp(\mathbf{b}_3^T\mathbf{x})+\exp(-\mathbf{b}_3^T\mathbf{x})}) = 0$, we must have $\mathbf{b}_1=\mathbf{b}_3=\mathbf{d}$ to avoid any additional penalty coming from the regularization term in the maximization problem. Setting the gradient of the maximization objective to be $0$ at $\mathbf{b}_1=\mathbf{b}_3=\mathbf{d}$ implies
\begin{equation}\label{app: Thm4: proof Eq 6}
    \mathbb{E}\bigl[G_{\Lambda,\boldsymbol{\mu}}(\mathbf{Z})\tanh(\mathbf{d}^TG_{\Lambda,\boldsymbol{\mu}}(\mathbf{Z}))\bigr] = \mathbb{E}\bigl[\mathbf{X}\tanh(\mathbf{d}^T\mathbf{X}) \bigr].
\end{equation} 
Moreover, since $\nabla_{\operatorname{vec}(\Lambda,\boldsymbol{\mu})} \mathcal{L}_2(G_{\Lambda,\boldsymbol{\mu}})=\mathbf{0}$, we should also have $\nabla_{\operatorname{vec}(\Lambda,\boldsymbol{\mu})} \mathcal{L}_1(G_{\Lambda,\boldsymbol{\mu}})=\mathbf{0}$ which for a full-rank $\Lambda$ implies
\begin{equation}\label{app: Thm4: proof Eq 7}
        \mathbb{E}\bigl[G_{\Lambda,\boldsymbol{\mu}}(\mathbf{Z})G_{\Lambda,\boldsymbol{\mu}}(\mathbf{Z})^T\bigr] = \mathbb{E}\bigl[\mathbf{X}\mathbf{X}^T\bigr].
\end{equation}
However, in Theorem 3's proof we showed \eqref{app: Thm4: proof Eq 6} and \eqref{app: Thm4: proof Eq 7} under Theorem 4's assumption implies the same distribution for $\mathbf{X}$ and $G_{\Lambda,\boldsymbol{\mu}}(\mathbf{Z})$. This result shows the only stationary point in the characterized set provides the underlying GMM. The proof is therefore complete.
%\end{proof}

\subsection{Proof of Theorem 5}
Considering the additional optimization constraint, all the moment functions in the minimax objective, i.e.,  $\mathbb{E}[\mathbf{X}\mathbf{X}^T]$ and $\mathbb{E}[\log(\exp(\mathbf{b}^T\mathbf{X})+\exp(-\mathbf{b}^T\mathbf{X}))]$, are the expected vales of even functions for which we have the same expected value under $\mathcal{N}(\boldsymbol{\mu}_{\mathbf{X}},\Sigma_{\mathbf{X}})$ and $\mathcal{N}(-\boldsymbol{\mu}_{\mathbf{X}},\Sigma_{\mathbf{X}})$. Therefore, without loss of generality we perform the generalization analysis assuming that all the $n$ samples are drawn from $\mathcal{N}(\boldsymbol{\mu}_{\mathbf{X}},\Sigma_{\mathbf{X}})$. Note that as discussed in the proof of Theorem 4, this additional constraint does not change the optimal solution to the minimax problem.

Similar to Theorem 4's proof, we consider $\mathcal{L}(G_{\Lambda,\boldsymbol{\mu}})$ as the summation of the following two non-negative components
\begin{align*}
    \mathcal{L}_1(G_{\Lambda,\boldsymbol{\mu}}):&= \max_A\: \mathbb{E}[\frac{1}{2}\mathbf{X}^TA\mathbf{X}] - \mathbb{E}[\frac{1}{2}G_{\Lambda,\boldsymbol{\mu}}(\mathbf{Z})^TAG_{\Lambda,\boldsymbol{\mu}}(\mathbf{Z})] + \frac{\lambda}{2}\Vert A\Vert^2_F \\
    &= \frac{1}{2\lambda}\big\Vert \mathbb{E}[\mathbf{X}\mathbf{X}^T]- \mathbb{E}[G_{\Lambda,\boldsymbol{\mu}}(\mathbf{Z})G_{\Lambda,\boldsymbol{\mu}}(\mathbf{Z})^T] \big\Vert^2_F, \numberthis\\
    \mathcal{L}_2(G_{\Lambda,\boldsymbol{\mu}}) :&= \max_{\mathbf{b}_1,\mathbf{b}_3}\: \mathbb{E}\bigl[\log\bigl(\frac{\exp(\mathbf{b}_1^T\mathbf{X})+\exp(-\mathbf{b}_1^T\mathbf{X})}{\exp(\mathbf{b}_3^T\mathbf{X})+\exp(-\mathbf{b}_3^T\mathbf{X})}\bigr) \bigr] \\
    &\quad - \mathbb{E}\bigl[\log\bigl(\frac{\exp(\mathbf{b}_1^TG_{\Lambda,\boldsymbol{\mu}}(\mathbf{Z}))+\exp(-\mathbf{b}_1^TG_{\Lambda,\boldsymbol{\mu}}(\mathbf{Z}))}{\exp(\mathbf{b}_3^TG_{\Lambda,\boldsymbol{\mu}}(\mathbf{Z}))+\exp(-\mathbf{b}_3^TG_{\Lambda,\boldsymbol{\mu}}(\mathbf{Z}))}\bigr) \bigr]  -\frac{\lambda}{2}\bigl(\Vert\mathbf{b}_1-\mathbf{d}\Vert^2_2  + \Vert\mathbf{b}_3-\mathbf{d}\Vert^2_2 \bigr). \numberthis
\end{align*}
We also use $\widehat{\mathcal{L}}_1(G_{\Lambda,\boldsymbol{\mu}})$ and $\widehat{\mathcal{L}}_2(G_{\Lambda,\boldsymbol{\mu}})$ to denote the empirical versions of the above definitions evaluated on the empirical distribution of $n$ observed samples. 

To establish the generalization bound, we first bound the convergence rate for estimating the empirical mean vector $\widehat{\boldsymbol{\mu}}$ and covariance and $\widehat{\Sigma}$ using $n$ i.i.d. samples from the Gaussian distribution $\mathcal{N}(\boldsymbol{\mu}_{\mathbf{X}},\Sigma_{\mathbf{X}})$. Applying standard covariance bound developed for $\Vert \Sigma\Vert_\sigma$ sub-Gaussian distribution $\mathcal{N}(\mathbf{0},\Sigma_{\mathbf{X}})$ we can show that for any $\delta>0$ with probability at least $1-\delta$ we have:
\begin{equation}\label{app: Thm 5: covariance}
    \big\Vert \widehat{\Sigma} - \Sigma_{\mathbf{X}} \big\Vert_{\sigma}\, \le \, \Vert\Sigma_{\mathbf{X}} \Vert_{\sigma}\sqrt{\frac{C_1 d\log(2/\delta)}{n}}
\end{equation}
where $d$ denotes $\mathbf{X}$'s dimension and $C_1$ is a universal constant. We refer the readers to \cite{vershynin2012close} for a complete proof of this result. For the convergence of the empirical mean $\widehat{\boldsymbol{\mu}}$ to the true mean $\boldsymbol{\mu}_{\mathbf{X}}$, we can use an $\epsilon$-covering over the unit ball with size $N=6^d$ and apply standard concentration inequalities to show there exists a universal constant $C_2$ that for any $\delta$ with probability at least $1-\delta$ the following holds for any unit-norm $\Vert \mathbf{u}\Vert_2=1$ 
\begin{equation}
    \big\vert (\mathbf{u}^T\widehat{\boldsymbol{\mu}})^2 - (\mathbf{u}^T\boldsymbol{\mu}_{\mathbf{X}})^2  \big\vert \le \Vert\Sigma_{\mathbf{X}} \Vert_{\sigma}\Vert\boldsymbol{\mu}_{\mathbf{X}}\Vert_2\sqrt{\frac{C_2 d\log(2/\delta)}{n}}.
\end{equation}

Having that $\Vert \boldsymbol{\mu}_1\boldsymbol{\mu}_1^T - \boldsymbol{\mu}_2\boldsymbol{\mu}_2^T  \Vert_\sigma = \max_{\Vert\mathbf{u}\Vert_2=1} \vert (\mathbf{u}^T\boldsymbol{\mu}_1)^2 - (\mathbf{u}^T\boldsymbol{\mu}_2)^2 \vert$, the above results say there exists a constant $C$ that for any $\delta>0$ with probability at least $1-\delta$ we have
\begin{equation}
    \big\Vert \widehat{\boldsymbol{\mu}}\widehat{\boldsymbol{\mu}}^T +\widehat{\Sigma} - {\boldsymbol{\mu}}_{\mathbf{X}}{\boldsymbol{\mu}}_{\mathbf{X}}^T - {\Sigma}_{\mathbf{X}} \big\Vert_{\sigma} \, \le \, \Vert\Sigma_{\mathbf{X}} \Vert_{\sigma}\bigl(1+\Vert\boldsymbol{\mu}_{\mathbf{X}}\Vert_2\bigr)\sqrt{\frac{C d\log(4/\delta)}{n}}.
\end{equation}
Since $\Vert A\Vert_F\le \sqrt{d}\Vert A\Vert_{\sigma}$ always holds for a $d\times d$ matrix $A$, the above inequality also implies
\begin{equation}
    \big\Vert \widehat{\boldsymbol{\mu}}\widehat{\boldsymbol{\mu}}^T +\widehat{\Sigma} - {\boldsymbol{\mu}}_{\mathbf{X}}{\boldsymbol{\mu}}_{\mathbf{X}}^T - {\Sigma}_{\mathbf{X}} \big\Vert_{F} \, \le \, d\Vert\Sigma_{\mathbf{X}} \Vert_{\sigma}\bigl(1+\Vert\boldsymbol{\mu}_{\mathbf{X}}\Vert_2\bigr)\sqrt{\frac{C \log(4/\delta)}{n}}.
\end{equation}
Therefore, the following inequalities will also hold
\begin{align}
    &\big\vert \widehat{\mathcal{L}}_1(G_{\Lambda,\boldsymbol{\mu}})- {\mathcal{L}}_1(G_{\Lambda,\boldsymbol{\mu}})\big\vert \nonumber\\
    \le\, & 2\bigl(\mathbb{E}[\Vert\mathbf{X}\Vert^2]+\mathbb{E}[\Vert G_{\Lambda,\boldsymbol{\mu}}(\mathbf{Z})\Vert^2]\bigr)d\Vert\Sigma_{\mathbf{X}} \Vert_{\sigma}\bigl(1+\Vert\boldsymbol{\mu}_{\mathbf{X}}\Vert_2\bigr)\sqrt{\frac{C \log(4/\delta)}{\lambda^2 n}}\nonumber  +\mathcal{O}\bigl(\frac{d^2\log(1/\delta)}{n}\bigr) \nonumber\\
    =\, & \mathcal{O}(\sqrt{\frac{d^2\log(1/\delta)}{\lambda^2 n}}), \label{app: Thm 5: L1 Value} \\
    &\big\Vert \nabla_{\boldsymbol{\mu}}\widehat{\mathcal{L}}_1(G_{\Lambda,\boldsymbol{\mu}})(\boldsymbol{\mu})- \nabla_{\boldsymbol{\mu}} {\mathcal{L}}_1(G_{\Lambda,\boldsymbol{\mu}})(\boldsymbol{\mu})\big\Vert_{\sigma}\nonumber \\
    \le\, & \Vert\boldsymbol{\mu}\Vert_2\Vert\Sigma_{\mathbf{X}} \Vert_{\sigma}\bigl(1+\Vert\boldsymbol{\mu}_{\mathbf{X}}\Vert_2\bigr)\sqrt{\frac{C d\log(4/\delta)}{\lambda^2 n}} \nonumber \\
    = &\mathcal{O}(\sqrt{\frac{d\log(1/\delta)}{\lambda^2 n}}),\label{app: Thm 5: L1 Mu Gradient}\\
    &\big\Vert \nabla_{\operatorname{vec}(\Lambda)}\widehat{\mathcal{L}}_1(G_{\Lambda,\boldsymbol{\mu}})(\Lambda)- \nabla_{\operatorname{vec}(\Lambda)} {\mathcal{L}}_1(G_{\Lambda,\boldsymbol{\mu}})(\Lambda)\big\Vert_{\sigma} \nonumber \\
    \le\, & \Vert\Lambda\Vert_{\sigma}\Vert\Sigma_{\mathbf{X}} \Vert_{\sigma}\bigl(1+\Vert\boldsymbol{\mu}_{\mathbf{X}}\Vert_2\bigr)\sqrt{\frac{C d\log(4/\delta)}{\lambda^2 n}} \nonumber \\
    =\,& \mathcal{O}(\sqrt{\frac{d\log(1/\delta)}{\lambda^2 n}}).\label{app: Thm 5: L1 Lambda Gradient}
\end{align}
Next, we bound the generalization error terms for $\mathcal{L}_2(G_{\Lambda,\boldsymbol{\mu}})$. Note that $h_{\mathbf{b}}(\mathbf{x})=\log\bigl(\exp(\mathbf{b}^T\mathbf{x})+\exp(-\mathbf{b}^T\mathbf{x})\bigr)$ is a Lipschitz function with Lipschitz constant $\Vert\mathbf{b}\Vert_2$. Also, since the maximization problem in
$\mathcal{L}_2(G_{\Lambda,\boldsymbol{\mu}})$ maximizes a $\{ \Lambda -\mathbb{E}[\Vert\mathbf{X}\Vert^2_2]-\mathbb{E}[\Vert G_{\Lambda,\boldsymbol{\mu}}(\mathbf{Z})\Vert^2_2] \}$-strongly concave objective, any optimal 
$\mathbf{b}_1$ or $\mathbf{b}_3$ has a Euclidean norm upper-bounded by $M=\Vert\mathbf{d}\Vert_2+\frac{\mathbb{E}[\Vert\mathbf{X}\Vert^2_2]+\mathbb{E}[\Vert G_{\Lambda,\boldsymbol{\mu}}(\mathbf{Z})\Vert^2_2]}{\lambda}$.

We hence consider a cover of size $O(M^d)$ for such norm-bounded $\mathbf{b}$'s. Note that for each $\Vert\mathbf{b}\Vert_2\le M$, $h_{\mathbf{b}}(\mathbf{X})=\log\bigl(\exp(\mathbf{b}^T\mathbf{X})+\exp(-\mathbf{b}^T\mathbf{X})\bigr)$ provides an $M$-Lipschitz function of Gaussian $\mathbf{X}$ for which we can apply standard concentration bounds \cite{vershynin2010introduction} to obtain
\begin{equation}
    \Pr\bigl(\big|\frac{1}{n}\sum_{i=1}^n h_{\mathbf{b}}(\mathbf{x}_i)- \mathbb{E}[h_\mathbf{b}(\mathbf{X})] \big|\ge t\bigr) \le \exp\bigl(-\frac{n t^2}{2M^2\Vert\Sigma_\mathbf{X}\Vert^2_\sigma C_3}\bigr)
\end{equation}
where $C_3$ is a universal constant. As a result, for any $\delta$ with probability at least $1-\delta$ the following generalization bound uniformly holds over the set of norm-bounded $\Vert\mathbf{b}\Vert_2\le M$
\begin{equation}
    \big|\frac{1}{n}\sum_{i=1}^n h_{\mathbf{b}}(\mathbf{x}_i)- \mathbb{E}[h_\mathbf{b}(\mathbf{X})] \big| \le \mathcal{O}\biggl(M\Vert\Sigma_\mathbf{X}\Vert_\sigma\sqrt{\frac{d\log(M/\delta) }{n}}\biggr).
\end{equation}
Therefore, the empirical maximization objective of $\mathcal{L}_2(G_{\Lambda,\boldsymbol{\mu}})$ is at most $\mathcal{O}\bigl(M\Vert\Sigma_\mathbf{X}\Vert_\sigma\sqrt{\frac{d\log(M/\delta) }{n}}\bigr)$-different from the underlying objective. Since $\mathcal{L}_2(G_{\Lambda,\boldsymbol{\mu}})$ reduces to maximizing a $\bigl\{\lambda-\mathbb{E}[\Vert\mathbf{X}\Vert^2_2]-\mathbb{E}[\Vert G_{\Lambda,\boldsymbol{\mu}}(\mathbf{Z})\Vert^2_2]\bigr\}$-strongly concave objective, this result implies that the optimal $\widehat{\mathbf{b}}_1,\widehat{\mathbf{b}}_3$ for the empirical objective will be different from the optimal solution $\mathbf{b}_1,\mathbf{b}_3$ for the underlying problem by at most
\begin{equation}
    \max\bigl\{\Vert\widehat{\mathbf{b}}_1-\mathbf{b}_1\Vert_2,\Vert\widehat{\mathbf{b}}_3-\mathbf{b}_3\Vert\bigr\} \le \mathcal{O}\biggl(M\Vert\Sigma_\mathbf{X}\Vert_\sigma\sqrt{\frac{d\log(M/\delta) }{\lambda^2 n}}\biggr).
\end{equation}
Remember from Theorem 4's proof that by applying the Danskin's theorem we get
\begin{align*}
    \nabla_{\boldsymbol{\mu}} \mathcal{L}_2(G_{\Lambda,\boldsymbol{\mu}}) &= \mathbb{E}\bigl[\tanh(\mathbf{b}_3^T G_{\Lambda,\boldsymbol{\mu}}(\mathbf{Z})) \bigr]\mathbf{b}_3- \mathbb{E}\bigl[\tanh(\mathbf{b}_1^T G_{\Lambda,\boldsymbol{\mu}}(\mathbf{Z})) \bigr]\mathbf{b}_1 \\
    \nabla_{\Lambda} \mathcal{L}_2(G_{\Lambda,\boldsymbol{\mu}}) &= \mathbf{b}_3\mathbb{E}\bigl[\tanh(\mathbf{b}_3^T G_{\Lambda,\boldsymbol{\mu}}(\mathbf{Z}))\mathbf{Z}^T \bigr]- \mathbf{b}_1\mathbb{E}\bigl[\tanh(\mathbf{b}_1^T G_{\Lambda,\boldsymbol{\mu}}(\mathbf{Z})) \mathbf{Z}^T\bigr].
\end{align*}
Consequently, with probability at least $1-\delta$ we will have the following inequalities hold for all feasible $\boldsymbol{\mu},\Lambda$
\begin{align}
    &\big\vert \widehat{\mathcal{L}}_2(G_{\Lambda,\boldsymbol{\mu}})- {\mathcal{L}}_2(G_{\Lambda,\boldsymbol{\mu}})\big\vert \le \mathcal{O}\biggl(M\Vert\Sigma_\mathbf{X}\Vert_\sigma\sqrt{\frac{d\log(M/\delta) }{n}}\biggr) \label{app: Thm 5: L2 Value} \\
    &\big\Vert \nabla_{\boldsymbol{\mu}}\widehat{\mathcal{L}}_1(G_{\Lambda,\boldsymbol{\mu}})(\boldsymbol{\mu})- \nabla_{\boldsymbol{\mu}} {\mathcal{L}}_2(G_{\Lambda,\boldsymbol{\mu}})(\boldsymbol{\mu})\big\Vert_{\sigma}  \le \mathcal{O}\biggl(\mathbb{E}[
    \Vert G_{\Lambda,\boldsymbol{\mu}}(\mathbf{Z}) \Vert_2]  M\Vert\Sigma_\mathbf{X}\Vert_\sigma\sqrt{\frac{d\log(M/\delta) }{\lambda^2 n}}\biggr),\label{app: Thm 5: L2 Mu Gradient}\\
    &\big\Vert \nabla_{\Lambda}\widehat{\mathcal{L}}_2(G_{\Lambda,\boldsymbol{\mu}})(\Lambda)- \nabla_{\Lambda} {\mathcal{L}}_2(G_{\Lambda,\boldsymbol{\mu}})(\Lambda)\big\Vert_{\sigma} \le \mathcal{O}\biggl(\mathbb{E}[
    \Vert G_{\Lambda,\boldsymbol{\mu}}(\mathbf{Z}) \Vert_2]  M\Vert\Sigma_\mathbf{X}\Vert_\sigma\sqrt{\frac{d\log(M/\delta) }{\lambda^2 n}}\biggr).\label{app: Thm 5: L2 Lambda Gradient}
\end{align}
Finally, combining \eqref{app: Thm 5: L1 Value}-\eqref{app: Thm 5: L1 Lambda Gradient} with \eqref{app: Thm 5: L2 Value}-\eqref{app: Thm 5: L2 Lambda Gradient} shows that for every $\delta>0$ with probability at least $1-\delta$ the following will hold for any feasible $G_{\Lambda,\boldsymbol{\mu}}$
\begin{align}
    &\big\vert \widehat{\mathcal{L}}(G_{\Lambda,\boldsymbol{\mu}})- {\mathcal{L}}(G_{\Lambda,\boldsymbol{\mu}})\big\vert \le \mathcal{O}\biggl(\sqrt{\frac{d^2\log(1/\delta) }{\lambda^2 n}}\biggr), \label{app: Thm 5: L Value} \\
    &\big\Vert \nabla_{\boldsymbol{\mu}}\widehat{\mathcal{L}}_1(G_{\Lambda,\boldsymbol{\mu}})(\boldsymbol{\mu})- \nabla_{\boldsymbol{\mu}} {\mathcal{L}}_2(G_{\Lambda,\boldsymbol{\mu}})(\boldsymbol{\mu})\big\Vert_{\sigma}   \nonumber \\
   & + \big\Vert \nabla_{\Lambda}\widehat{\mathcal{L}}_1(G_{\Lambda,\boldsymbol{\mu}})(\Lambda)- \nabla_{\Lambda}{\mathcal{L}}_2(G_{\Lambda,\boldsymbol{\mu}})(\Lambda)\big\Vert_{\sigma} \, \le \, \mathcal{O}\biggl(\sqrt{\frac{d\log(1/\delta) }{\lambda^2 n}}\biggr).\label{app: Thm 5: L Gradient} 
\end{align}
Therefore, the proof is complete.
%\end{proof}

\end{appendices}

\end{document}